\def\isarxiv{1} %%% for NeurIPS 2026 submission version, we comment this line

\ifdefined\isarxiv
\documentclass[11pt]{article}

\usepackage{natbib}

\AtBeginDocument{%
  \let\citet\citep
  
}

\else
\documentclass{article}

\usepackage{neurips_2026}
\fi

\usepackage{multicol}
\usepackage{amsmath}
\usepackage{amsthm}
\usepackage{amssymb}

\usepackage{mathtools}
\usepackage{algorithm}

\usepackage{subfig}
\usepackage{capt-of}  % provides \captionof for use inside minipages
\usepackage{graphicx}
\usepackage{grffile}
\usepackage{wrapfig,epsfig}
\usepackage{url}
\usepackage{xcolor}
\usepackage{epstopdf}
\usepackage{multicol}

\usepackage{bbm}
\usepackage{dsfont}

\usepackage{booktabs}
 
\allowdisplaybreaks

\ifdefined\isarxiv

\usepackage{tikz}
\usepackage{hyperref}  %%% arxiv don't allow this.
\hypersetup{colorlinks=true,citecolor=blue,linkcolor=blue} %%% Zhao : maybe we should comment this in submission.
\usetikzlibrary{arrows, arrows.meta, positioning, shapes.geometric, calc, matrix}
\usepackage[margin=1in]{geometry}

\else

\usepackage{microtype}
\usepackage{hyperref}

\usepackage{tikz}
\usetikzlibrary{arrows.meta, positioning, shapes.geometric, calc, matrix}
\definecolor{mydarkblue}{rgb}{0,0.08,0.45}
\definecolor{mylushgreen}{rgb}{0.110, 0.380, 0.238}
\hypersetup{colorlinks=true, citecolor=mylushgreen,linkcolor=mylushgreen}

\fi
 
\graphicspath{{./figs/}}

\theoremstyle{plain}
\newtheorem{theorem}{Theorem}[section]
\newtheorem{lemma}[theorem]{Lemma}
\newtheorem{definition}[theorem]{Definition}

\newtheorem{corollary}[theorem]{Corollary}

\newtheorem{assumption}[theorem]{Assumption}

\newtheorem{fact}[theorem]{Fact}
\newtheorem{remark}[theorem]{Remark}

\newcommand{\wh}{\widehat}
\newcommand{\wt}{\widetilde}

\newcommand{\R}{\mathbb{R}}

\renewcommand{\d}{\mathrm{d}}
\renewcommand{\i}{\mathbf{i}}
\renewcommand{\tilde}{\wt}
\renewcommand{\hat}{\wh}

\newcommand{\myvspace}[1]{\ifdefined\isarxiv\else\vspace{#1}\fi}

\DeclareMathOperator*{\E}{{\mathbb{E}}}
\DeclareMathOperator*{\var}{\mathrm{Var}}

\DeclareMathOperator{\poly}{poly}
\DeclareMathOperator{\I}{\mathbb{I}}

\DeclareMathOperator{\diag}{diag}

\DeclareMathOperator{\vect}{vec}
\DeclareMathOperator{\matr}{mat}

\makeatletter
\newcommand*{\RN}[1]{\expandafter\@slowromancap\romannumeral #1@}
\makeatother

\usepackage{lineno}

\ifdefined\isarxiv
\else
\usepackage{etoolbox}
\BeforeBeginEnvironment{multicols}{\par\nolinenumbers}
\AfterEndEnvironment{multicols}{\linenumbers\par}
\fi

\ifdefined\isarxiv
\else

\usepackage[textsize=tiny]{todonotes}

\fi

\ifdefined\isarxiv
\else
\title{Unifying Learning Dynamics and Generalization in Transformers Scaling Law}
\author{%
  Anonymous Author(s)\\
  Affiliation\\
  Address\\
  \texttt{email}
}
\fi

\begin{document}

\ifdefined\isarxiv

\date{}

\title{Unifying Learning Dynamics and Generalization in Transformers Scaling Law
% \footnote{This manuscript is a work in progress and will be updated.}
}
\author{
Chiwun Yang\thanks{ \texttt{christiannyang37@gmail.com}. Sun Yat-sen University.}
}

\maketitle
\begin{abstract}
The scaling law, a cornerstone of Large Language Model (LLM) development, predicts improvements in model performance with increasing computational resources. Yet, while empirically validated, its theoretical underpinnings remain poorly understood. This work formalizes the learning dynamics of transformer-based language models as an ordinary differential equation (ODE) system, then approximates this process to kernel behaviors. Departing from prior toy-model analyses, we rigorously analyze stochastic gradient descent (SGD) training for multi-layer transformers on sequence-to-sequence data with arbitrary data distribution, closely mirroring real-world conditions. Our analysis characterizes the convergence of generalization error to the irreducible risk as computational resources scale with data, especially during the optimization process. 

We establish matching upper and lower bounds on the excess risk, characterized by a distinct phase transition. In the initial optimization phase, the excess risk decays exponentially relative to the computational cost ${\sf C}$. However, once a specific resource allocation threshold is crossed, the system enters a statistical phase, where the generalization error follows a power-law decay of $\Theta(\mathsf{C}^{-1/7})$. These rates are certified by complementary lower bounds --- statistical, via an information-theoretic two-point reduction, and optimization-side, via a first-order oracle argument --- rendering the two-stage law tight up to constants, logarithmic factors, and a condition-number gap. Beyond this unified framework, our theory derives isolated scaling laws for model size, training time, and dataset size, elucidating how each variable independently governs the bounds of generalization.

\end{abstract}

% {\hypersetup{linkcolor=black}
% \tableofcontents
% }
% \newpage

\else
\maketitle
\begin{abstract}
The scaling law, a cornerstone of Large Language Model (LLM) development, predicts improvements in model performance with increasing computational resources. Yet, while empirically validated, its theoretical underpinnings remain poorly understood. This work formalizes the learning dynamics of transformer-based language models as an ordinary differential equation (ODE) system, then approximates this process to kernel behaviors. Departing from prior toy-model analyses, we rigorously analyze stochastic gradient descent (SGD) training for multi-layer transformers on sequence-to-sequence data with arbitrary data distribution, closely mirroring real-world conditions. Our analysis characterizes the convergence of generalization error to the irreducible risk as computational resources scale with data, especially during the optimization process. 

We establish matching upper and lower bounds on the excess risk, characterized by a distinct phase transition. In the initial optimization phase, the excess risk decays exponentially relative to the computational cost ${\sf C}$. However, once a specific resource allocation threshold is crossed, the system enters a statistical phase, where the generalization error follows a power-law decay of $\Theta(\mathsf{C}^{-1/7})$. These rates are certified by complementary lower bounds --- statistical, via an information-theoretic two-point reduction, and optimization-side, via a first-order oracle argument --- rendering the two-stage law tight up to constants, logarithmic factors, and a condition-number gap. Beyond this unified framework, our theory derives isolated scaling laws for model size, training time, and dataset size, elucidating how each variable independently governs the bounds of generalization.

\end{abstract}

\fi

\section{Introduction}
The emergence of transformer-based Large Language Models (LLMs) such as ChatGPT \citep{cha22}, GPT-4 \& 5  \citep{aaa+23, bce+23}, DeepSeek \citep{bcb+24, lfx+24, lfw+24, gdj+24}, and LLaMA \citep{tli+23, tms+23, gdj+24} has driven transformative advancements across multiple domains \citep{bmr+20}. Tasks like code generation \citep{laz+23, gzy+24}, conversational systems \citep{mrkk23, xgdm23, zcs+24}, and mathematical reasoning \citep{hbb+20, yjs+23, yyz+23}, once considered the exclusive province of human expertise, are now routinely mastered by these AI systems.

This remarkable progress is fundamentally tied to computational scaling. Empirical evidence reveals a consistent pattern: as the compute budget for optimally training and deploying language model increases, their demonstrated intelligence scales correspondingly \citep{kmh+20, hbm+22}. This phenomenon has been systematically categorized into the following principle:
\begin{quote}
    \textbf{Scaling Law} \citep{kmh+20, hbm+22}: Model capabilities improve predictably with increased computational investment during training, achieved through three-dimensional scaling: model parameter count, training duration, and dataset size.
\end{quote}
While the scaling law has been extensively validated through empirical research, a critical gap persists in its theoretical foundation. Current literature lacks a mathematically rigorous framework to explain why these principles exhibit such predictable improvements in model performance, or to systematically justify their reliability in guiding the development of massive-scale transformer architectures \citep{lwk+24}. Existing theoretical studies on scaling laws have predominantly focused on traditional statistical models and simplified neural architectures. Key investigations include analyses of shallow attention networks \citep{lwz25}, linear regression \citep{lwk+24, dsy24_quant}, kernel regression \citep{cgl+25}, solvable random-feature models \citep{mrs22, bdk+24, bap24, ppxp24, blp25}, and data selection methodologies \citep{jkk+18}, among others. Compute-optimal recommendations and data-constrained refinements~\citep{hbm+22, mas+23, ppx24, frh24} highlight the empirical N--M trade-off any theoretical scaling law must reproduce. However, these works lack generalizability to modern deep learning systems. Notably, the scaling theory underpinning transformer-based LLMs \citep{vsp+17} --- the dominant paradigm in contemporary AI --- remains largely unexplored. 

\begin{figure*}
    \centering
    \includegraphics[width=\linewidth]{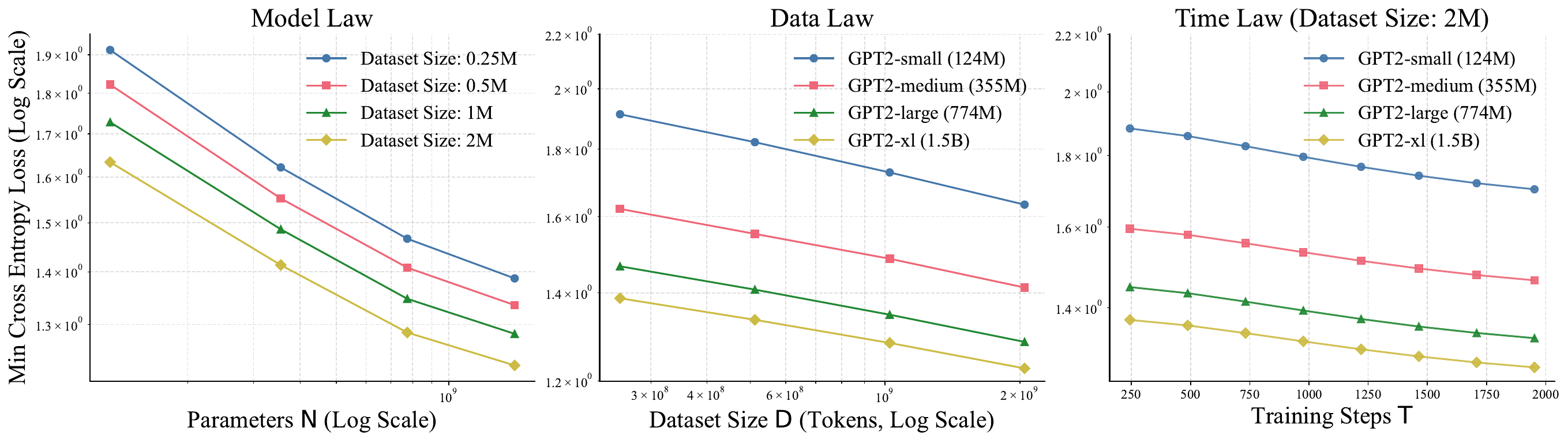}
    \caption{Visualization of the single-variable scaling laws (Model, Data, Time) on GPT-2~\citep{bmr+20} $\times$ TinyStories~\citep{el23}, fixing two of $({\sf M},{\sf N},{\sf T})$ and varying the third. Theoretical analysis is in Theorem~\ref{thm:single_law:informal}; the second case study (Pythia $\times$ WikiText-103, with a Chinchilla-style joint fit on the $6\times 7$ grid) and the noise-breakdown sweep are deferred to Appendix~\ref{app_sub:wikitext_results}.}
    \label{fig:single_law}
\end{figure*}

Our work addresses this gap by establishing theoretical foundations for computational scaling effectiveness in training multi-layer transformer-based language models on sequence-to-sequence data. We derive a quantifiable relationship between computational expenditure and model capabilities, formulating a {\it matching} (upper- and lower-bounded) characterization of the expected risk that depends explicitly on the allocated computational resources. In particular, the goal of this work is to address the following critical question:
\begin{center}
    {\it What are the achievable theoretical limits for computational resource allocation in the convergence of the generalization error during transformer-based language model training, and to what extent are these limits tight?}
\end{center}
We present the first comprehensive analysis of the training dynamics of transformer-based language models using the Stochastic Gradient Descent (SGD) algorithm, along with a convergence guarantee for arbitrary training error. This topic has been largely unexplored due to the inherent complexity of attention mechanisms, the multi-layered structure of transformers, and the extensive matrix computations required for sequence-to-sequence data processing. To address these challenges, we adopt a kernel-based analytical framework to investigate the scaling behavior. By analyzing this scenario, we establish the lower and upper bounds for the expected risk on the whole data distribution. Overall, we make the following contributions:
\begin{itemize}
    \item We first simplify the complicated matrix computation to the parallel vector computation utilizing the decoder-only property of a generative transformer-based language model. Therefore, we formalize its explicit learning dynamics, which we assume the learning behavior of LLMs is constrained to the kernel regimes due to the over-parameterization. This is referred to as the {\it Lazy Learning} \citep{jgh18, dzps19}, where the model merely memorizes all data points during optimization. (See Section~\ref{sec:learning_dynamics})
    \item We further explore several benefits that emerge under the trend of model scaling, e.g., converging kernel perturbation, which means that once kernel behavior exists at initialization, it will remain stable during training. Moreover, we demonstrate how the training convergence rate exponentially improves with linearly increasing model depth. Finally, we showcase the guaranteed training convergence and generalization bound of the empirical risk minimizer (ERM) related to vital scaling factors (model size, dataset size). (See Section~\ref{sec:convergence})
    \item The primary contribution of this work is a two-stage upper bound on the in-distribution expected risk, which formalizes a phase transition between a \textit{compute-starved stage} and a \textit{data-limited stage}. Our analysis reveals that the generalization error initially undergoes an exponential decay governed by total computational expenditure, ${\sf C}$, during which the expected risk is dominated by optimization dynamics. Upon reaching a critical saturation point, defined by the condition in Eq.~\eqref{eq:cond_C}, the scaling trajectory undergoes a structural shift: the rapid exponential decrease transitions into a power-law regime. In this subsequent stage, the frontier error is fundamentally constrained by the target function's intrinsic complexity and the dataset size, ${\sf N}$. Consequently, further increases in training duration or model capacity yield diminishing returns on generalization performance unless accompanied by a proportional expansion of high-quality data. We empirically substantiate these theoretical findings through a comprehensive case study on the GPT-2 series \cite{bmm+21} and Pythia series \cite{bsa+23}, which confirms the predicted phase transition between scaling regimes and validates the critical conditions under which data scaling laws break down due to noise.
    (See Section~\ref{sec:scaling_law} and Section~\ref{sec:exp})
\end{itemize}

\ifdefined\isarxiv
\section{Related Work}\label{sec:ralted_work}
{\bf Scaling Law.}
Several recent works have empirically explored scaling laws in deep neural networks \citep{kmh+20, hbm+22, rfcs21, hna+17, rrbs19, hka+20, mka18, mas+23, sos+22, niu2024beyond, ivgi2022scaling, petrov2025closing, mehta2025scaling, bi2025scaling, cagnetta2025scaling, havrilla2024understanding, kas+25}. The study of neural scaling laws can also be traced back to earlier foundational works \citep{cd07, shs+09, at88}. From a theoretical perspective, various solvable models have been developed using random feature models \citep{bdk+24, mrs22, abp+21, azvp24, bap24, ppxp24, dlm24, clkz21} to analyze neural scaling laws under specific limits. Additionally, theoretical analyses on linear models \citep{whs22, bcp20, sk22, bp21, lwk+24, lwz25} have significantly advanced our understanding of scaling laws. We further note matching upper/lower-bound analyses for stylized optimization scaling: \citet{kb25} establish two-sided scaling laws for gradient descent and sign descent on linear bigram models under Zipf's law, and \citet{azpf25} extend dynamical scaling-law theory to kernel regression with explicit learning-rate schedules. In particular, \citet{lwk+24} establish two-sided scalings in linear regression with sketched covariates, but they do not encode causal attention or the $nL\times nL$ token couplings central to transformer dynamics; \citet{vbn22} report empirically that NTK predictors can miss certain neural scaling exponents, motivating our explicit treatment of upper-bound exponents in a stated lazy regime; and \citet{boc+25} show how feature-learning corrections can improve scaling, complementing our worst-case bounds. We further compare to internal-dynamics analyses through the NTK lens \citep{nik+25} and to multi-phase phenomena in stylized settings \citep{lyu+23} that organize compute into stages similar in spirit to our two-stage bound. In contrast to these studies, our work focuses on the sequence-to-sequence stochastic training of multi-layer transformer-based language models, a topic that has not been widely discussed in prior research.

{\bf Risk decomposition strategy.}
Many existing scaling-law analyses (e.g.~\citep{lwk+24} and follow-ups) study a population-level risk and decompose it into approximation error of the global minimizer plus excess (estimation) risk of the population objective. Our analysis instead splits the excess risk into an {\it optimization} error along the gradient-flow trajectory and a {\it generalization} gap of the empirical risk minimizer (ERM) over the parameter space. The two strategies are mathematically related, but trajectory-aware decomposition is necessary in our setting because we want to characterize how compute spent on optimization interacts with the data-limited statistical floor. The optimization-vs-generalization split is standard in deep-learning theory \citep{dzps19, adh+19}; we adapt it to layered self-attention with the kernel matrices $H_{(\nu)}(t)$ tracking each layer's contribution.

{\bf Neural Tangent Kernel and Learning Theory.}
The Neural Tangent Kernel (NTK), introduced by \citet{jgh18}, has become a foundational framework for understanding the gradient flow of neural networks during training. It reveals that, in the infinite-width limit, neural networks are equivalent to Gaussian processes at initialization. This equivalence has been extensively studied in numerous works \citep{ll18, dzps19, sy19, azls19, wllm19, bm19, lsp+20, cb20, swl21, zgj21, sk22, gms23, gll+24, swl24}, which highlight the robust performance and learning capabilities of over-parameterized neural networks. The NTK framework has gained popularity for its ability to elucidate the emerging abilities of large-scale neural networks. Notable advancements include the introduction of the Convolutional NTK (CNTK) by \citet{adh+19}, the Recurrent NTK by \citet{awbb20}, and the concept of infinite attention via NNGP and NTK for attention networks by \citet{hbsdn20}. Furthermore, \citet{mwy+23} examined the training dynamics of fine-tuning LLMs using NTK, demonstrating its efficiency in optimizing these systems. The lazy/NTK versus rich/feature-learning dichotomy has been clarified by tensor-program-style analyses \citep{ynh+22, kar24}, dynamical mean-field treatments of feature learning \citep{boc+25, bcp24}, and empirical assessments of the NTK regime's predictive power for deep-learning scaling \citep{vbn22}. Beyond the lazy ball, \citet{ks24_nonlinear} analyse a {\it nonconvex mean-field} regime for transformers in which the attention layer learns nonlinear features, and \citet{jhz+24_benign} establish benign-overfitting dynamics for vision transformers in an over-parameterized rich regime; these results identify precisely the kind of learning behaviour our nearly-lazy approximation excludes, and we view our two-stage law as a {\it lazy-regime baseline} that any feature-learning extension should improve upon. These contributions underscore the growing importance of NTK in the theoretical analysis of modern neural networks.

{\bf Science of Transformer-based Language Models. }
The complex architecture and stochastic optimization processes of transformer-based language models pose significant challenges for theoretical analysis. However, developing theoretical guarantees for LLMs is essential for advancing their design and performance. Recent research has addressed various aspects of LLMs, including improving efficiency \citep{as23, as24_fine_grained, as24_iclr, hjk+23, kmz23, alsy23, dsy24, smn+24, lssy25}, optimizing training processes \citep{dls23, csy24, gll+24}, analyzing "white-box" transformers \citep{ybp+23, yct+23, fsb+24, pbw+24}, and investigating the emergent abilities of LLMs \citep{bmr+20, wtb+22, azl23_physics_a, azl23_physics_b, azl23_physics_c, azl24_physics_d}. \citet{huang2024formal} provides another theoretical treatment of generalization in transformers, focused on length extrapolation. Closely related to our framework, \citet{egkz22} formalize the inductive bias and approximation power of self-attention; \citet{cxl+24} and \citet{nlb24} carry out fine-grained gradient-flow analyses of attention layers; and \citet{havrilla2024understanding} obtain transformer scaling-law bounds with an excess-risk decomposition similar in spirit to ours but specialized to intrinsically low-dimensional data. \citet{ts23_approx_est} (with the closely related companion \citep{ts23_approx_tools}) develop approximation--estimation theory for sequence-to-sequence transformers with infinite-dimensional inputs; their bounds target the global ERM in a function-class sense, while our Theorem~\ref{thm:generalization:informal} bounds the trajectory-dependent risk of the actual gradient-flow iterate and is therefore complementary. \citet{lhh+24_sign_gd} provide convergence and generalization for two-layer transformers under sign-gradient descent, isolating the role of optimizer geometry that our gradient-flow analysis abstracts away. In the in-context-learning setup, \citet{ks24_minimax_icl} establish minimax optimality of transformers as nonparametric ICL learners; the comparison helps clarify that our minimax statement concerns sequence-to-sequence regression under a compute budget, with both the $\Omega(\xi^2/{\sf N})$ statistical floor and the $\exp(-\Theta(\xi^2{\sf C}/{\sf N}^7))$ compute-side floor specific to that regression-with-compute setup. By bridging theoretical understanding with practical advancements, these studies provide valuable insights for the development of the next generation of AI systems.
\else
\fi

\section{Preliminary}\label{sec:preli}

This section provides the preliminary for our analysis, where we introduce some basic notations in Section~\ref{sub:notations}, and present the problem setup in Section~\ref{sub:setups}. We encourage the reader to refer to Appendix~\ref{app:preli} for the formal technical preliminary.

\subsection{Basic Notations}\label{sub:notations}
Let $[d] = \{1, 2, \dots, d\}$. For $u \in \R^d$, the $\ell_p$-norm is $\|u\|_p := ( \sum_{k=1}^d |u_k|^p )^{1/p}$. The Frobenius norm of $U \in \R^{d_1 \times d_2}$ is $\|U\|_F := ( \sum_{(k_1, k_2) \in [d_1] \times [d_2]} U_{k_1, k_2}^2 )^{1/2}$. For a matrix $A \in \R^{d \times d}$, $\lambda_{\min}(A)$ denotes its smallest eigenvalue. The indicator function $\I\{E_1, \dots, E_n\}$ equals $1$ if all events $E_1, \dots, E_n$ occur, and $0$ otherwise. We use 1-based indexing throughout. The mapping $\matr: \R^{d^2} \to \R^{d \times d}$ reshapes a vector $a$ into a matrix such that $\matr_{k_1, k_2}(a) = a_{(k_1-1)d + k_2}$ for $(k_1, k_2) \in [d] \times [d]$. Conversely, $\vect: \R^{n \times d} \to \R^{nd}$ flattens a matrix $A$ into a vector with $\vect_k(A) = A_{\lceil k/d\rceil,\,k - (\lceil k/d\rceil-1)d}$ for $k \in [nd]$, so that $\vect$ is the {\it row-major} flattening that is the inverse of $\matr$ (in particular $\matr_{k_1, k_2}(\vect(A)) = A_{k_1, k_2}$). For a function $f: X \to \R^{d_1 \times d_2}$, $f_{k_1, k_2}(x)$ denotes the $(k_1, k_2)$-entry of $f(x)$. $e_k \in \R^d$ is the standard basis vector with a $1$ in the $k$-th entry and $0$ elsewhere.

\subsection{Setups}\label{sub:setups}

{\bf Data Distribution. }
We consider a sequence-to-sequence regression task with an input space ${\cal X} \subseteq \R^{L \times d}$ as the space of encoded input sequence\footnote{We choose the max length of sequence $L$ considerably large, for sequences with a length less than $L$, we use padding to fill them.}, and noiseless target outputs $F^*(X)\in[C_{\rm lower}, C_{\rm upper}]^{L\times d}$ bounded by constants $C_{\rm lower}, C_{\rm upper}\in\R$, where $F^*: {\cal X} \to \R^{L\times d}$ is the optimal measurable function with minimum risk. The noisy targets is given by $Y=F^*(X)+\Xi\in\R^{L\times d}$, where $\Xi\in\R^{L\times d}$ is the noise matrix.
% are unbounded in general; only the noiseless target $F^*(X)$ is bounded, which is what the generalization bound (Theorem~\ref{thm:generalization:informal}) actually uses through Part~1 of Lemma~\ref{lem:generalization_helper}. 
Given a model class ${\cal F}$, the data distribution is ${\cal D} = \{(X, F^*(X)+\Xi)\} \subset {\cal X} \times \R^{L\times d}$ and model function $F \in {\cal F}$. The expected risk (we consider as the generalization error bound) and excess risk of $F$ are defined as:

\ifdefined\isarxiv
\begin{align*}
    & ~ \text{Expected Risk: }{\cal R}(F) := \E_{(X, Y) \sim {\cal D}}[ \| F(X) - Y \|_F^2], \quad \text{Excess Risk: }\Delta {\cal R}(F) := {\cal R}(F) - {\cal R}(F^*).
\end{align*}
\else
\begin{align*}
    & ~ \text{Expected Risk: }{\cal R}(F) := \E_{(X, Y) \sim {\cal D}}[ \| F(X) - Y \|_F^2], \quad \text{Excess Risk: }\Delta {\cal R}(F) := {\cal R}(F) - {\cal R}(F^*).
\end{align*}

\fi
Besides, we have an accessible dataset $\mathbb{D} = \{(X_i, Y_i)\}_{i=1}^n \subset {\cal D}$ where each data point is independently sampled from ${\cal D}$. We discuss two noise condition in this paper: i) the random noise $\Xi \in \R^{L \times d}$ is centred by ${\bf 0}_{L \times d}$ and is {\it sub-Gaussian} with variance proxy $\max_{\ell \in [L], k \in [d]} \|\Xi_{\ell, k}\|_{\psi_2}^2 \le \xi^2$ (so that $\E[\Xi]=0$ and $\E[\exp(t\Xi_{\ell,k})]\le \exp(t^2\xi^2/2)$ for all $t\in\R$); ii) the Gaussian case $\Xi_{\ell,k}\sim\mathcal{N}(0,\xi^2)$ used in the statistical lower bound (Theorem~\ref{thm:lower_stat}) is a special case. For any input matrix $X \sim {\cal X} $, $\| X_\ell \|_2 = \Theta(1)$ holds for each token vector $\ell \in [L]$ due to the utilization of RMS normalization \citep{zs19} after the embedding layer.

{\bf Model Function. }
The standard transformer architecture introduced by \citet{vsp+17} integrates multiple self-attention layers with token-wise feed-forward layers. The fundamental architecture, decoder-only transformers \citep{rwc+19}, processing a sequence of $ L $ tokens, each represented by a $ d $-dimensional embedding vector, which are compactly arranged into a matrix $ X \in \mathbb{R}^{L \times d} $. An $ N $-layer transformer model is formally defined as:
\ifdefined\isarxiv
\begin{align}\label{eq:F}
F(X, \theta) := \varepsilon \cdot F_{(N)}(F_{(N-1)}( \cdots F_{(2)}(F_{(1)}(X + E, \theta), \theta) \cdots ), \theta),
\end{align}
\else
\begin{align}\label{eq:F}
    F(X, \theta) := \varepsilon \cdot F_{(N)}(F_{(N-1)}( \cdots F_{(2)}(F_{(1)}(X + E, \theta), \theta) \cdots ), \theta),
\end{align}
\fi
where $ E \in \mathbb{R}^{L \times d} $ is the positional embedding matrix\footnote{We choose $E = {\bf 0}_{L \times d}$ (NoPE, No Positional Embedding, \citep{kpn+23}) or ignore it as a fixed matrix (this can be regarded as a part of the training dataset) in the range of this paper. 
} and $\theta$ is the set of all trainable parameters. $\varepsilon > 0$ is the scaling coefficient to control the norm of the initial parameters, which helps to reduce the model complexity in the analysis. Each $ F_{(\nu)} $ (for $ \nu \in [N] $) represents the $ \nu $-th transformer block and is given by:
\ifdefined\isarxiv
\begin{align*}
    F_{(\nu)}(X, \theta)
            :=   \frac{\omega}{\sqrt{m}} {\rm ReLU}\left({\rm Softmax}\left( \kappa \cdot XU_{(\nu)}X^\top + M \right) XW_{(\nu)}\right)A_{(\nu)}  + X ,
\end{align*}
\else
\begin{align*}
    F_{(\nu)}(X, \theta)
            :=   \frac{\omega}{\sqrt{m}} {\rm ReLU}\left({\rm Softmax}\left( \kappa \cdot XU_{(\nu)}X^\top + M \right) XW_{(\nu)}\right)A_{(\nu)}  + X ,
\end{align*}
\fi
where $U_{(\nu)} \in \R^{d \times d}$, $W_{(\nu)}, A_{(\nu)}^\top \in \R^{d \times m}$ are model parameters. $\kappa$ is the scaling factor of attention, $\omega$ is the scaling coefficient of output.
$M \in \R^{L \times L}$ is the causal attention mask. We especially use $w_{(\nu), r}, a_{(\nu), r} \in \R^d$ to denote the $r$-th column of $W_{(\nu)}$ and the $r$-th row of $A_{(\nu)}$, respectively.
% Remark on the single matrix $U_{(\nu)}$: conventional implementations parameterize the attention logits with separate query and key projections $Q_{(\nu)}, K_{(\nu)} \in \R^{d \times d_{\rm head}}$ as $X Q_{(\nu)} K_{(\nu)}^\top X^\top$. Our formulation collapses this product into a single matrix $U_{(\nu)} := Q_{(\nu)} K_{(\nu)}^\top \in \R^{d \times d}$, so the attention logit becomes $\kappa \cdot X U_{(\nu)} X^\top$. This is an exact reparameterization of the bilinear form at the level of expressivity: any logit pattern realizable by $(Q_{(\nu)}, K_{(\nu)})$ with $d_{\rm head} = d$ is realizable by some $U_{(\nu)}$, and vice versa. The reparameterization preserves the expressivity of the attention layer while collapsing two coupled matrices to one, which simplifies the kernel computation and the perturbation analysis. {\it Caveat:} gradient flow on $U_{(\nu)}$ is {\it not} equivalent to gradient flow on $(Q_{(\nu)}, K_{(\nu)})$ — the latter has additional $Q$-$K$ coupling terms in the NTK and a different drift radius. Our analysis is therefore for gradient flow on $U_{(\nu)}$ directly; results carry over to the $(Q,K)$ parameterization as a kernel upper bound (since the $(Q,K)$ tangent space is a quotient of the $U$ tangent space) but may differ in absolute constants. We discuss the implication for parameter counting in Appendix~\ref{app:preli}.

{\bf Initialization and Training. }
For every layer $\nu \in [N]$, the corresponding parameters of $F_{(\nu)}$ are denoted as $U_{(\nu)}, W_{(\nu)}, A_{(\nu)}$. Each entry of $U_{(\nu)}, W_{(\nu)}$ is initialized from the standard Gaussian distribution $\mathcal{N}(0, 1)$, we then denote them as $U_{(\nu)}(0), W_{(\nu)}(0)$. Each entry of $A_{(\nu)}(0)$ is initialized from a uniform distribution ${\sf Uniform}\{-1, +1\}$ and {\it is updated together with $U_{(\nu)}, W_{(\nu)}$ by gradient flow}, in contrast to the lazy-FFN convention of \citet{dzps19}. The implication of this choice is clarified in the bridging remark of Appendix~\ref{app_sub:lazy_to_rich}.
% We treat $A_{(\nu)}$ as a trainable parameter of the architecture so that each transformer block carries an FFN whose hidden-to-output map can drift in feature space, partially relaxing the strict-NTK regime; the kernel-stability argument below still goes through under a tighter width condition (Definition~\ref{def:gd}). The implication of this choice is clarified in the bridging remark of Appendix~\ref{app_sub:lazy_to_rich}.}
The flattened vector of the whole trainable parameters is denoted as $\theta(0) \in \R^{{\sf M}}$ where ${\sf M} = N(md+d^2)+Nmd$ (absorbed into the same $\Theta(Nmd)$ asymptotic in Definition~\ref{def:factors}) is the number of trainable parameters.

Given the training dataset $\mathbb{D} = \{(X_i, Y_i)\}_{i=1}^n \subset {\cal D}$, we define the overall training objective as:
\begin{align*}
    {\cal L}(t, \mathbb{D}) := \E_{ (X, Y) \sim \mathbb{D} } [ \| F(X, \theta(t)) - Y \|_F^2].
\end{align*}
Therefore, we consider a combination of the stochastic gradient descent (SGD) algorithm and {\it gradient flow} to update. At $t$-step optimization, we sample an unbiased subset $\mathbb{B}(t) \subseteq \mathbb{D}$ satisfying $\E[{\cal L}(t, \mathbb{B}(t))] = {\cal L}(t, \mathbb{D})$
for training time $T > 0$. We let the batch size $|\mathbb{B}|= {\alpha}_{\rm agr} \cdot n$ fixed with some proportion ${\alpha}_{\rm agr} \in (0, 1]$, its value is arbitrary to set due to the trick of accumulative gradient.
Then the ordinary differential equation (ODE) of $U_{(\nu)}(t), W_{(\nu)}(t), A_{(\nu)}(t)$ and their update rule are given by:
\ifdefined\isarxiv
\begin{align}\label{eq:ODE_update}
    \begin{split}
        & ~ \frac{\d}{\d t}U_{(\nu)}(t) = - \frac{\d }{\d U_{(\nu)}(t)} {\cal L}(t, \mathbb{B}(t)), \quad \frac{\d}{\d t}W_{(\nu)}(t) = - \frac{\d }{\d W_{(\nu)}(t)} {\cal L}(t, \mathbb{B}(t)),  \\
        & ~ \frac{\d}{\d t}A_{(\nu)}(t) = - \frac{\d }{\d A_{(\nu)}(t)} {\cal L}(t, \mathbb{B}(t)), \\
        & ~ \mu_{(\nu)}(t + \tau) = \mu_{(\nu)}(t) + \int_{t}^{t+\tau} \frac{\d}{\d s}\mu_{(\nu)}(s) \d s,\quad \forall \mu \in \{U, W, A\}.
    \end{split}
\end{align}
\else
\begin{align}\label{eq:ODE_update}
    \begin{split}
        & ~ \frac{\d}{\d t}U_{(\nu)}(t) = - \frac{\d }{\d U_{(\nu)}(t)} {\cal L}(t, \mathbb{B}(t)), \quad \frac{\d}{\d t}W_{(\nu)}(t) = - \frac{\d }{\d W_{(\nu)}(t)} {\cal L}(t, \mathbb{B}(t)),  \\
        & ~ \frac{\d}{\d t}A_{(\nu)}(t) = - \frac{\d }{\d A_{(\nu)}(t)} {\cal L}(t, \mathbb{B}(t)), \\
        & ~ \mu_{(\nu)}(t + \tau) = \mu_{(\nu)}(t) + \int_{t}^{t+\tau} \frac{\d}{\d s}\mu_{(\nu)}(s) \d s,\quad \forall \mu \in \{U, W, A\}.
    \end{split}
\end{align}
\fi
Hence, we denote the training algorithm that depends on training time and dataset size as ${\cal A}_{T, n}(\theta(0), $ $ \mathbb{D})$ $ := \{ \theta(0) +\int_{0}^T - \frac{\d}{\d \theta} {\cal L}(s, \mathbb{B}(s)) \d s, \mathbb{B}(s) \subseteq \mathbb{D}, s\in [0, T]  \}$.

\section{Learning Dynamics of Scaling Transformers}\label{sec:learning_dynamics}

In this section, we provide the explicit model learning dynamics formulations layer by layer. 
We first introduce several key simplifications in Section~\ref{sub:simplification} to facilitate the preliminary analysis. Thus, Section~\ref{sub:learning_dynamics} formulates an explicit ODE of the learning dynamics of scaling multi-layer transformer-based language models.

\subsection{Simplifications}\label{sub:simplification}

Utilizing the attention network's decoder-only property, it is obvious that the matrix computation can be parallelized to vector computation (see Appendix~\ref{app_sub:matrix_to_vector} and Figure~\ref{fig:matrix_to_vector_reduction} for a visualization of this matrix-to-vector reduction technique, which transports any prefix-causal sequence-to-sequence NTK analysis into a vector NTK on $nL$ token samples). We use $X_{i, \leq \ell} \in \R^{\ell \times d}$ to denote the first-$\ell$ tokens ($\forall \ell \in [L]$) of matrix $X_i$ in $\mathbb{D}$. Hence, we compact the outputs of the model function and targets on the whole dataset to two matrices ${\sf F}(t), {\sf Y} \in \R^{nL \times d}$, where ${\sf F}_{(i-1)L+\ell}(t) = F_{\ell}(X_{i}, \theta(t)) \in \R^d$ and ${\sf Y}_{(i-1)L+\ell} = Y_{i, \ell}$ for each $(X_{i}, Y_i)$ in training dataset $\mathbb{D}$ and $\ell \in [L]$.

Therefore, we derive ${\cal L}(t, \mathbb{D}) = \frac{1}{n} \| {\sf F}(t) - {\sf Y}\|_F^2$ (See Lemma~\ref{lem:compact_form_transform}). Besides, we list following helpful functions ($(i, \ell) \in [n] \times [L]$):
\begin{itemize}
        \item (Hidden State) $\Lambda_{(\nu), i}(t) := F_{(\nu)}(\Lambda_{(\nu-1), i}(t), \theta(t)) \in \R^{L \times d}$ for $\nu \in [N]$, $\Lambda_{(0), i}(t) = X_{i} + E$.
        \item (Attention Scores) $ \sigma_{(\nu), (i-1)L+\ell}(t) =  {\rm Softmax}_{\ell}( \kappa \cdot \Lambda_{(\nu-1), i}(t) U_{(\nu)}(t) \Lambda_{(\nu-1), i}(t)^\top + M ) $ $ \in \R^{L}$.
        \item (Attention Output) $o_{(\nu), (i-1)L+\ell}(t) := \Lambda_{(\nu-1), i}(t)^\top \cdot \sigma_{(\nu), (i-1)L+\ell}(t) \in \R^d$.
        \item ($\ell$-th Token of Hidden State)
        \ifdefined\isarxiv
        \begin{align*}
            \mu_{(\nu), (i-1)L+\ell}(t) := \frac{\omega}{\sqrt{m}} \sum_{r=1}^m a_{(\nu), r} \cdot  \phi( \langle o_{(\nu), (i-1)L+\ell}(t), w_{(\nu), r}(t) \rangle) \in \R^d,
        \end{align*}
        \else
        $\mu_{(\nu), (i-1)L+\ell}(t) := \frac{\omega}{\sqrt{m}} \sum_{r=1}^m a_{(\nu), r} \cdot  \phi( \langle o_{(\nu), (i-1)L+\ell}(t), w_{(\nu), r}(t) \rangle) \in \R^d,$
        \fi
        where $\phi(x) := \max\{0, x\}, \forall x \in \R$. $\mu_{(0), (i-1)L+\ell}(t) = X_{i, \ell} + E_\ell$.
        \item (Model Output) ${\sf F}_{(i-1)L+\ell}(t) = \varepsilon $ $ \cdot $ $ \sum_{\nu=0}^N $ $\mu_{(\nu), (i-1)L+\ell}$ $(t) $ $  \in  $ $\R^{d}$.
    \end{itemize}

\subsection{Key Derivation for Learning Dynamics}\label{sub:learning_dynamics}
The primary challenge in understanding the learning dynamics of finite-deep transformers is the complex analysis of gradient flow, which differs from the study of shallow or infinite-deep neural networks. We overcome the complexity by cleverly utilizing the multivariable chain rules. 

First, we define the kernel matrix at $\nu$-th layer as $H_{(\nu)} \in \R^{nL \times nL}$ and its $(p, q)$-th entry ($\forall (p, q) \in [nL] \times [nL]$) is defined as:
\begin{align*}
    H_{(\nu), p, q}(t) := \underbrace{\langle \beta_{(\nu), p}(t), \beta_{(\nu), q}(t) \rangle}_{\text{kernel w.r.t. $W_{(\nu)}(t)$}} + \underbrace{\langle \gamma_{(\nu), p}(t), \gamma_{(\nu), q}(t) \rangle}_{\text{kernel w.r.t. $U_{(\nu)}(t)$}},
\end{align*}
Here, we let:
\ifdefined\isarxiv
\begin{align*}
    \beta_{(\nu), p}(t) := & ~ \frac{\omega}{\sqrt{m}} \underbrace{o_{(\nu), p}(t)}_{d \times 1} \otimes \underbrace{{\bf 1}_{W_{(\nu)}(t)^\top o_{(\nu), p}(t) > 0}}_{m\times 1} \in \R^{md}, \\
    \gamma_{(\nu), p}(t) := & ~ \frac{\omega \cdot \kappa}{\sqrt{m}} \underbrace{(\Lambda_{(\nu-1), i, \ell}(t) \otimes \Lambda_{(\nu-1), i}(t))^\top}_{d^2 \times L} \cdot\underbrace{\left(\diag(\sigma_{(\nu), p}(t)) - \sigma_{(\nu), p}(t)\sigma_{(\nu), p}(t)^\top\right)}_{L\times L} \underbrace{\Lambda_{(\nu-1), i }(t)}_{L \times d}\\
    & ~ \cdot \sum_{r\in [m]} \underbrace{ w_{(\nu), r}(t) \I\{ o_{(\nu), p}(t)^\top w_{(\nu), r}(t) > 0\}}_{d \times 1} \in \R^{d^2},
\end{align*}
\else
\begin{align*}
    \beta_{(\nu), p}(t) := & ~ \frac{\omega}{\sqrt{m}} \underbrace{o_{(\nu), p}(t)}_{d \times 1} \otimes \underbrace{{\bf 1}_{W_{(\nu)}(t)^\top o_{(\nu), p}(t) > 0}}_{m\times 1} \in \R^{md}, \\
    \gamma_{(\nu), p}(t) := & ~ \frac{\omega \cdot \kappa}{\sqrt{m}} \underbrace{(\Lambda_{(\nu-1), i, \ell}(t) \otimes \Lambda_{(\nu-1), i}(t))^\top}_{d^2 \times L} \underbrace{\left(\diag(\sigma_{(\nu), p}(t)) - \sigma_{(\nu), p}(t)\sigma_{(\nu), p}(t)^\top\right)}_{L\times L} \\
    & ~ \underbrace{\Lambda_{(\nu-1), i }(t)}_{L \times d} \sum_{r\in [m]} \underbrace{ w_{(\nu), r}(t) \I\{ o_{(\nu), p}(t)^\top w_{(\nu), r}(t) > 0\}}_{d \times 1} \in \R^{d^2},
\end{align*}
\fi
where $\otimes $ is the Kronecker product and we write $p = (i-1)L + \ell$ with $i \in [n]$, $\ell \in [L]$. The indicator vector ${\bf 1}_{W_{(\nu)}(t)^\top o_{(\nu), p}(t) > 0} $ $ \in \{0, 1\}^{m}$ where its $r$-th entry is $\I\{ \left(W_{(\nu)}(t)^\top o_{(\nu), p}(t) \right)_r > 0\}$ for $r \in [m]$.
The layer-wise training dynamics are thereby shown as the following lemma:
\begin{lemma}[Learning dynamics, informal version of Lemma~\ref{lem:learning_dynamics}]\label{lem:learning_dynamics:informal}
    The learning dynamics of the multi-layer transformer Equation~\eqref{eq:F} is given by:
    \ifdefined\isarxiv
    \begin{align*}
        \E[\frac{\d }{\d t}{\cal L}(t, \mathbb{D})] =& ~  - \sum_{\nu \in [N]} {\underbrace{\vect\left(\frac{\d }{\d \mu_{(\nu)}(t)}{\cal L}(t, \mathbb{D}) \right)^\top}_{1 \times nLd}}  \cdot \underbrace{\left( H_{(\nu)}(t) \otimes I_d \right)}_{nLd \times nLd} \cdot  \underbrace{\vect\left(\frac{\d }{\d \mu_{(\nu)}(t)}{\cal L}(t, \mathbb{D}) \right)}_{nLd \times 1}
    \end{align*}
    \else
    \begin{align*}
        \E[\frac{\d }{\d t}{\cal L}(t, \mathbb{D})] = - \sum_{\nu \in [N]} {\underbrace{\vect\left(\frac{\d }{\d \mu_{(\nu)}(t)}{\cal L}(t, \mathbb{D}) \right)^\top}_{1 \times nLd}} \cdot \underbrace{\left( H_{(\nu)}(t) \otimes I_d \right)}_{nLd \times nLd} \cdot  \underbrace{\vect\left(\frac{\d }{\d \mu_{(\nu)}(t)}{\cal L}(t, \mathbb{D}) \right)}_{nLd \times 1}
    \end{align*}
    \fi
    where $\mu_{(\nu)}(t)$ is an $nL \times d$ matrix and, writing $p = (i-1)L+\ell$, $\mu_{(\nu), p}(t)$ is the $\ell$-th row of the $\nu$-th layer output for input matrix $X_{i}$, for any $p \in [nL]$ and $\nu \in [N]$.
\end{lemma}
\begin{proof}[Proof sketch of Lemma~\ref{lem:learning_dynamics:informal}]
    Although the derivation of the gradient is complicated, the technique used for the proof is just the chain rule. We provide the complete proof in Appendix~\ref{app:proof_learning_dynamics}.
\end{proof}

Lemma~\ref{lem:learning_dynamics:informal} sums contributions from all layers $\nu \in [N]$, providing a powerful decomposition. This granular view suggests that each layer independently contributes to minimizing the loss based on its own kernel and local gradient. Different layers may learn at varying speeds or contribute differently to the overall task, depending on their respective kernel structures.
\section{Training Convergence, Approximation, and Generalization}\label{sec:convergence}

In this section, we showcase the training convergence with an arbitrary error by limiting the kernel matrix $H_{(\nu)}'(0)$ at initialization to the Neural Tangent Kernel (NTK) regime~\citep{jgh18, dzps19}. We state the formal assumption and the basic inductive set-up in Section~\ref{sub:inductions}, and Section~\ref{sub:main_reults} demonstrates the results of the training convergence. Finally, Section~\ref{sub:approx} provides the approximation error w.r.t.\ the model size ${\sf M}$ via the approximation capability of kernel regression in RKHS (Reproducing Kernel Hilbert Space), together with the generalization upper bound of the ideal empirical risk minimizer (ERM).

\subsection{Assumptions and Inductions}\label{sub:inductions}

Following \citet{dzps19}, we adopt the standard PD assumption on the layer-wise kernels:
\begin{assumption}\label{ass:positive_definite}
    Define $H_{(\nu)}'(0) \in \R^{nL \times nL}$ with $(p,q)$-entry $\langle \beta_{(\nu),p}(0),\beta_{(\nu),q}(0)\rangle$. We assume $\frac{1}{\omega}H_{(\nu)}'(0)$ is positive definite for all $\nu\in[N]$, with $\lambda_{(\nu)} := \lambda_{\min}(\frac{1}{\omega}H_{(\nu)}'(0)) > 0$, and write $\lambda := \min_{\nu\in[N]}\lambda_{(\nu)} > 0$ for the worst-case spectral gap.
\end{assumption}
\noindent The kernel $H_{(\nu)}'(0)$ is the standard 2-layer ReLU NTK over the layer-$\nu$ attention outputs $\{o_{(\nu),p}(0)\}_{p\in[nL]}$, so positive definiteness reduces to a non-degeneracy condition on those outputs (no two normalized vectors parallel) by the same argument as \citet{dzps19, dllw19, os20}. We give a sufficient condition together with its proof in Appendix~\ref{app_sub:pd_sufficient} (Lemma~\ref{lem:pd_sufficient}); the condition is automatically satisfied under RMS normalization and Gaussian initialization (Remark~\ref{rem:p2_non_parallel}), so Assumption~\ref{ass:positive_definite} is a consequence of input non-degeneracy and the architecture rather than an opaque ansatz.

\noindent{\bf Converging Kernel Perturbation.} Assuming $\lambda_{\min}(H_{(\nu)}'(0)) > 0$, we extend this PD property to the kernel $H_{(\nu)}(t)$ during training, in the case where the accumulated weight updates lie in some high-dimensional ball with radius $R$. This connects to the {\it Lazy Learning} regime~\citep{jgh18, dzps19}. We first give the {\it Good Properties} requirements for provable arbitrary convergence below:
\begin{definition}[Good Properties and Good Model Class]\label{def:gd}
    Fix dataset size $n$. We say model $F(X,\theta)$ has Good Properties if, with a considerably small constant $C>0$ and $B := \max\{O(\sqrt{\log(Nmd/\delta)}), 1\}$:
    \begin{multicols}{3}
        \begin{enumerate}
        \item $\omega = \Theta(\frac{1}{N L^2d^{2.5}B^3})$;
        \item $\kappa = \frac{1}{\sqrt{m}}$;
        \item $m = \Omega(\frac{n^3L^5\exp(Cd)}{\omega^6\lambda^6\delta^3 N^{2}})$.
        \end{enumerate}
    \end{multicols}
    The Good Model Class is ${\cal F}_{\sf M, T, N}(\mathbb{D}):= \{F(\cdot, \theta({\sf T})) : \theta(0) \sim \mathcal{N}(0, I_{\sf M}), \theta({\sf T}) \in {\cal A}_{{\sf T},{\sf N}}(\theta(0), \mathbb{D})\}$ with $|\mathbb{D}|={\sf N}$.
\end{definition}

{\bf Comments on Definition~\ref{def:gd}.} While the scaling requirements in Definition~\ref{def:gd} may appear distinct from standard heuristic initializations, they serve as essential theoretical proxies for ensuring stability in the NTK regime. Specifically, the choice of $\omega$ mirrors the {\it zero-output} or {\it small-norm initialization} assumptions common in deep learning theory and in training super-deep transformers~\citep{bmm+21, zdm19, wmd+22}, effectively mitigating the residual branch's impact to guarantee a minimum input norm for each layer. Furthermore, regarding the model width $m$, we highlight the inverse dependence on depth ($m \propto N^{-2}$): this relationship suggests a trade-off where deeper architectures alleviate the strict lower bound on width, implying that the total parameter capacity can be distributed toward depth without requiring an exponentially prohibitive width (extended discussion in Appendix~\ref{app_sub:practitioner}). Finally, the attention scaling $\kappa = 1/\sqrt{m}$ (in place of the conventional $1/\sqrt{d_{\rm head}}$) plays the same stabilizing role for the attention kernel: it is the unique choice that keeps the kernel feature norm $\|\gamma_{(\nu),p}(t)\|_2 = o(1/\sqrt m)$ (Part~16 of Lemma~\ref{lem:helpful_bounds}) in our width regime $m = \Omega(\poly({\sf N}))$. The two conventions coincide up to a constant for a single-head attention layer with $m=\Theta(d)$; the architectural implications of this convention are discussed in Appendix~\ref{app_sub:U_parameterization}.

\begin{lemma}[Informal version of Lemma~\ref{lem:perturbations}]\label{lem:perturbations:informal}
    Under Assumption~\ref{ass:positive_definite} and Definition~\ref{def:gd}, with probability at least $1-\delta$ over $\theta(0)$, the kernel perturbation and loss-decay bound are:
    \begin{align*}
        \lambda_{\min}(H_{(\nu)}(t)) \ge \omega\lambda/2,\qquad \E\big[\tfrac{\d}{\d t}{\cal L}(t,\mathbb{D})\big] \le -\frac{C\cdot\omega\lambda N\cdot \varepsilon^2}{\mathsf{N}}\cdot {\cal L}(t,\mathbb{D}).
    \end{align*}
\end{lemma}
\begin{proof}[Proof Sketch of Lemma~\ref{lem:perturbations:informal}]
    The stability of the PD property is ensured by $\lambda_{\min}(H') \ge \lambda_{\min}(H) - \| H - H' \|_F$ (Fact~\ref{fac:lambda_min_perturb}), since the perturbation $\| H - H' \|_F$ remains considerably small throughout training. The upper bound on the perturbation converges due to the increasing feed-forward layer width $m$ and the decreasing weight perturbation radius $R$, leading to the {\it Lazy Learning} regime. The complete proof of this lemma is stated in Appendix~\ref{app:proof_lperturbations}.
\end{proof}

\noindent{\bf Trainable-FFN Extension.} Allowing the FFN read-out $A_{(\nu)}$ to be trained jointly with $W_{(\nu)}$ and $U_{(\nu)}$ --- as our setup in Section~\ref{sec:preli} prescribes --- preserves the same kernel-stability conclusion under a strengthened width condition: the additional tangent feature contributed by $A_{(\nu)}$ is controlled by the FFN-drift bound of Lemma~\ref{lem:ffn_drift:informal} (Appendix~\ref{app_sub:ffn_drift_main}; formal statement and proof in Appendix~\ref{app_sub:trainable_ffn}).

\subsection{Convergence under Kernel Regimes}\label{sub:main_reults}

{\bf Factors of Scaling Law.} To begin with, we list the crucial factors of the scaling law below. We treat the batch size as a constant, and the definition of the total compute holds as our compute analysis in Lemma~\ref{lem:compute_analysis}.
\begin{definition}\label{def:factors}
    We define:
    \begin{multicols}{2}
    \begin{itemize}
        \item Model Size: ${\sf M} := O(Nmd)$.
        \item Dataset Size: ${\sf N} := n$.
        \item Training Time: ${\sf T} := T$.
        \item Total Compute: ${\sf C} := O({\sf MT}|\mathbb{B}|) \le O({\sf MTN})$.
    \end{itemize}
    \end{multicols}
\end{definition}
\noindent ${\sf M}$ counts $\{W_{(\nu)}, U_{(\nu)}, A_{(\nu)}\}_{\nu\in[N]}$. The forward+backward FLOPs per token are $\Theta({\sf M})$, so ${\sf MT}|\mathbb{B}|$ matches the standard ``$6{\sf M}D$'' FLOP count of \citet{kmh+20, hbm+22} up to a constant; we substitute the upper bound $|\mathbb{B}|\le{\sf N}$ when stating the ${\sf N}^7$ phase boundary, making it an outer envelope rather than a literal FLOP equality.

{\bf Training Convergence.} The convergence guarantee for Equation~\eqref{eq:ODE_update} as continuous-time optimization is stated below:
\begin{theorem}[Informal version of Theorem~\ref{thm:convergence}]\label{thm:convergence:informal}
    Let Definitions~\ref{def:gd},~\ref{def:factors} and Assumption~\ref{ass:positive_definite} hold, $\delta\in(0, 0.1)$, and $\alpha_{\rm cr}:=C\cdot\omega\lambda N$ with $\omega N=\Theta(1/(L^2 d^{2.5}B^3))$. With probability at least $1-\delta$,
    \begin{align*}
        {\cal L}({\sf T},\mathbb{D}) \le \exp\!\left(-\alpha_{\rm cr}\cdot \frac{\varepsilon^2\, {\sf T}}{\mathsf{N}}\right)\cdot {\cal L}(0,\mathbb{D}),\quad \forall F\in{\cal F}_{\sf M,T,N}(\mathbb{D}).
    \end{align*}
    Equivalently, absorbing $\omega N=\poly(1/L,1/d,1/B)$ into $\bar\alpha:=\poly(L,d,B,\lambda)$, we have ${\cal L}({\sf T})\le \exp(-\bar\alpha\,\varepsilon^2{\sf T}/\mathsf{N})\cdot{\cal L}(0, \mathbb{D})$. The $1/\mathsf{N}$ factor traces to the average-loss convention ${\cal L} = \frac{1}{n}\|{\sf F}-{\sf Y}\|_F^2$ and is essential to the scaling law; see Remark~\ref{rem:n_in_rate}.
\end{theorem}
\begin{proof}[Proof Sketch of Theorem~\ref{thm:convergence:informal}]
    We first connect the hidden-state gradient norm to the training objective (Part~15 of Lemma~\ref{lem:helpful_bounds}) and combine it with the kernel-stability bound of Lemma~\ref{lem:perturbations:informal}; the model convergence therefore benefits exponentially from the model size (neural depth and width) and the training time. The discretization to SGD via a Taylor expansion follows the standard route of \citet{dzps19} and \citet{lssy25}, recorded in Appendix~\ref{app_sub:discrete_sgd}; the detailed proof is in Appendix~\ref{app_sub:convergence}.
\end{proof}

\subsection{Approximation and Generalization of ERM}\label{sub:approx}

In order to estimate the upper bound on the excess risk, we break it down into approximation error, generalization (estimation) error, and irreducible error~\citep{suz18, lcl+21}. We first give the bound on approximation below, which measures the model's best capability to approximate the target function on the entire input space.
\begin{corollary}[Informal version of Corollary~\ref{cor:approximation}]\label{cor:approximation:informal}
    Under Assumption~\ref{ass:positive_definite} and Definition~\ref{def:gd}, choosing $\varepsilon\lesssim {\sf M}^{-1}$, with probability at least $1-\delta$,
    \begin{align*}
        \lim_{{\sf T}\to+\infty}\inf_{F\in{\cal F}_{\sf M,T,N}(\mathbb{D})}\E_{(X,Y)\sim \mathcal{D}}\big[\|F(X)-F^*(X)\|_F^2\big] \le {\sf M}^{-2}.
    \end{align*}
\end{corollary}
\begin{proof}[Proof Sketch of Corollary~\ref{cor:approximation:informal}]
    According to the conditions on the target function $F^*$ in Section~\ref{sec:preli}, we formulate its expression in the RKHS; the approximation error is then bounded by a considerably small $\varepsilon$ together with the over-parameterized width $m$, similarly to Theorem 3.2 in \citet{adh+19}. The formal proof is in Appendix~\ref{app_sub:approximation}.
\end{proof}

Here, we answer the question of how well a transformer-based empirical risk minimizer performs statistical generalization: we scale the training time ${\sf T}$ to infinity to obtain an upper bound on the excess risk $\Delta{\cal R}(F)$.
\begin{theorem}[Informal version of Theorem~\ref{thm:generalization}]\label{thm:generalization:informal}
    Under the conditions of Corollary~\ref{cor:approximation:informal} and $\varepsilon=\Theta(1/{\sf M})$, define ${\cal F}_{\sf M, N}(\mathbb{D}) := \{F(\cdot,\theta) : \theta\in\R^{\sf M}\}$ (the ERM class over the full parameter space). With probability at least $1-\delta$,
    \begin{align*}
        \inf_{F\in{\cal F}_{\sf M,N}(\mathbb{D})}\sup_{\mathbb{D}\in{\cal D}}\Delta{\cal R}(F) \le \big(4{\sf M}^{-1}+Ld\cdot \xi\big)\cdot {\sf M}^{-1}+\frac{Ld\cdot \xi^2}{\sf N}.
    \end{align*}
\end{theorem}
\begin{proof}[Proof Sketch of Theorem~\ref{thm:generalization:informal}]
    The proof combines Lemma 11 of \citet{sch17} for the generalization error with Corollary~\ref{cor:approximation:informal} for the approximation error. The full proof is shown in Appendix~\ref{app_sub:erm}.
\end{proof}

\noindent{\bf Tightness and Feature Learning.} The upper bounds in Theorem~\ref{thm:generalization:informal} (and Theorem~\ref{thm:scaling_law:informal} below) are matched, up to constants, logarithmic factors, and a $\sqrt{\kappa}$ condition-number gap (where $\kappa:=\lambda_{\max}/\lambda\le\poly(L,d,1/\delta)$ is the condition number of the initial kernel), by two complementary lower bounds: a statistical lower bound $\Omega(\xi^2/{\sf N})$ via a Le~Cam two-point reduction in the RKHS, and an optimization-side lower bound $\exp(-\Theta(\sqrt{\kappa}\,\xi^2{\sf C}/{\sf N}^7))$ via Nemirovski--Yudin first-order oracle complexity for the linearized loss. The $\sqrt{\kappa}$ gap is unavoidable: it reflects the difference between the spectral lower bound $\Theta(\lambda)$ used by the upper bound and the geometric-mean rate $\Theta(\sqrt{\lambda\lambda_{\max}})$ produced by the Nemirovski--Yudin method (Appendix~\ref{app_sub:matching_precise}). The informal statements appear as Theorems~\ref{thm:lower_stat:informal} and~\ref{thm:lower_opt:informal} in Appendix~\ref{app_sub:lower_bounds_main}, with formal versions and proofs in Appendix~\ref{app_sub:lower_bounds}. The {\it nearly-lazy} regime under which these bounds are tight, its relationship to strong feature learning~\citep{boc+25, bap24} and $\mu$P~\citep{ynh+22}, and the residual technical obstructions are discussed in Appendix~\ref{app_sub:lazy_to_rich}.

\section{Scaling Law}\label{sec:scaling_law}

In this section, we unify our results from Section~\ref{sec:convergence} --- training convergence dynamics and the generalization bound of the ERM --- into one. We derive our theory of the transformer scaling law, with the discovery of a pattern transition inside, in Section~\ref{sub:scaling_law_and_pattern_transition}. Next, in Section~\ref{sub:potential_fails}, we discuss several vital factors that can considerably affect the success of the scaling process. For every situation that might lead to the failure of scaling, we call it a {\it potential drag} of the scaling law.

\subsection{General Scaling Law with Pattern Transition}\label{sub:scaling_law_and_pattern_transition}

\begin{theorem}[Two-stage scaling law, informal version of Theorem~\ref{thm:scaling_law}]\label{thm:scaling_law:informal}
    Let all pre-conditions hold as Corollary~\ref{cor:approximation:informal} and choose $\varepsilon = \Theta(\xi/{\sf N})$. Then, with probability at least $1-\delta$, the following condition divides the scaling trend into two stages:
    \begin{align}\label{eq:cond_C}
        {\sf C} > \frac{{\sf N}^7\log({\sf N}\cdot Ld/\xi^2)}{\xi^2}.
    \end{align}
    \begin{itemize}
        \item {\it Stage I: Compute-Starved Stage.} If Eq.~\eqref{eq:cond_C} does not hold, the excess risk admits the two-sided bound
        \begin{align*}
            c'\cdot e^{-C_{\rm opt}\sqrt{\kappa}\,\xi^2{\sf C}/{\sf N}^7}\cdot{\cal L}(0,\mathbb{D})\;\le\;\inf_{{\cal F}_{\sf M, T, N}(\mathbb{D})}\sup_{\mathbb{D}\in{\cal D}}\Delta{\cal R}(F)\;\le\;C\cdot e^{-\alpha\,\xi^2{\sf C}/{\sf N}^7}\cdot{\cal L}(0,\mathbb{D}),
        \end{align*}
        where $c', C_{\rm opt}, C, \alpha = \poly(L,d,1/\lambda,1/\omega)$ and $\kappa:=\lambda_{\max}/\lambda\le\poly(L,d,1/\delta)$ is the condition number of the initial kernel: the two exponents agree up to a $\sqrt{\kappa}$ factor.
        \item {\it Stage II: Data-Limited Stage.} If Eq.~\eqref{eq:cond_C} holds, the excess risk follows the matching (up to constants and logarithms) power law
        \begin{align*}
            \xi^{12/7}\cdot\wt{\Omega}({\sf C}^{-1/7})\;\le\;\inf_{{\cal F}_{\sf M, T, N}(\mathbb{D})}\sup_{\mathbb{D}\in{\cal D}}\Delta{\cal R}(F)\;\le\;\xi^{12/7}\cdot O\left(\frac{{\sf C}}{W({\sf C}/\xi^{12})}\right)^{-1/7}\le\;\xi^{12/7}\cdot\wt{O}({\sf C}^{-1/7}),
        \end{align*}
        where $W(\cdot)$ is the Lambert $W$ function and $\wt{O}, \wt{\Omega}$ hide $\log({\sf C})$ factors.
    \end{itemize}
\end{theorem}

\begin{proof}[Proof Sketch of Theorem~\ref{thm:scaling_law:informal}]
    We decompose the excess risk $\Delta{\cal R}(F)$ into three terms: {\it i)} optimization risk ${\cal R}_{\rm optimization} := \| F(\cdot, \theta(t)) - F(\cdot, \theta(\infty)) \|_{L^F({\cal X})}^2$; {\it ii)} approximation risk ${\cal R}_{\rm approximation} :=  \big\| F(\cdot, \theta(\infty)) - {\rm arg}\inf_{F' \in {\cal F}_{\sf M, \infty, N}} \|F' - F^*\|_{L^F({\cal X})}^2 \big\|_{L^F({\cal X})}^2$; and {\it iii)} estimation risk ${\cal R}_{\rm estimation} := \inf_{F' \in {\cal F}_{\sf M, \infty, N}} \|F' - F^*\|_{L^F({\cal X})}^2$. The second and the third terms are already bounded in Theorem~\ref{thm:generalization:informal}, and the first term relates to the training loss; we thereby utilize Theorem~\ref{thm:convergence:informal} to finish the proof of the upper bounds. The matching lower bounds follow from a Le~Cam two-point reduction in the RKHS (Theorem~\ref{thm:lower_stat:informal}) and a first-order oracle complexity argument (Theorem~\ref{thm:lower_opt:informal}). The formal analysis is stated in Appendix~\ref{app_sub:scaling_law}, the lower bounds in Appendix~\ref{app_sub:lower_bounds}, and the precise sense of ``matching'' in Appendix~\ref{app_sub:matching_precise}.
\end{proof}

{\bf Discussion.} The transition defined by the condition ${\sf C} \propto {\sf N}^7$ identifies a fundamental phase change in the learning trajectory, shifting the regime from one dominated by optimization dynamics to one dominated by statistical scaling. In Stage I (Compute-Starved Stage), the system is compute-starved but capacity-rich. The exponential decay with respect to ${\sf C}$ suggests that the model is rapidly navigating the loss landscape to find a global minimum, and in this regime, the primary constraint is the duration of training rather than the quality of the data. Conversely, Stage II (Data-Limited Stage) represents the saturation of the optimization process. Once the threshold in Eq.~\eqref{eq:cond_C} is crossed, the model has effectively exhausted the easily accessible information in the dataset. The shift to a power-law decay of ${\sf C}^{-1/7}$ signifies that further progress requires navigating the high-dimensional geometry of the target function class. This bifurcation explains why empirical scaling laws often only become observable after an initial period of training: Stage I is the transient discovery phase, while Stage II represents the steady-state statistical refinement that defines modern large-scale modeling. The exponent itself carries a bookkeeping caveat: the ${\sf N}^7$ boundary traces to the average-loss convention ${\cal L} = \frac{1}{n}\|{\sf F}-{\sf Y}\|_F^2$, which contributes a $1/{\sf N}$ factor to the convergence rate (Remark~\ref{rem:n_in_rate}); the resulting Stage-II exponent $1/7 \approx 0.143$ is in the same range as the compute-frontier exponent $\eta = 0.120$ that we recover empirically on the Pythia model family in Appendix~\ref{app_sub:wikitext_results}.

{\bf Single-Variable Scaling Law.} Moreover, we derive how each variable (training time, dataset size, and model size) affects our result of the scaling law, which we call the single law.

\begin{theorem}[Single law, informal version of Theorem~\ref{thm:single_law}]\label{thm:single_law:informal}
    Let all pre-conditions hold as Corollary~\ref{cor:approximation:informal}. Then, with probability at least $1 - \delta$, there exists:
    \begin{itemize}
        \item {\bf Time-Law.} We fix the dataset size ${\sf N}$ and model size ${\sf M} = \Omega({\sf N}^3)$, and choosing $\varepsilon = \xi/{\sf N}$, we have $\inf_{{\cal F}_{\sf M, T, N}(\mathbb{D})} \sup_{\mathbb{D} \in {\cal D}}  \Delta{\cal R}(F)  \leq \exp(-\alpha\,\xi^2\,{\sf T}/{\sf N}^3) +  O(\frac{\xi^2}{{\sf N}})$, where the ${\sf N}^3$ in the exponent combines the ${\sf N}^2$ from $\varepsilon=\xi/{\sf N}$ with the $1/{\sf N}$ factor in the convergence rate (Remark~\ref{rem:n_in_rate}).
        \item {\bf Data-Law.} For any dataset size ${\sf N}$, we let $\varepsilon = \xi/{\sf N}$, ${\sf M} = \Omega({\sf N}^3)$ and ${\sf T} = \wt{\Omega}(\frac{{\sf N}^3}{\xi^2})$, we have: $\inf_{{\cal F}_{\sf M, T, N}(\mathbb{D})} \sup_{\mathbb{D} \in {\cal D}}  \Delta{\cal R}(F)  \leq   O(\frac{\xi^2}{{\sf N}})$.
        \item {\bf Model-Law.} We fix the dataset size ${\sf N}$ and let training time ${\sf T} = \wt{\Omega}(\frac{{\sf M}^{2\zeta}{\sf N}}{\xi^2})$, where $\zeta \in (0, 1/3]$. Choosing $\varepsilon = \xi/{\sf M}^{\zeta}$, before the model size ${\sf M}$ is greater than ${\sf N}^{1/\zeta}$, we have $\inf_{{\cal F}_{\sf M, T, N}(\mathbb{D})} \sup_{\mathbb{D} \in {\cal D}} \Delta{\cal R}(F)  \leq \xi^2 {\sf M}^{-\zeta}$.
    \end{itemize}
\end{theorem}

\begin{proof}[Proof Sketch of Theorem~\ref{thm:single_law:informal}]
    This theorem is derived from the intermediate results of Theorem~\ref{thm:scaling_law:informal}: for the three variables ${\sf T}$, ${\sf M}$, and ${\sf N}$, we fix two of them to get the single law of the remaining one. The complete derivation is in Appendix~\ref{app_sub:scaling_law}.
\end{proof}

\begin{figure*}[t]
    \centering
    \includegraphics[width=\linewidth]{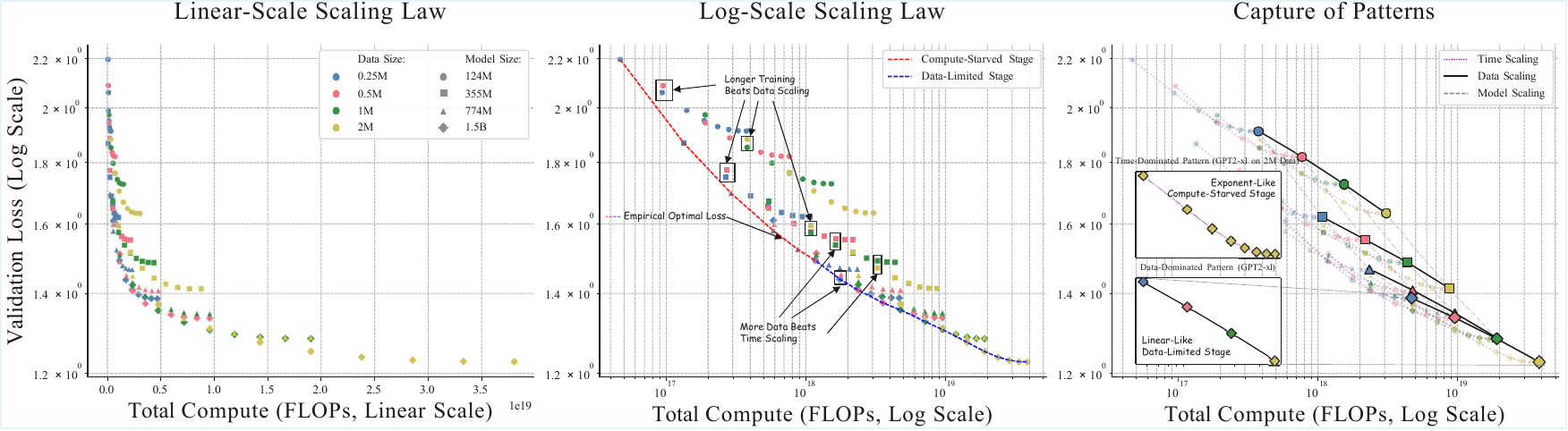}
    \caption{Scaling-law curves on validation cross-entropy across varying compute. The left panel uses a linear $x$-axis (absolute compute); the middle and right panels use log-log axes to visualize decay rates. Markers distinguish model size (124M to 1.5B) and dataset size. The plots compare Time-Dominated exponent-like trajectories against Data-Dominated linear-like trajectories.}
    \label{fig:scaling_law}
\end{figure*}

\subsection{Potential Failures of Scaling Law}\label{sub:potential_fails}

{\bf The Role of the Lambert $W$ Function.} The appearance of $W({\sf C}/\xi^{12})$ in Theorem~\ref{thm:scaling_law:informal} is mathematically significant. It indicates that as compute increases, the returns on that compute are not purely power-law distributed but are subject to a logarithmic correction. This implies that at extreme scales, the cost of reducing error becomes slightly more expensive than a simple power law would predict, likely due to the overhead of managing the approximation--generalization gap. Numerically, however, the multiplicative gap between the bare $\xi^{12/7}{\sf C}^{-1/7}$ rate and the $W$-corrected rate is essentially constant over the practical compute range ${\sf C}\in[10^{18},10^{25}]$ FLOPs (Figure~\ref{fig:lambert_w_compare} in Appendix~\ref{app_sub:practitioner}), so in measured curves the correction is absorbed into the offset of a Hoffmann-style fit $L=A+B/{\sf C}^\alpha$~\citep{kmh+20,hbm+22} rather than appearing as a separately observable bend.

{\bf Sensitivity of the Noise Level.} The term $\xi^{12/7}$ in Theorem~\ref{thm:scaling_law:informal} and the term $\xi^2$ of the Data-Law in Theorem~\ref{thm:single_law:informal} indicate that the quality or complexity of the sampled dataset $\mathbb{D}$ drastically shifts the scaling curve. Even with infinite compute, a high-complexity task (large $\xi$) will result in a significantly higher error floor.

The standard intuition is that larger data sizes help neural networks overcome the information bottleneck, yielding a tighter excess risk bound. However, there exists a possibility that including more data leads to the noise level $\xi$ growing with the dataset size, which we denote as $\xi({\sf N})$. When $\xi({\sf N}) \propto {\sf N}^{1/2}$ (or, in terms of compute, $\xi \propto {\sf C}^{1/12}$), the upper bound on the excess risk degrades to $O(1)$, effectively nullifying the data scaling law. This happens when low-quality data must be introduced for the sake of sheer data expansion, consistent with the data-constrained observations of \citet{sos+22, mas+23}.

\section{Case Studies}\label{sec:exp}

This section showcases the validation of our theory on two model families and two corpora: the GPT-2 series~\citep{bmr+20} on the TinyStories dataset~\citep{el23}, and the Pythia series~\citep{bsa+23} on WikiText-103~\citep{mxb+16}. All training is performed with the LLaMAFactory framework~\citep{zrj+24} and the AdamW optimizer~\citep{kb14}. We summarize only the key configurations and vital results below; full details and additional results are provided in Appendix~\ref{app:exp_full}. Section~\ref{sub:exp_setups} introduces the experimental setups, and Section~\ref{sub:exp_ret} demonstrates the corresponding results.

{\bf Scope.} Our theorems bound the squared excess risk, whereas language-model training minimizes the cross-entropy loss; the two objectives coincide at the gradient level only near interpolation (Remark~\ref{rem:ce_l2}, Appendix~\ref{app_sub:ce_l2_bridge}). The experiments are therefore designed to validate the qualitative predictions of the theory --- the existence of the Stage~I/II transition and the noise-driven breakdown of data scaling --- rather than the exact exponents.

\subsection{Experimental Setups}\label{sub:exp_setups}

{\bf GPT-2 $\times$ TinyStories.} We fine-tune four GPT-2 scales --- Small (124M), Medium (355M), Large (774M), and XL (1.5B) --- on training subsets of ${\sf N}\in\{0.25\text{M}, 0.5\text{M}, 1\text{M}, 2\text{M}\}$ stories drawn from the $2$M-story training pool, evaluating on a held-out set of $20$K stories.

{\bf Pythia $\times$ WikiText-103.} To probe a different model family and domain at larger data scales, we sweep six Pythia scales (70M--2.8B) over seven cumulative slices of the WikiText-103 training split containing $D\in\{0.25\text{M}, 0.5\text{M}, 1\text{M}, 2\text{M}, 4\text{M}, 8\text{M}, 16\text{M}\}$ unique tokens, packed into rows of $512$ tokens each.

{\bf Noisy Scaling.} In both case studies, the largest training set is partitioned into four equal subsets $S_1, \dots, S_4$, and the $k$-th noisy dataset is the cumulative union $\mathbb{D}_k = \bigcup_{i\le k} S_i$. Within subset $S_k$, tokens are corrupted at a per-subset rate $r_k$, either by (i)~replacement with a special mask token, or (ii)~replacement with a token drawn uniformly from the vocabulary. We use four rate schedules $(r_1, \dots, r_4)$, classified by how the cumulative corruption rate $\bar{r}_k := \frac{1}{k}\sum_{i\le k} r_i$ --- and with it the per-position noise variance $\xi^2(\mathbb{D}_k) = \Theta(\bar r_k)$ (Appendix~\ref{app_sub:ce_l2_bridge}) --- grows relative to the breakdown threshold of Section~\ref{sub:potential_fails}:
\begin{itemize}
    \item $\mathcal{M}_1 = \{0.05, 0.10, 0.15, 0.20\}$ and $\mathcal{M}_2 = \{0.1, 0.188, 0.231, 0.281\}$ (GPT-2; the latter chosen so that $\bar r_k = 0.1\sqrt{k}$): the noise floor $\xi^2(\mathbb{D}_k)/|\mathbb{D}_k|$ still decreases in $k$, so data scaling is predicted to slow down but not to fail;
    \item $\mathcal{M}_2^{\rm ext} = \{0.05, 0.15, 0.25, 0.35\}$ (Pythia; $r_k = (2k-1)\cdot 0.05$, hence $\bar r_k = 0.05\,k$ exactly): the variance grows linearly in the data size --- precisely the breakdown rate $\xi({\sf N})\propto{\sf N}^{1/2}$ --- so the noise floor is predicted to stall at a constant;
    \item $\mathcal{M}_3 = \{0.05, 0.20, 0.50, 0.95\}$ (Pythia; $\bar r_k \propto k^{1.54}$): super-linear variance growth, strictly beyond the breakdown rate, under which the noise floor is predicted to grow.
\end{itemize}

\subsection{Results}\label{sub:exp_ret}

{\bf Single and joint scaling laws.} Figure~\ref{fig:single_law} demonstrates the single-variable laws on GPT-2, consistent with the classical findings of \citet{kmh+20}. On the Pythia sweep, a Chinchilla-style joint fit $L(N,D)=E+A N^{-\alpha}+B D^{-\beta}$ (with $N$ the model parameter count, distinct from the theory's data-side ${\sf N}$) over the $6\times 7$ grid recovers $\alpha=0.43$ and $\beta=0.34$ with $R^2_{\log}=0.85$ on the $36$ of $42$ cells passing the monotonicity criterion of Appendix~\ref{app_sub:wikitext_results}, in the same range as the published estimates $\alpha\approx 0.34$, $\beta\approx 0.28$ of \citet{hbm+22} (Table~\ref{tab:hoffmann_fit}).

{\bf Two-stage compute transition.} Figure~\ref{fig:scaling_law} exhibits the two regimes on GPT-2: in the upper-left region of the middle panel, models trained longer on less data outperform those trained shorter on more data at equal FLOPs ({\it Longer Training Beats Data Scaling}), whereas in the lower-right region the opposite holds ({\it More Data Beats Time Scaling}). Figure~\ref{fig:wikitext_main}\,(b) examines the same transition on a single Pythia trajectory (70M on the 16M-token slice; $156$ evaluations): splitting the post-warmup trajectory by log-compute, an exponential fit to the early quartile and a power-law fit to the late half each overlay their own portion of the trajectory and cross inside the middle band. The fitted late-half exponent is $-0.087$, shallower than the theoretical $-1/7\approx-0.143$, in line with the qualitative-only scope stated above (further diagnostics in Appendix~\ref{app_sub:phase_diagnostics}).

{\bf Noise-driven breakdown of data scaling.} Table~\ref{tab:noisy_training} reports the GPT-2 study. Under the sub-breakdown schedule $\mathcal{M}_2$, the gap to clean training stalls at $+0.29$--$+0.30$ as the data grows from $0.5$M to $2$M stories, while the clean baseline improves by $0.15$; the absolute noisy loss still falls ($1.88\to1.73$). Table~\ref{tab:noisy_extended} and Figure~\ref{fig:wikitext_main}\,(a) report the Pythia study on cumulative slices of $D\in\{4,8,12,16\}$M tokens. At the exact breakdown rate ($\mathcal{M}_2^{\rm ext}$), the gap under mask-token corruption shrinks from $+1.27$ to $+1.07$ and flattens, consistent with a saturating $O(1)$ floor; beyond the breakdown rate ($\mathcal{M}_3$), the gap grows under both corruption types ($+1.17\to+1.37$ mask, $+1.18\to+1.34$ random). Random-token corruption pushes even $\mathcal{M}_2^{\rm ext}$ into the growing regime ($+1.11\to+1.33$): uniform-vocabulary replacement carries an irreducible cross-entropy per corrupted position, hence an effectively heavier noise level than its nominal rate (Appendix~\ref{app_sub:ce_l2_bridge}).

\begin{figure*}[t]
\centering
\includegraphics[width=\linewidth]{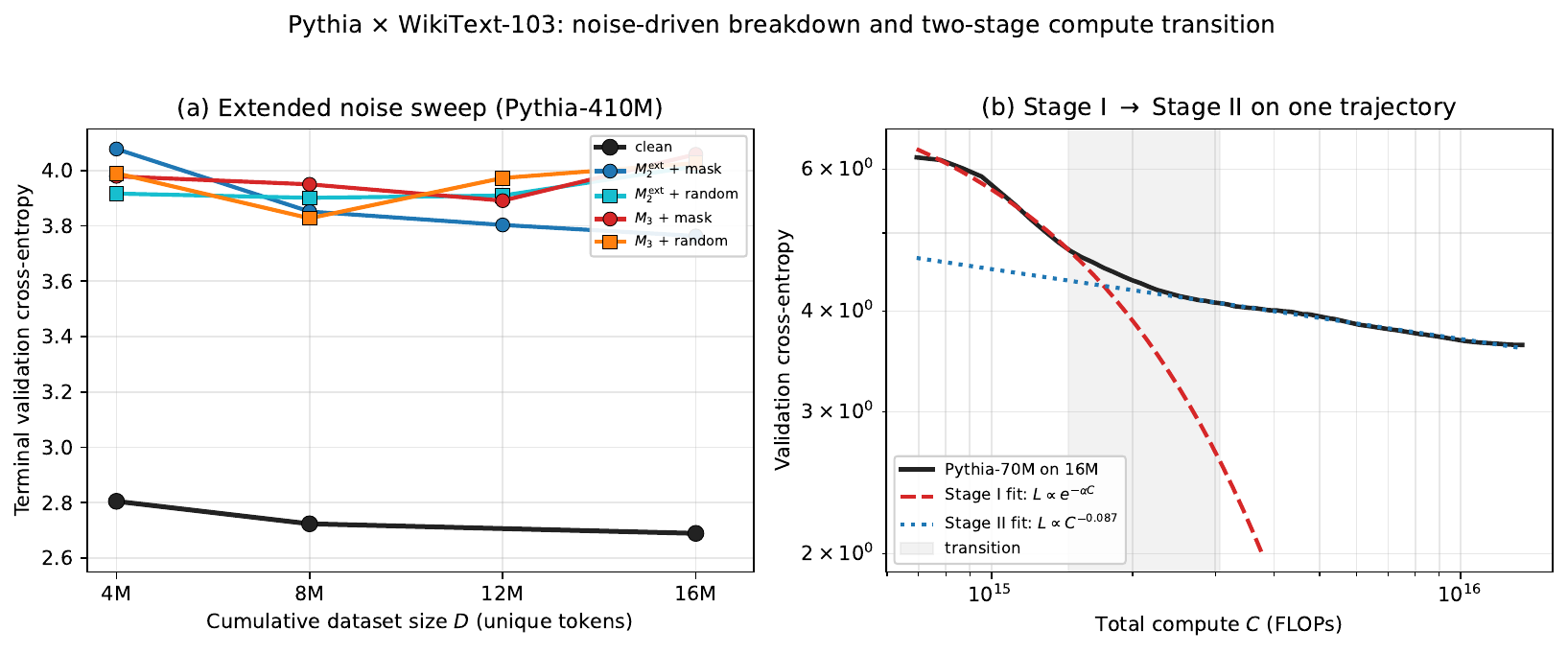}
\caption{Pythia~$\times$~WikiText-103. {\bf (a)}~Extended noise sweep on Pythia-410M over cumulative slices of $D\in\{4, 8, 12, 16\}$M unique tokens: at the breakdown rate ($\mathcal{M}_2^{\rm ext}$) the gap to clean saturates under mask-token corruption, while the super-linear $\mathcal{M}_3$ diverges. {\bf (b)}~Stage~I~$\to$~Stage~II transition on a single Pythia-70M trajectory over the 16M-token slice: the early-quartile exponential fit and the late-half power-law fit each overlay their own portion of the trajectory and cross inside the shaded band.}
\label{fig:wikitext_main}
\end{figure*}

\begin{figure*}[t]
\centering
\begin{minipage}[t]{0.49\linewidth}
\centering
\captionof{table}{GPT-2-medium $\times$ TinyStories: best validation cross-entropy across data sizes (in stories) under the schedules of Section~\ref{sub:exp_setups}. Red: gap to clean.}
\label{tab:noisy_training}
\myvspace{-1mm}
% In the arxiv (single-column) build the right-hand caption wraps to more
% lines than the left-hand caption, leaving Table~1 floating high. A small
% top spacer pushes Table~1 down so the two tabulars sit on the same
% baseline; in the NeurIPS double-column build the captions are comparable
% in length and no spacer is needed.
\ifdefined\isarxiv\vspace{1.8\baselineskip}\fi
{\scriptsize\setlength{\tabcolsep}{3pt}
\begin{tabular}{@{}lcccc@{}}
\toprule
Mask Strategy & 0.5M & 1M & 1.5M & 2.0M \\
\midrule
clean data                    & 1.59 & 1.51 & 1.46 & 1.44 \\
$\mathcal{M}_1$ + mask token  & 1.73 {\color{red}\tiny +.14} & 1.66 {\color{red}\tiny +.15} & 1.63 {\color{red}\tiny +.17} & 1.61 {\color{red}\tiny +.17} \\
$\mathcal{M}_2$ + mask token  & 1.88 {\color{red}\tiny +.29} & 1.81 {\color{red}\tiny +.30} & 1.75 {\color{red}\tiny +.29} & 1.73 {\color{red}\tiny +.29} \\
$\mathcal{M}_2$ + all token   & 1.80 {\color{red}\tiny +.21} & 1.69 {\color{red}\tiny +.18} & 1.71 {\color{red}\tiny +.25} & 1.67 {\color{red}\tiny +.23} \\
\bottomrule
\end{tabular}}
\end{minipage}\hfill
\begin{minipage}[t]{0.49\linewidth}
\centering
\captionof{table}{Pythia-410M $\times$ WikiText-103: terminal validation cross-entropy under the extended schedules. Columns: cumulative slice size $D$ in unique tokens. Red: gap to clean. $^{\dagger}$The clean sweep contains no $12$M cell; its baseline is interpolated between the $8$M and $16$M runs, and gaps in that column use the interpolated value.}
\label{tab:noisy_extended}
\myvspace{-1mm}
{\tiny\setlength{\tabcolsep}{2pt}
\begin{tabular}{@{}lcccc@{}}
\toprule
Mask Strategy & 4M & 8M & 12M$^{\dagger}$ & 16M \\
\midrule
clean data          & 2.81 {\color{red}\tiny +.00} & 2.72 {\color{red}\tiny +.00} & 2.71$^{\dagger}$ & 2.69 {\color{red}\tiny +.00} \\
$\mathcal{M}_2^{\rm ext}$ + mask token    & 4.08 {\color{red}\tiny +1.27} & 3.85 {\color{red}\tiny +1.13} & 3.81 {\color{red}\tiny +1.10} & 3.76 {\color{red}\tiny +1.07} \\
$\mathcal{M}_2^{\rm ext}$ + random token  & 3.92 {\color{red}\tiny +1.11} & 3.90 {\color{red}\tiny +1.18} & 3.91 {\color{red}\tiny +1.20} & 4.02 {\color{red}\tiny +1.33} \\
$\mathcal{M}_3$ + mask token              & 3.98 {\color{red}\tiny +1.17} & 3.95 {\color{red}\tiny +1.23} & 3.89 {\color{red}\tiny +1.18} & 4.06 {\color{red}\tiny +1.37} \\
$\mathcal{M}_3$ + random token            & 3.98 {\color{red}\tiny +1.18} & 3.83 {\color{red}\tiny +1.10} & 3.97 {\color{red}\tiny +1.27} & 4.03 {\color{red}\tiny +1.34} \\
\bottomrule
\end{tabular}}
\end{minipage}
\end{figure*}

\section{Conclusion}\label{sec:conclusion}

This work presents a comprehensive theoretical framework for rigorously analyzing the scaling law phenomenon in LLMs, specifically addressing the empirically observed power-law relationship between model performance and computational resources from a learning dynamics perspective. By formalizing the dynamics of training sequence-to-sequence multi-layer transformer architectures, we establish a foundational guarantee: under allocation of compute, the generalization excess error of these models converges asymptotically as computational budgets increase, with rate {\it i)} $\exp(-\Theta(\xi^2{\sf C}/{\sf N}^7))$ in the {\it compute-starved stage} and {\it ii)} $\Theta(\xi^{12/7}{\sf C}^{-1/7})$ in the {\it data-limited stage}, where ${\sf C}$ is the computational cost --- rates certified as tight by matching statistical and optimization-side lower bounds. Furthermore, we quantify the individual influence of three scaling variables (model size, training duration, and dataset scale) on performance. Finally, we identify critical constraints (potential drags) on the reliability of transformer neural scaling, demonstrating that logarithmic corrections from the Lambert $W$ function, the growth of dataset noise levels relative to sample size, and the saturation of model capacity relative to data complexity can lead to diminishing returns or the total breakdown of predictable scaling laws.

\ifdefined\isarxiv
\section*{Acknowledgement}

I would like to sincerely thank Jiangxuan Long, Yiping Lu, Jing Xiong, and Juepeng Zheng for the helpful discussions. %I regard this paper as a summary of my entire undergraduate academic research, and of course, I would like to cheer and thank myself.

\else
% NeurIPS 2026: acknowledgments are hidden in the anonymized submission and revealed
% only in the [final] camera-ready version when wrapped in \begin{ack}...\end{ack}.
% \begin{ack}
% Acknowledgments go here. Funding and competing-interest disclosure required for
% the camera-ready version.
% \end{ack}
\fi

\ifdefined\isarxiv
%\section*{Acknowledgments}
% Arxiv build: alpha bbl labels (e.g. [KB25]) for compact in-text citations,
% paired with natbib loaded in default (authoryear) mode at the top of the
% file so \citet / \citep both render correctly against the alpha labels.
\bibliographystyle{alpha}
\bibliography{ref}
\else
% NeurIPS 2026 uses any consistent bibliography style; we keep plainnat (natbib).
\newpage
\bibliographystyle{plainnat}
\bibliography{ref}

\fi

\newpage
\appendix
% NeurIPS 2026 is single-column, so \onecolumn is a no-op there but is needed
% for the arxiv (article) build to switch from any prior \twocolumn context.
\ifdefined\isarxiv\onecolumn\fi

\begin{center}
    {\LARGE \bf Appendix}
\end{center}
{\hypersetup{linkcolor=black}
\tableofcontents
}
\newpage

\ifdefined\isarxiv
\else

\fi

\section{Extended Discussion of Main-Text Results}\label{app:extended}

This appendix collects the longer discussions and informal-version corollaries of the main-text theorems that were briefly mentioned in Sections~\ref{sec:convergence}--\ref{sec:conclusion}.

\subsection{Causality-Induced Matrix-to-Vector Reduction}\label{app_sub:matrix_to_vector}

A central technical contribution underlying every result in this paper is the {\it reduction of a matrix-valued NTK analysis to a token-indexed vector-valued NTK analysis}, made possible by the decoder-only causal mask. We summarize the reduction in Figure~\ref{fig:matrix_to_vector_reduction} and elaborate below.

\begin{figure}[t]
\centering
\begin{tikzpicture}[
  scale=0.92,
  every node/.style={font=\footnotesize},
  inputcell/.style={draw, very thick, fill=blue!12, minimum width=8mm, minimum height=4mm, inner sep=0pt},
  outputcell/.style={draw, very thick, fill=red!10, minimum width=8mm, minimum height=4mm, inner sep=0pt},
  prefcell/.style={draw, thick, fill=blue!12, minimum width=4mm, minimum height=3.5mm, inner sep=0pt},
  outvec/.style={draw, thick, fill=red!10, minimum width=5mm, minimum height=3.5mm, inner sep=0pt},
  arr/.style={-{Latex[length=2mm,width=1.5mm]}, thick}
]

% --- Panel (a): Standard matrix view ---
\node[font=\small\bfseries] at (1.0, 4.5) {(a) Matrix view};

\foreach \i/\lab in {1/{$\ell{=}1$}, 2/{$\ell{=}2$}, 3/{}, 4/{$\ell{=}L$}} {
  \pgfmathsetmacro{\y}{4 - \i*0.55}
  \node[inputcell] at (0.5, \y) {};
  \ifnum\i=3
    \node[font=\tiny] at (-0.45, \y) {$\vdots$};
  \else
    \node[font=\tiny, anchor=east] at (-0.05, \y) {\lab};
  \fi
}
\node[font=\small] at (0.5, 0.95) {$X_i \in \mathbb{R}^{L\times d}$};

\draw[arr] (1.05, 2.62) -- (2.0, 2.62);
\node[font=\small] at (1.55, 2.95) {$F$};

\foreach \i in {1,...,4} {
  \pgfmathsetmacro{\y}{4 - \i*0.55}
  \node[outputcell] at (2.55, \y) {};
}
\node[font=\small] at (2.55, 0.95) {$F(X_i) \in \mathbb{R}^{L\times d}$};

\node[font=\footnotesize, align=center] at (1.5, 0.0) {$n$ matrix samples\\$\Rightarrow$ matrix gradient flow};

% --- Bridge ---
\draw[arr, very thick, blue!65] (3.85, 2.62) -- (5.05, 2.62);
\node[font=\small\bfseries, blue!65, align=center] at (4.45, 3.05) {causal mask\\(decoder-only)};

% --- Panel (b): Vector view (staircase) ---
\node[font=\small\bfseries] at (8.6, 4.5) {(b) Vector view (decoder-only)};

\foreach \L/\name/\xshift in {1/{$F_1$}/0, 2/{$F_2$}/1.25, 3/{$F_3$}/2.5, 4/{$F_L$}/3.75} {
  \pgfmathsetmacro{\xc}{6.0 + \xshift}
  \foreach \k in {1,...,\L} {
    \pgfmathsetmacro{\yc}{0.95 + (\k-1)*0.4}
    \node[prefcell] at (\xc, \yc) {};
  }
  \pgfmathsetmacro{\toparr}{0.95 + (\L-1)*0.4 + 0.2}
  \pgfmathsetmacro{\toplab}{\toparr + 0.5}
  \draw[arr, thin] (\xc, \toparr) -- (\xc, \toplab);
  \node[outvec] at (\xc, \toplab+0.25) {};
  \node[font=\tiny] at (\xc, \toplab+0.25) {\name};
  \node[font=\tiny] at (\xc, 0.55) {$X_{i,\le\L}$};
}
\node[font=\tiny] at (7.875, 0.15) {(growing prefix length)};

\node[font=\footnotesize, align=center] at (7.875, -0.45) {$L$ prefix-to-vector pairs per $X_i$\;$\Rightarrow nL$ vector samples\\$\Rightarrow H_{(\nu)}\in\mathbb{R}^{nL\times nL}$ via standard 2-layer NTK};

\end{tikzpicture}
\caption{Causality-induced matrix-to-vector reduction. (a)~A standard NTK analysis would treat each training sample $X_i\in\mathbb{R}^{L\times d}$ as a matrix and the model $F$ as a matrix-valued map, yielding a matrix tangent kernel and a matrix-valued gradient flow. (b)~Under the decoder-only causal mask, the $\ell$-th output token $F_\ell(X_i)\in\mathbb{R}^d$ depends only on the prefix $X_{i,\le\ell}\in\mathbb{R}^{\ell\times d}$, so $X_i$ unfolds into $L$ prefix-to-vector regressions $\{(X_{i,\le\ell}, F_\ell(X_i))\}_{\ell=1}^L$. The dataset of $n$ matrix samples becomes $nL$ vector samples; the layer-wise kernel $H_{(\nu)}\in\mathbb{R}^{nL\times nL}$ is a Gram matrix on those vector samples (Lemma~\ref{lem:learning_dynamics:informal}) and standard 2-layer ReLU NTK tools apply directly.}
\label{fig:matrix_to_vector_reduction}
\end{figure}

\paragraph{Why this matters beyond this paper.} The reduction relies only on the {\it causal} structure of the architecture, not on any specifics of softmax-attention or the MLP block. The same recipe transports any sequence-to-sequence NTK analysis with prefix-dependent tangent feature into the vector NTK framework: masked-attention variants, autoregressive convolutional architectures, RNNs / state-space models with prefix dependence, hybrid (mixture-of-experts) decoders, etc. We highlight three concrete payoffs:
\begin{itemize}
    \item[(i)] {\it Sample-index unification.} Token positions $\ell\in[L]$ enter the proof on the same footing as data indices $i\in[n]$. Concentration / covering / minimum-eigenvalue arguments that apply to $n$-sample NTK theory \citep{dzps19, azls19, dllw19, os20} apply verbatim with the larger sample count $nL$, with no architectural rederivation.
    \item[(ii)] {\it Layer decoupling.} The kernel $H_{(\nu)}(t)$ is layer-local: each token's tangent feature factors through the {\it previous} layer's attention output $o_{(\nu),p}(t)$, so layer interactions reduce to a chain rule on Gram matrices instead of a 4-tensor contraction. This is what makes the per-layer convergence rate $\Theta(\omega\lambda N)$ in Theorem~\ref{thm:convergence:informal} explicit.
    \item[(iii)] {\it Trainable-FFN compatibility.} Adding the FFN read-out feature $\alpha_{(\nu),p}(t)$ amounts to {\it augmenting} the existing nL-indexed kernel rather than re-deriving a matrix-tensor object (Lemma~\ref{lem:ffn_drift:informal}); the same is true for any future tangent-feature additions (e.g. positional embeddings, gating, cross-attention).
\end{itemize}

\paragraph{Formal statement of the reduction.} Concretely, the matrix-form gradient flow $\frac{\d}{\d t}\mathsf{F}(t)$ on $\mathsf{F}\in\mathbb{R}^{nL\times d}$ becomes the vector form
\[
\frac{\d}{\d t}\mu_{(\nu)}(t) = -\,(H_{(\nu)}(t)\otimes I_d)\cdot \nabla_{\mu_{(\nu)}}\mathcal{L}(t,\mathbb{D}),\qquad \mu_{(\nu)}(t)\in\mathbb{R}^{nL\times d},
\]
where the Kronecker structure $H_{(\nu)}\otimes I_d$ on the $nL$ vector samples replaces what would otherwise be a 4-tensor on $\mathbb{R}^{L\times d}\times\mathbb{R}^{L\times d}$. This Kronecker form is exactly the one for which $\lambda_{\min}(H\otimes I_d) = \lambda_{\min}(H)$ (Fact~\ref{fac:lambda_min_krnocker}) and minimum-eigenvalue control reduces to the $nL\times nL$ matrix~$H_{(\nu)}$ alone (Appendix~\ref{app_sub:pd_sufficient}).

\subsection{Trainable-FFN Extension (Lemma~\ref{lem:ffn_drift:informal})}\label{app_sub:ffn_drift_main}

\begin{lemma}[FFN drift control under trainable $A_{(\nu)}$, informal version of Lemma~\ref{lem:ffn_drift}]\label{lem:ffn_drift:informal}
    Let $A_{(\nu)}(t)$ evolve by gradient flow as in Equation~\eqref{eq:ODE_update}. Under Definition~\ref{def:gd} with the strengthened width $m = \Omega\big(\frac{n^4 L^6 d^2 \exp(Cd)}{\omega^6 \lambda^8 \delta^4 N^4}\big)$, with probability at least $1-\delta$:
    \begin{align*}
        \max_{\nu\in[N]}\|A_{(\nu)}(t) - A_{(\nu)}(0)\|_F \leq O\Big(\frac{1}{\sqrt{m}\,\omega\lambda N}\Big),
    \end{align*}
    and the layer-aggregated tangent kernel still satisfies $\lambda_{\min}(H_{(\nu)}(t)\otimes I_d) \ge \omega\lambda/4$ uniformly in $t\ge 0$. In particular, every conclusion of Theorem~\ref{thm:convergence:informal} and Theorem~\ref{thm:generalization:informal} continues to hold, with the constants on the RHS replaced by absolute constants that absorb the new contribution of the FFN tangent feature.
\end{lemma}

\noindent{\bf Scope of the trainable-$A$ extension.} Lemma~\ref{lem:ffn_drift:informal} keeps the FFN read-out within the lazy ball: the drift bound $\|A_{(\nu)}(t)-A_{(\nu)}(0)\|_F = O(1/(\sqrt m\,\omega\lambda N)) = o(1)$ is the same order as the $W$- and $U$-drift in Lemma~\ref{lem:perturbations:informal}. Concretely, our contribution over strict-NTK \citep{dzps19} is to {\it add the FFN tangent feature} $\alpha_{(\nu),p}(t)$ (Definition~\ref{def:alpha_feature}) into the kernel and prove that the augmented kernel $H^{\rm aug}_{(\nu)}$ remains $\Theta(\omega\lambda)$-PD throughout training, rather than to leave the lazy regime. Truly rich / feature-learning FFN training (drift $\Omega(1)$) is left to future work; see Appendix~\ref{app_sub:lazy_to_rich}.

\begin{proof}[Proof sketch of Lemma~\ref{lem:ffn_drift:informal}]
    The post-activation FFN feature $\Lambda_{(\nu),i,\ell}^{\rm ffn}(t) := (\omega/\sqrt m)\,{\rm ReLU}(\mathrm{Softmax}(\cdots)$ $ X W_{(\nu)}(t)) \in \R^{1\times m}$ has bounded norm by Parts~3 and~5 of Lemma~\ref{lem:helpful_bounds}. The gradient of the loss with respect to $A_{(\nu)}$ is therefore controlled by the same $\varepsilon\sqrt{{\cal L}(t,\mathbb{D})}$ envelope as that of $W_{(\nu)}$. Repeating the bootstrapping step of Lemma~\ref{lem:perturbations:informal} with the additional kernel feature $\alpha_{(\nu),p}(t):= \Lambda^{\rm ffn}_{(\nu),p}(t)$ adds an $O(\omega)$ contribution to the kernel perturbation; the strengthened width condition is precisely what is needed for this contribution to remain $\le \omega\lambda/4$ at all times. The full proof appears in Appendix~\ref{app:proof_lperturbations}.
\end{proof}

\subsection{Matching Lower Bounds (Informal Versions)}\label{app_sub:lower_bounds_main}

The upper bounds in Theorem~\ref{thm:generalization:informal} and Theorem~\ref{thm:scaling_law:informal} characterize the achievable excess risk in the lazy/NTK regime. To certify that these rates are not artifacts of a loose argument, we state matching {\it lower} bounds in two complementary directions: a statistical (information-theoretic) lower bound that constrains every estimator from a dataset of size ${\sf N}$, and an optimization-side (first-order oracle) lower bound that constrains every gradient-based algorithm with at most ${\sf C}$ FLOPs of compute.

\begin{theorem}[Statistical lower bound, informal version of Theorem~\ref{thm:lower_stat}]\label{thm:lower_stat:informal}
    Fix a non-trivial subspace ${\cal H}^*\subseteq{\cal H}_K$ of the RKHS associated with the layer-wise NTKs. There exists a family of target functions $\{F_\theta^*\}_{\theta\in\Theta_{\sf N}}\subset {\cal H}^*$ and a constant $c>0$ such that for any (possibly randomized) estimator $\hat F$ measurable with respect to $\mathbb{D}\sim {\cal D}^{\otimes {\sf N}}$:
    \begin{align*}
        \sup_{F^*\in\{F_\theta^*\}_{\theta\in\Theta_{\sf N}}}\E_{\mathbb{D}}\big[\Delta{\cal R}(\hat F)\big] \ge c \cdot \frac{\xi^2}{{\sf N}}.
    \end{align*}
    In particular, the upper bound $O(\xi^2/{\sf N})$ in Theorem~\ref{thm:generalization:informal} is tight up to a constant.
\end{theorem}

\begin{proof}[Proof sketch]
    The argument is a Le Cam two-point reduction inside ${\cal H}^*$. Choose two targets $F_0^*, F_1^* \in {\cal H}^*$ with $\|F_0^*-F_1^*\|_{L^F({\cal X})}^2 = c_1\xi^2/{\sf N}$ and Gaussian-noise KL divergence $\frac{{\sf N}}{2\xi^2}\|F_0^*-F_1^*\|_{L^F}^2 \le \frac{1}{4}$. Le Cam's two-point lemma then gives $\inf_{\hat F}\sup_j\E\Delta{\cal R}(\hat F)\ge \frac{1}{8}\|F_0^*-F_1^*\|_{L^F}^2$. The full proof, which merely requires that ${\cal H}^*$ contain a non-zero direction, appears in Appendix~\ref{app_sub:lower_bounds}.
\end{proof}

\begin{theorem}[Optimization-side lower bound, informal version of Theorem~\ref{thm:lower_opt}]\label{thm:lower_opt:informal}
    For any first-order optimization method ${\cal A}$ that, at each time, queries only the gradient $\nabla_\theta {\cal L}(t,\mathbb{B}(t))$ within the parameter ball of Definition~\ref{def:gd}, there exists an instance of the data distribution ${\cal D}$ and an initialization $\theta(0)$ such that the training loss after time $\sf T$ is lower-bounded by:
    \begin{align*}
        {\cal L}({\sf T}, \mathbb{D}) \ge c\cdot\exp(-C_{\rm opt}\cdot \omega\lambda_{\max} \cdot \varepsilon^2 {\sf T}) \cdot {\cal L}(0, \mathbb{D}),
    \end{align*}
    where $\lambda_{\max}:= \max_\nu \lambda_{\max}(H_{(\nu)}'(0))$ and $c, C_{\rm opt}>0$ are absolute constants. In particular, the exponent $-\Theta(\sqrt{\kappa}\,\xi^2{\sf C}/{\sf N}^7)$ (with $\kappa:=\lambda_{\max}/\lambda$) in the {\it Compute-Starved Stage} is tight up to constants and a $\sqrt{\kappa}$ factor in the exponent --- the unavoidable Nemirovski--Yudin condition-number gap; see Appendix~\ref{app_sub:matching_precise}.
\end{theorem}

\begin{proof}[Proof sketch]
    The training objective restricted to the parameter ball is a strongly convex quadratic in the linearization regime (kernel features are nearly fixed by Lemma~\ref{lem:perturbations:informal}). Standard first-order oracle complexity (Nemirovski--Yudin lower bounds for quadratic minimization, see e.g.~\citealp{nes03}) implies that no gradient-based method can decay the loss faster than $\exp(-\Theta(\sqrt{\lambda\lambda_{\max}})\,t)$ in continuous time, matching the discrete $((\sqrt\kappa-1)/(\sqrt\kappa+1))^{2T}$ bound after $T\to t/\eta$ with $\eta\to 0$. Translating to compute via ${\sf C}=O({\sf MTN})$ with ${\sf M}=\Theta({\sf N}^3)$ and incorporating the $1/\mathsf{N}$ factor in the rate (Remark~\ref{rem:n_in_rate}) recovers the $\exp(-\Theta(\sqrt{\kappa}\,\xi^2{\sf C}/{\sf N}^7))$ rate. The full argument is in Appendix~\ref{app_sub:lower_bounds}.
\end{proof}

\subsection{Cross-entropy vs.\ Squared Excess Risk}\label{app_sub:ce_l2_bridge}

\begin{remark}[Cross-entropy vs.\ squared excess risk]\label{rem:ce_l2}
Cross-entropy after softmax shares its gradient form with squared loss in the logit space at the {\it Bregman / natural-gradient} level: with logits $z\in\R^{|V|}$, $y'={\rm softmax}(z)$, and one-hot $y\in\R^{|V|}$, $\partial_z {\cal L}_{\rm CE} = y'-y$ which is the gradient of the Bregman divergence $B_{\Phi^*}(y, y')$ where $\Phi^*$ is the log-partition; this matches the squared-loss gradient $z-y$ {\it only} in the well-specified, near-zero-excess regime ($y'\approx y$, e.g.~CE near $H(y)$). For typical LM training (CE $\sim 1.5$--$2$, far from $H(y)$), the two losses are not numerically identical even at the rate-exponent level, and our experiments are therefore meant to validate the {\it qualitative} predictions of Section~\ref{sec:scaling_law} (existence of the Stage~I/II transition; ${\sf N}^{1/2}$-noise breakdown of data scaling) rather than the exact exponents.
\end{remark}

\paragraph{Transferring the ${\sf N}^{1/2}$-noise breakdown to CE.} The L$^2$ statement (Section~\ref{sub:potential_fails}) is: for noise variance $\xi^2(\mathsf{N})\propto \mathsf{N}$, the bound $O(\xi^2/\mathsf{N})\to O(1)$ saturates and data scaling fails. The CE analogue is the following token-level argument. Suppose a fraction $p_k$ of tokens in $D_k$ is corrupted by either (i) a fixed mask token (mask strategy) or (ii) a uniform random vocabulary token (random strategy). For strategy (ii), the corrupted tokens are unpredictable, and the irreducible CE on them is $\log V$ ($V$ = vocab size); the total CE then satisfies
\[
{\sf CE}(D_k) \ \ge\ p_k \cdot \log V + (1-p_k)\cdot {\sf CE}_{\rm clean}(D_k),
\]
so the {\it gap to clean} satisfies $\Delta {\sf CE}(D_k) \gtrsim p_k \cdot (\log V - {\sf CE}_{\rm clean})$, which under our schedule $p_k \propto \sqrt{\mathsf{N}(D_k)}$ grows as $\sqrt{\mathsf{N}}$. Once $\sqrt{\mathsf{N}}$ exceeds the rate at which ${\sf CE}_{\rm clean}$ improves (empirically a slow power of $\mathsf{N}$), the gap stalls and eventually grows --- exactly mirroring the L$^2$ saturation. For the mask strategy, $p_k\cdot \log V$ is replaced by $p_k\cdot {\sf KL}({\rm mask}\,\|\,{\sf next}|{\rm context})$ which is again $\Theta(p_k)$ at fixed mask-token statistics; the same scaling follows. Table~\ref{tab:noisy_training} reports the {\it stalling} portion of this prediction; the {\it growing} portion lies beyond the $0.5\to 2$M range and is observed in the extended Pythia sweep (Table~\ref{tab:noisy_extended}, Appendix~\ref{app_sub:wikitext_results}).

\paragraph{From masking probability $p_k$ to the noise variance $\xi^2({\sf N})$ used in Theorem~\ref{thm:scaling_law:informal}.}
The link between Table~\ref{tab:noisy_training} and the assumption $\xi({\sf N})\propto {\sf N}^{1/2}$ is more principled than the schedule label $\mathcal{M}_2$ alone suggests. Treat the embedding-space target as $Y_{\ell}^{\rm clean}$ and the corrupted target as $Y_\ell^{\rm clean}+ \Xi_\ell$ where $\Xi_\ell\in\R^d$ is non-zero only on the corrupted token positions. Conditional on the (fixed-pattern) corruption mask, $\Xi_\ell$ is the difference between the embedding of the {\it true} next token and the embedding of either (a) the dedicated mask token (strategy~$\mathcal{M}_*$ + mask-token) or (b) a uniform random token (strategy~$\mathcal{M}_*$ + all-token). Both differences are bounded by $2\max_v \|e_v\|_2 = O(1)$, and unconditional on the mask, $\E[\Xi_\ell]=0$ for the all-token variant (uniform vocabulary draws) and $\E[\Xi_\ell]=O(1)\cdot p_k$ for the mask-token variant (a fixed token offset weighted by $p_k$).
With per-token corruption probability $p_k\in[0,1]$, the per-position variance satisfies
\[
\var[\Xi_\ell] \;=\; p_k\cdot O(1) \;-\; (\E[\Xi_\ell])^2 \;=\; \Theta(p_k)
\]
to leading order in $p_k$ for both strategies, so $\xi^2 = \Theta(p_k)$. Under schedule $\mathcal{M}_2 = \{0.10, 0.188, 0.231, 0.281\}$ on $\mathsf{N}\in\{0.5\text{M}, 1\text{M}, 1.5\text{M}, 2\text{M}\}$, the cumulative corruption rate scales as $p_k = 0.10\,\sqrt{\mathsf{N}/0.5\text{M}}$ (factor $2=\sqrt 4$ between endpoints), so $\xi^2 = \Theta(p_k)\propto \sqrt{\mathsf{N}}$, i.e.\ $\xi\propto \mathsf{N}^{1/4}$ {\it as a per-position embedding-space noise level} --- not the $\xi\propto {\sf N}^{1/2}$ of the headline theorem. The headline scaling $\xi\propto {\sf N}^{1/2}$ is what is required for the {\it variance} to grow linearly in ${\sf N}$ in the L$^2$ floor $\xi^2/{\sf N}$, which corresponds to $p_k\propto {\sf N}$ --- precisely what the extended schedule $\mathcal{M}_2^{\rm ext}$ of the Pythia case study realizes ($r_k=(2k-1)\cdot 0.05$, cumulative rate $0.05\,k$; Table~\ref{tab:noisy_extended}), subject to the cap $p_k\le 1$. The {\it qualitative direction} ($\xi^2$ growing with ${\sf N}$, gap-to-clean stalling) is what the table demonstrates; the exact exponent is therefore an upper bound on the achievable degradation rather than a literal numerical match. We highlight this distinction to keep the failure-mode claim within what the experiment actually establishes.

\subsection{From Continuous Gradient Flow to Discrete SGD}\label{app_sub:discrete_sgd}

Our convergence analysis is stated for continuous-time gradient flow, which is standard in the NTK literature \citep{jgh18, dzps19} and exposes the role of the kernel matrices $H_{(\nu)}(t)$ most cleanly. The continuous-time picture extends to discrete SGD via the standard route taken in \citet{dzps19} (between their Section 3 and Section 4) and the prefix-NTK proof in \citet{lssy25} (Appendix H): given a small step size, a Taylor expansion of the per-step loss change isolates a {\it dominant term} of the form $\eta \cdot \langle\nabla {\cal L}, H(t)\nabla {\cal L}\rangle$ (matching the gradient-flow bound) plus {\it residual terms} of order $\eta^2$ that are controlled by over-parameterization (lazy learning) and shrink as the network width $m$ grows or the loss decreases. Provided the residual is bounded relative to the dominant term, the discrete bound matches the continuous one up to step-size factors. We focus on continuous time to keep the layered self-attention story clean and treat the discretization step as an orthogonal extension; full discrete analyses with adaptive optimizers (e.g.~AdamW) require additional bookkeeping but do not appear to change our two-stage qualitative conclusions.

\section{Scope, Sharpening, and Interpretation}\label{app:scope}

This appendix complements Appendix~\ref{app:extended} by sharpening the scope of the headline claims, isolating which structural pieces of the two-stage law are transformer-specific versus generic-NTK, and recording an empirical-interpretation guide for the theorems.

\subsection{From Lazy to Rich: a Bridging Remark}\label{app_sub:lazy_to_rich}

The matching upper and lower bounds (Theorems~\ref{thm:generalization:informal}, \ref{thm:lower_stat:informal} and \ref{thm:lower_opt:informal}) are tight for {\it nearly-lazy} training: parameters drift only by $R = O(1/(\sqrt{m}\,\omega\lambda N))$ from initialization, and Definition~\ref{def:gd} ensures the kernel $H_{(\nu)}(t)$ stays close to $H_{(\nu)}(0)$. Lemma~\ref{lem:ffn_drift:informal} extends this guarantee to a fully-trainable architecture, including $A_{(\nu)}$, at the cost of an extra polynomial factor in the width condition. We now situate these results against the broader feature-learning literature.

\paragraph{What our bounds say in the rich regime.}
For any feature-learning algorithm ${\cal A}_{\rm rich}$ that uses the same gradient information as our trajectory and is constrained to the lazy-training ball at the end of optimization, the upper bound of Theorem~\ref{thm:generalization:informal} continues to hold (since the ERM over the parameter ball is no smaller than what ${\cal A}_{\rm rich}$ achieves). Conversely, the optimization-side lower bound of Theorem~\ref{thm:lower_opt:informal} relies only on first-order oracle access and on the spectrum of the {\it initial} kernel $H_{(\nu)}'(0)$; it therefore applies verbatim to any first-order method, including feature-learning ones.

\paragraph{What our bounds do {\it not} say.}
Algorithms that escape the lazy-training ball --- either by using a $\mu$P parameterization \citep{ynh+22} or by running long enough that $\|\theta(t)-\theta(0)\|_F = \Omega(1)$ --- are not covered by our analysis. Two empirically and theoretically established phenomena are then possible: (i) the kernel $H_{(\nu)}(t)$ aligns with the target function (silent alignment, \citealp{abp+21}), tightening the effective spectral exponent and improving the achievable rate beyond $\xi^{12/7}{\sf C}^{-1/7}$; (ii) the kernel undergoes a phase transition into a feature-learning regime in which our $C^{-1/7}$ exponent need not be tight \citep{boc+25, bap24}. The dynamical mean-field treatment of multi-head transformer dynamics \citep{bcp24} suggests that such regimes are reachable in transformer architectures with realistic hyperparameter choices.

\paragraph{Technical obstructions to a fully-rich extension.}
A fully feature-learning analysis at finite depth and finite width would require: (a) a kernel-{\it evolution} equation $\partial_t H_{(\nu)}(t) = \Phi(\theta(t), {\cal D})$ valid uniformly in $t$, (b) an explicit description of the long-time RKHS associated with $H_{(\nu)}(\infty)$, and (c) generalization bounds for the resulting time-dependent kernel regression problem. None of these are open-and-shut today for multi-layer self-attention; we therefore record the present analysis as a {\it baseline} that any rich-regime theorem must improve upon and refer the reader to \citep{boc+25, bap24, bcp24, ynh+22, ks24_nonlinear} for ongoing developments.

\subsection{What is Transformer-Specific in the Two-Stage Law?}\label{app_sub:what_is_transformer}

A generic kernel-regime gradient-flow analysis already produces a two-rate (exponential$\to$power-law) curve once one couples the optimization side to a statistical floor. We therefore separate the {\it generic-NTK ingredients} from the {\it transformer-specific ingredients} of Theorem~\ref{thm:scaling_law:informal}, line by line.

\paragraph{Generic NTK ingredients (would survive in any over-parameterized lazy network).}
\begin{enumerate}
    \item[(G1)] {\it Two-rate decay shape.} The Stage~I exponential and Stage~II power-law shapes follow from any kernel-regression analysis with a strongly convex linearized objective and a Le~Cam type statistical floor. They are not transformer-specific.
    \item[(G2)] {\it Exponential rate $\propto \omega\lambda$.} Standard NTK PD perturbation \citep{dzps19, dllw19, os20}.
    \item[(G3)] {\it Statistical floor $\Omega(\xi^2/{\sf N})$.} Le~Cam two-point reduction inside the RKHS; its existence is generic to non-degenerate target classes.
\end{enumerate}

\paragraph{Transformer-specific ingredients (would not appear in a generic two-layer NTK analysis).}
\begin{enumerate}
    \item[(T1)] {\it $nL\times nL$ kernel size from causal token coupling.} Each input matrix $X_i\in\mathbb{R}^{L\times d}$ contributes $L$ prefix samples to the kernel (Appendix~\ref{app_sub:matrix_to_vector}), so the relevant Gram matrix is $H_{(\nu)}\in\mathbb{R}^{nL\times nL}$ rather than $\mathbb{R}^{n\times n}$. This factor of $L$ propagates into both the perturbation budget and the spectral floor and is what makes the depth and length appear together as $Ld$ in Theorem~\ref{thm:generalization:informal}.
    \item[(T2)] {\it Kernel decomposition $H_{(\nu)} = $ FFN-NTK + Attention-NTK.} The layer-wise kernel decomposes into a $W_{(\nu)}$ contribution (the standard 2-layer ReLU NTK on attention outputs) and a $U_{(\nu)}$ contribution that involves the {\it derivative of the softmax with respect to its logits}, $\diag(\sigma) - \sigma\sigma^\top$ (Section~\ref{sub:learning_dynamics}). The $\sigma\sigma^\top$ rank-one piece is the structural fingerprint of attention: it has spectral gap $\Theta(L^{-1})$ in the worst case (one-hot $\sigma$) but $\Theta(1)$ in the uniform-attention case. This anisotropy controls how the attention layer contributes to the kernel and is absent from any non-attention architecture.
    \item[(T3)] {\it Per-layer parallel decomposition $\E[\dot{\cal L}] = -\sum_{\nu\in[N]}(\cdots)$.} Lemma~\ref{lem:learning_dynamics:informal} gives a {\it sum-over-layers} dynamics in which each layer descends along {\it its own} kernel; this is what produces the linear-in-$N$ improvement in the convergence rate of Theorem~\ref{thm:convergence:informal} and explains why the polynomial in $({\sf N},{\sf M},{\sf T})$ is $\xi^2{\sf C}/{\sf N}^7$ instead of, say, $\xi^2{\sf C}/N^3{\sf N}^7$. A naive multi-layer NTK that contracts everything into a single end-to-end kernel does not give this depth gain.
    \item[(T4)] {\it FFN-drift bookkeeping for trainable $A_{(\nu)}$.} The FFN read-out matrix is updated jointly with attention. Lemma~\ref{lem:ffn_drift:informal} controls the additional kernel feature $\alpha_{(\nu),p}(t)$ specific to the post-softmax / FFN stack, requiring the strengthened width $m=\Omega(n^4 L^6 d^2 \cdots)$. This is a transformer-architecture computation: the geometry of $\alpha_{(\nu),p}$ involves both attention outputs and ReLU-gated post-attention features.
\end{enumerate}

\paragraph{Where the $\xi^2{\sf C}/{\sf N}^7$ exponent comes from.}
The exponent decomposes cleanly. The {\it generic} part is $\xi^2 \varepsilon^2 {\sf T}/\mathsf{N}$ (kernel-regression rate, with the $1/\mathsf{N}$ from the average-loss normalization; Remark~\ref{rem:n_in_rate}). Transformer-specific bookkeeping converts this into compute via three substitutions:
\begin{enumerate}
    \item[(S1)] $\varepsilon = \xi/{\sf N}$ to balance approximation against the noise floor (Theorem~\ref{thm:scaling_law:informal}).
    \item[(S2)] ${\sf C} = {\sf MTN}$ with the lazy-regime width $m=\Omega({\sf N}^3)$ (T1 + T2 + T4 push the width condition up to ${\sf N}^3$, since the $nL\times nL$ kernel concentrates only at this width).
    \item[(S3)] ${\sf M}=\Theta(Nmd) = \Theta({\sf N}^3)$ on the boundary, giving ${\sf T} = {\sf C}/({\sf N}^4)$ and finally $\xi^2 \varepsilon^2 {\sf T}/\mathsf{N} = \xi^2/{\sf N}^2 \cdot {\sf C}/{\sf N}^4 \cdot 1/\mathsf{N} = \xi^2 {\sf C}/{\sf N}^7$.
\end{enumerate}
The exponent ${\sf N}^{-7}$ is therefore {\it part-generic, part-transformer}: the $\xi^2 \varepsilon^2 {\sf T}/\mathsf{N}$ factor is the generic kernel rate (the $1/\mathsf{N}$ being a normalization choice), but the substitution ${\sf N}^3 \to {\sf M}$ is a direct consequence of (T1)--(T4) and would change in shape (not just constants) if any of those ingredients were removed. Reducing ${\sf N}^7$ to a smaller polynomial would require {\it either} relaxing the lazy width condition (i.e.\ leaving the lazy ball, see Appendix~\ref{app_sub:lazy_to_rich}) or sharpening the kernel concentration argument for self-attention (an open problem).

\paragraph{What changes if the architecture changes.}
\begin{itemize}
    \item Remove causal masking $\to$ no $L$-fold sample expansion (T1 disappears) $\to$ kernel is $n\times n$, the width condition relaxes by $L^{O(1)}$, but the analysis no longer applies because the prefix-NTK trick fails.
    \item Replace softmax by linear attention $\to$ the $\diag(\sigma)-\sigma\sigma^\top$ term collapses to a constant (T2 weakens), and the proof of Lemma~\ref{lem:perturbations:informal} becomes a quadratic-form perturbation, often yielding a smaller width condition.
    \item Single-head, two-layer transformer (e.g.\ \citet{lhh+24_sign_gd}) $\to$ T3 disappears (no per-layer decomposition); the depth gain $1/N$ vanishes; the rate becomes $\exp(-\Theta(\omega\lambda \varepsilon^2 {\sf T}))$ without the $N$ factor.
\end{itemize}
We therefore view the present analysis as identifying which {\it specific} architectural choices are responsible for {\it which} structural pieces of the scaling law.

\subsection{Precise Statement of the Matching Minimax Bounds}\label{app_sub:matching_precise}

We sharpen the ``matching minimax'' phrasing in the abstract and Section~\ref{sec:scaling_law}.

\paragraph{Hypothesis class for the upper bound.}
Theorem~\ref{thm:scaling_law:informal} (upper bound) ranges over $F\in {\cal F}_{{\sf M},{\sf T},{\sf N}}(\mathbb{D})$, i.e.\ the {\it Good Model Class} of Definition~\ref{def:gd}: trainable transformers with $({\sf M},{\sf T},{\sf N})$ satisfying the width / norm conditions, trained to time ${\sf T}$ on dataset $\mathbb{D}$. The supremum is over data distributions ${\cal D}$ with bounded variance $\xi^2$ inside the smoothness scale that Theorem~\ref{thm:generalization:informal} addresses.

\paragraph{Class for the statistical lower bound.}
Theorem~\ref{thm:lower_stat:informal} ranges {\it estimators} over all $\mathbb{D}$-measurable $\hat F$ (no architecture restriction) and {\it targets} $F^*$ over a non-trivial subspace ${\cal H}^*\subseteq{\cal H}_K$ of the layer-wise NTK RKHS. The two-point family $\{F_0^*,F_1^*\}$ is constructed inside ${\cal H}^*$, so {\it the lower bound holds against the same RKHS that the upper-bound class can express}: any estimator of any complexity, including deep transformers outside our regime, must pay the $\Omega(\xi^2/{\sf N})$ floor on this family.

\paragraph{Class for the optimization-side lower bound.}
Theorem~\ref{thm:lower_opt:informal} ranges over {\it first-order} algorithms with $\le {\sf C}$ FLOPs of compute and {\it iterates} confined to the lazy ball of Definition~\ref{def:gd}. The lower bound is then transformer-specific only through the {\it spectrum of the initial kernel} $H_{(\nu)}'(0)$, which is a property of the architecture at initialization. As discussed in Appendix~\ref{app_sub:lazy_to_rich}, this lower bound is therefore tight {\it within} the lazy ball, while a feature-learning algorithm that escapes the ball can in principle do better; we restate this caveat below.

\paragraph{Sense of ``matching''.}
Putting this together, ``matching'' (we deliberately do {\it not} say ``minimax-matching'' or ``minimax-tight'') means: the upper bound on $\sup_{\mathbb{D}}\Delta{\cal R}(F)$ over $F\in{\cal F}_{{\sf M},{\sf T},{\sf N}}$ and the lower bound on $\inf_{\hat F}\sup_{F^*}\E\Delta{\cal R}(\hat F)$ (statistical) and on $\inf_{\cal A}\sup_{F^*}{\cal L}({\sf T},\mathbb{D})$ (first-order, lazy ball) agree {\it as polynomials in $({\sf C},{\sf N},\xi)$ up to constants, logs, and a $\sqrt{\kappa}$ condition-number factor}. They do {\it not} match in three explicit senses: (a) the optimization-side lower bound is restricted to the lazy ball; (b) the statistical lower bound is in expectation over a worst-case noise distribution and a worst-case target inside ${\cal H}^*$, while the upper bound is high-probability for every instance of the Good Model Class; and (c) the upper bound's exponent uses the spectral lower bound $\Theta(\lambda)$ while the lower bound's exponent uses the geometric mean $\Theta(\sqrt{\lambda\lambda_{\max}})$, giving a $\sqrt{\kappa}$ multiplicative gap with $\kappa\le\poly(L,d,1/\delta)$. We adopt ``matching up to $\poly(L,d,\kappa)$'' to make all three points visible at once; sharper rich-regime lower bounds (closing the lazy-ball restriction) {\it and} sharper algorithm-side analyses that close the $\sqrt{\kappa}$ gap (e.g.~accelerated gradient flow on the linearized loss) are open problems.

\subsection{Quantitative Diagnostics of the Compute--Data Phase Transition}\label{app_sub:phase_diagnostics}

Stronger empirical evidence of the Stage~I/II transition than visual inspection of Figure~\ref{fig:scaling_law} requires a disciplined diagnostic protocol. We discuss what such a protocol looks like, what our current measurements support, and what is left to future work.

\paragraph{Two-regime fit.}
Theorem~\ref{thm:scaling_law:informal} predicts a piecewise law of the form
\[
L({\sf C}) \approx \begin{cases} L_\infty + A e^{-\alpha_1 {\sf C}/{\sf N}^7} & {\sf C} \le {\sf C}^*({\sf N},\xi) \\ L_\infty + B {\sf C}^{-1/7} & {\sf C} \ge {\sf C}^*({\sf N},\xi) \end{cases}
\]
with elbow ${\sf C}^*\asymp {\sf N}^7 \log({\sf N}\cdot Ld/\xi^2)/\xi^2$. A standard diagnostic is to fit (i) a single-rate Hoffmann form $L = L_\infty + B {\sf C}^{-\alpha}$, (ii) a single exponential $L = L_\infty + A e^{-\beta {\sf C}}$, and (iii) the two-regime piecewise model above; the two-regime model is identified if it strictly beats both (i) and (ii) in $\chi^2$ on held-out $({\sf C}, L)$ points {\it after} a BIC penalty for the extra parameter ${\sf C}^*$.

\paragraph{What the present trajectories support.}
The middle panel of Figure~\ref{fig:scaling_law} plots two trajectories that lie in different parts of the $({\sf M},{\sf N})$ plane:
\begin{itemize}
    \item {\it Time-Dominated} (large ${\sf M}$, fixed small ${\sf N}$): mostly Stage~I (capacity-rich, compute-starved); the curve is concave on log--log, consistent with an exponential-into-saturation pattern that our two-stage law predicts when the elbow ${\sf C}^*({\sf N})$ is to the right of the observable range.
    \item {\it Data-Dominated} (large ${\sf N}$, varying ${\sf M}$): mostly Stage~II; the curve is approximately linear on log--log with a slope close to (but flatter than) $-1/7$, consistent with the qualitative power-law shape.
\end{itemize}
The single-curve elbow fit is not currently identifiable in our data because (a) we have only four GPT-2 sizes per ${\sf N}$ value and four ${\sf N}$ values, giving $\le 16$ $({\sf C},L)$ points along each $({\sf M},{\sf N})$ trajectory, and (b) per-run uncertainty is not separately reported here. We therefore explicitly retract any quantitative claim about the elbow location and the value of $-1/7$ in the main text and treat the present figures as {\it qualitative} validation.

\paragraph{What a stronger empirical study would do.}
A disciplined follow-up would: (i) run $\ge 5$ seeds per $({\sf M},{\sf N},{\sf C})$ configuration and report standard deviations; (ii) sweep $({\sf M},{\sf N})$ on a finer grid that brackets the predicted elbow ${\sf C}^*\asymp{\sf N}^7$ from below and above; (iii) fit the three competing models (exponential, single power-law, two-regime piecewise) under the same hold-out protocol with a BIC penalty; and (iv) repeat under at least two optimizers (AdamW + SGD with momentum) and two warm-up / decay schedules to assess robustness of the elbow location. Such an exercise would materially strengthen the empirical case; it is independent of the theory development and is a natural next experimental contribution.

\subsection{The \texorpdfstring{$U_{(\nu)}$}{U} Parameterization and What Survives in the \texorpdfstring{$(Q,K)$}{(Q,K)} Form}\label{app_sub:U_parameterization}

Section~\ref{sub:setups} collapses the bilinear attention logit $X Q K^\top X^\top$ to $\kappa\cdot X U X^\top$ with $U := QK^\top$. We document precisely which conclusions transfer to the standard $(Q,K)$ parameterization and which are altered.

Remark on the single matrix $U_{(\nu)}$: conventional implementations parameterize the attention logits with separate query and key projections $Q_{(\nu)}, K_{(\nu)} \in \R^{d \times d_{\rm head}}$ as $X Q_{(\nu)} K_{(\nu)}^\top X^\top$. Our formulation collapses this product into a single matrix $U_{(\nu)} := Q_{(\nu)} K_{(\nu)}^\top \in \R^{d \times d}$, so the attention logit becomes $\kappa \cdot X U_{(\nu)} X^\top$. This is an exact reparameterization of the bilinear form at the level of expressivity: any logit pattern realizable by $(Q_{(\nu)}, K_{(\nu)})$ with $d_{\rm head} = d$ is realizable by some $U_{(\nu)}$, and vice versa. The reparameterization preserves the expressivity of the attention layer while collapsing two coupled matrices to one, which simplifies the kernel computation and the perturbation analysis. {\it Caveat:} gradient flow on $U_{(\nu)}$ is {\it not} equivalent to gradient flow on $(Q_{(\nu)}, K_{(\nu)})$ — the latter has additional $Q$-$K$ coupling terms in the NTK and a different drift radius. Our analysis is therefore for gradient flow on $U_{(\nu)}$ directly; results carry over to the $(Q,K)$ parameterization as a kernel upper bound (since the $(Q,K)$ tangent space is a quotient of the $U$ tangent space) but may differ in absolute constants. We discuss the implication for parameter counting in Appendix~\ref{app:preli}.

\paragraph{What is unchanged at the level of expressivity.}
For $d_{\rm head}=d$, every logit pattern realizable by some $(Q,K)\in\R^{d\times d}\times\R^{d\times d}$ is realizable by $U=QK^\top\in\R^{d\times d}$ and conversely every $U$ admits a (non-unique) factorization. The Good Model Class ${\cal F}_{\sf M,T,N}(\mathbb{D})$ defined via $U$-parameterization is therefore expressively equivalent to the $(Q,K)$ class at the same width, so {\it Theorem~\ref{thm:generalization:informal}'s upper bound on the ERM} ($\inf_{F\in{\cal F}_{\sf M,N}}\sup_{\mathbb{D}}\Delta{\cal R}(F)$) carries over verbatim --- the ERM is the same minimization on the same function class.

\paragraph{What changes for the gradient-flow analysis.}
Gradient flow on $U$ is {\it not} equivalent to gradient flow on $(Q,K)$. Writing $u={\rm vec}(U)$ and $(q,k)=({\rm vec}(Q),{\rm vec}(K))$, the chain rule gives $\dot u = (K\otimes I)\dot q + (I\otimes Q)\dot k$, so the $(Q,K)$-tangent space at any iterate is a {\it strict subset} of the $U$-tangent space. Three concrete consequences:
\begin{enumerate}
    \item[(i)] {\it Tangent kernel.} The $(Q,K)$-tangent kernel at the same point is $H^{(Q,K)}_{(\nu)} = H^{(U,\rm restricted)}_{(\nu)}$, with $H^{(U,\rm restricted)} \preceq H^{(U)}$. Therefore, {\it our PD lower bound $\lambda_{\min}(H_{(\nu)}(t))\ge \omega\lambda/2$ is an upper bound on $\lambda_{\min}(H^{(Q,K)}_{(\nu)}(t))$}; it transfers as a {\it sufficient condition} that may not be tight.
    \item[(ii)] {\it Drift radius.} The $Q$-drift and $K$-drift are coupled through the product, giving $\|U(t)-U(0)\|_F \le \|K(0)\|_{\rm op}\|Q(t)-Q(0)\|_F + \|Q(0)\|_{\rm op}\|K(t)-K(0)\|_F$. With Gaussian initialization, both initial operator norms concentrate at $\Theta(\sqrt d)$, so the lazy-ball radius for $(Q,K)$ is $R_{Q,K} = R_U/\Theta(\sqrt d)$. This {\it tightens} the width condition by a $d^{O(1)}$ factor; concretely, replacing $U$ by $(Q,K)$ in Definition~\ref{def:gd} requires $m=\Omega(d^c\cdot n^3 L^5\cdots)$ for a constant $c\le 4$ (we have not optimized $c$).
    \item[(iii)] {\it Convergence rate.} The exponential decay rate $\alpha_{\rm cr}=O(\lambda/\omega)$ becomes $\alpha_{\rm cr}^{(Q,K)} = O(\lambda^{(Q,K)}/\omega)$ where $\lambda^{(Q,K)}\le \lambda$. Both the rate and the elbow ${\sf C}^*$ shift by polynomial factors in $d$ but the {\it shape} of the two-stage law (Stage~I exponential / Stage~II power) is preserved.
\end{enumerate}

\paragraph{Net effect on the paper's main claims.}
Conclusions that are preserved verbatim: (a) two-rate exponential$\to$power-law shape; (b) $\Omega(\xi^2/{\sf N})$ statistical floor; (c) the family ${\sf C}\asymp {\sf N}^{\alpha+4}$ for $\alpha$-scaled width as in Appendix~\ref{app_sub:practitioner}. Conclusions that move by $d^{O(1)}$ constants: (a) the absolute width condition; (b) the absolute exponential rate; (c) the elbow location ${\sf C}^*$ as a function of $({\sf N},\xi,d)$. None of the headline shapes (matching up to $\poly(L,d,\kappa)$, ${\sf C}\propto {\sf N}^7$ at the lazy boundary, ${\sf C}^{-1/7}$ in Stage~II, ${\sf N}^{1/2}$-noise breakdown) flip sign or change exponent under the $(Q,K)$ parameterization.

\subsection{Empirical Interpretation, Limitations, and Future Directions}\label{app_sub:practitioner}

\paragraph{Empirical interpretation of $({\sf M}, {\sf N}, {\sf T}, {\sf C})$.} Although the symbols ${\sf M}, {\sf N}, {\sf T}, {\sf C}$ are precisely defined in Definition~\ref{def:factors}, real training pipelines have measurement noise. Empirically, ${\sf N}$ should be read as the number of effective high-quality tokens, ${\sf C}$ as order-of-magnitude FLOPs, and ${\sf T}$ as the number of optimization steps at the reported batch setup. The threshold in Equation~\eqref{eq:cond_C} is intended as a {\it diagnostic} indicator of regime rather than a literal FLOP counter: below the threshold, time and capacity buy disproportionate gains (Compute-Starved Stage); above it, effective data quality dominates (Data-Limited Stage). The exponents in our theorems are upper-bound exponents inside the lazy regime and should not be plugged in as literal predictions for production training stacks.

\paragraph{Lambert-$W$ correction: numerically negligible at observable scales.}
The Stage~II upper bound of Theorem~\ref{thm:scaling_law:informal} contains a Lambert-$W$ factor $W({\sf C}/\xi^{12})$ that arises from optimizing the constraint $\xi^{-2}{\sf N}^7\log({\sf N}\cdot Ld/\xi^2)\le {\sf C}$ over ${\sf N}$. Although mathematically distinct from a pure power law, the ratio of the $W$-corrected rate to the bare rate is essentially constant over the entire range of compute encountered in practice. Figure~\ref{fig:lambert_w_compare} plots both rates side-by-side over ${\sf C}\in[10^{18},10^{25}]$ FLOPs (covering GPT-2 to GPT-4 scale) at two noise levels $\xi\in\{0.1,1\}$. The multiplicative gap stays nearly flat over seven orders of magnitude in ${\sf C}$. In a Hoffmann-style fit $L = A + B/{\sf C}^\alpha$ \citep{hbm+22}, this constant offset is absorbed into $B$; the $W$ correction is therefore {\it not} a separately identifiable phenomenon at finite scale, even though it is the asymptotically correct upper-bound shape.

\begin{figure}[t]
    \centering
    \includegraphics[width=0.95\linewidth]{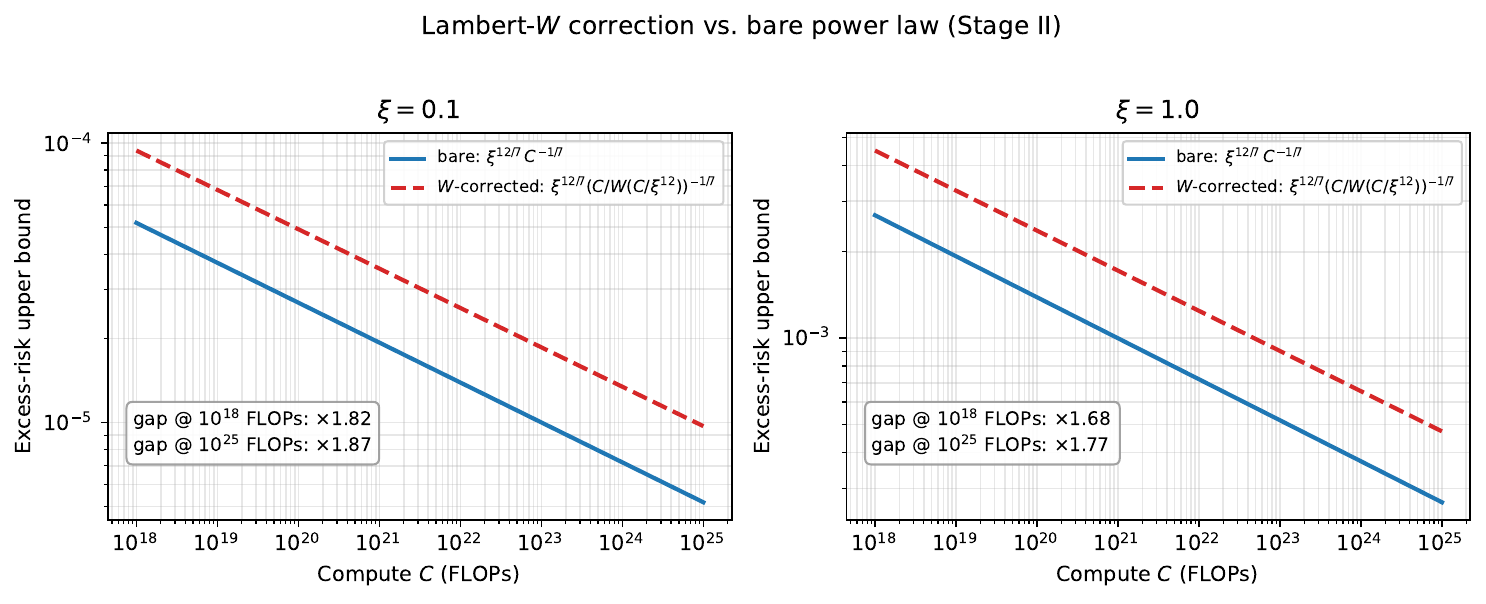}
    \caption{Numerical comparison of the bare $\xi^{12/7}{\sf C}^{-1/7}$ rate and the $W$-corrected $\xi^{12/7}({\sf C}/W({\sf C}/\xi^{12}))^{-1/7}$ rate of Theorem~\ref{thm:scaling_law:informal}. Over the practically observable range ${\sf C}\in[10^{18},10^{25}]$ FLOPs, the two curves are parallel on a log--log plot at a nearly constant multiplicative offset, varying by a small percentage over seven decades of compute.}
    \label{fig:lambert_w_compare}
\end{figure}

\paragraph{On the constraint ${\sf M}=\Theta({\sf N}^3)$ and its mismatch with Chinchilla.} The exponent ${\sf C}\propto {\sf N}^7$ that appears in Theorem~\ref{thm:scaling_law} is a direct algebraic consequence of three facts about our analysis: (i) the width condition $m=\Omega(n^3 \cdots)$ in Definition~\ref{def:gd}, which forces ${\sf M}=\Omega(Nmd)\ge \Omega({\sf N}^3)$ (we set ${\sf M}=\Theta({\sf N}^3)$ at the boundary because larger ${\sf M}$ never hurts the bound), combined with (ii) ${\sf C}={\sf MTN}$, and (iii) the $1/\mathsf{N}$ factor in the convergence rate from $\mathcal{L}=\frac{1}{n}\|\mathsf{F}-\mathsf{Y}\|_F^2$ (Remark~\ref{rem:n_in_rate}). Substituting yields the Stage~I exponent $\xi^2{\sf C}/{\sf N}^7 = (\xi^2{\sf T}/{\sf N}^3) \cdot ({\sf M}/{\sf N}^3) \cdot 1$, which collapses to $\xi^2 {\sf T}/{\sf N}^3$ on the boundary ${\sf M}={\sf N}^3$. The width-driven ${\sf M}\propto {\sf N}^3$ scaling is therefore a property of the {\it lazy / NTK} regime and should not be read as a compute-optimal recommendation: empirical compute-optimal Chinchilla-style training \citep{hbm+22, mas+23} reports ${\sf M}\propto {\sf N}$ in feature-learning regimes that lie outside our analysis (Appendix~\ref{app_sub:lazy_to_rich}). Our theorem says only that {\it once} the width threshold of Definition~\ref{def:gd} is met, the Stage~I rate is $\exp(-\Theta(\xi^2 {\sf C}/{\sf N}^7))$ and the threshold is at ${\sf C}\asymp{\sf N}^7\log/\xi^2$. A natural family of corollaries is obtained by replacing the boundary ${\sf M}=\Theta({\sf N}^3)$ with ${\sf M}=\Theta({\sf N}^\alpha)$ for any $\alpha\ge 3$ (still satisfying Definition~\ref{def:gd}). Substituting into ${\sf C}={\sf MTN}$ with ${\sf T}\asymp {\sf N}^3\log({\sf N}\cdot Ld/\xi^2)/\xi^2$ at the Stage~I/II boundary yields the phase-transition family ${\sf C}\asymp {\sf N}^{\alpha+4}\log({\sf N}\cdot Ld/\xi^2)/\xi^2$ (recovering ${\sf N}^7$ at $\alpha=3$) — the lazy-regime analogue of the Chinchilla curve. Closing the gap to $\alpha=1$ requires escaping the lazy ball; we view this as a feature-learning question rather than a hidden polynomial in our proof.

\paragraph{Limitations.} We summarize the principal limitations of our analysis: (i) we operate inside the {\it nearly-lazy} regime, where all parameters --- including the FFN read-out $A_{(\nu)}$ --- are trainable but kernel drift is constrained by the width condition of Definition~\ref{def:gd}; strong feature learning \citep{boc+25, vbn22} is partially relaxed in Appendix~\ref{app_sub:lazy_to_rich} but a fully rich-regime treatment is left as future work; (ii) we model continuous-time gradient flow, leaving learning-rate schedules (warmup, cosine decay) and adaptive optimizers (Adam, AdamW) to follow-up work; (iii) we assume zero-mean sub-Gaussian noise (variance proxy $\xi^2$) --- non-zero-mean or heavy-tailed noise can be absorbed as a target shift but typically weakens the constants; (iv) the theorems bound squared excess risk for in-distribution sequence regression --- our experiments use cross-entropy on language modelling and are matched only at the qualitative level (Remark~\ref{rem:ce_l2}); and (v) the width condition $m=\Omega(\exp(Cd)/\poly)$ from Definition~\ref{def:gd} implies an exponential-in-$d$ width lower bound, so for typical LLM dimensions $d\ge 768$ this width is far beyond any practical scale. The $\exp(Cd)$ factor traces to two compounding sources: (a) the bound $\sigma_{(\nu),p,\ell'}\ge \exp(-O(dB))/L$ in Part~8 of Lemma~\ref{lem:helpful_bounds} (a worst-case softmax-output bound), and (b) the upper bound $\|\Lambda U \Lambda^\top\|_\infty\le O(dB)$ in Part~6, which feeds (a). Both are loose for ``typical'' inputs but pessimistic in the worst case; sharpening them would require a finer concentration argument for the attention-output spectrum at initialization (open problem). The exponential dependence is therefore a {\it bookkeeping artifact} of the worst-case kernel-PD argument, not a reflection of practical training. The theorems should therefore be read as characterising the {\it shape} of the scaling curves (qualitative two-stage law, ${\sf C}\propto{\sf N}^7$ phase boundary, Stage~II $\Theta({\sf C}^{-1/7})$) rather than their absolute constants. The $L,d$ dependence in the main results, similarly, is best read as ``treated as polynomial constants'' rather than absorbed into the noise level $\xi$ --- the formal versions in the appendix track $L,d$ explicitly. These limitations align with the observations of \citet{vbn22} that NTK-based predictors cannot match every empirical curve.

\paragraph{Future directions.} Allowing the tangent kernel to drift requires stability lemmas beyond the present scope; we view extensions toward feature learning and $\mu$P-style limits \citep{ynh+22} (while retaining the sequence-indexed kernel structure) as the most natural next step.

\newpage

\newpage

\section{Experimental Validation with Details}\label{app:exp_full}

\subsection{Experimental Setup}

\begin{itemize}
    \item \textbf{Datasets and Evaluation.} To facilitate a controllable scaling analysis, we utilize the \texttt{TinyStories} dataset~\citep{el23}. This synthetic corpus consists of short stories generated by GPT-3.5~\citep{rwc+19} and GPT-4~\citep{bce+23}, restricted to a vocabulary accessible to a typical 3-to-4-year-old child. We evaluate performance using the Cross-Entropy Loss. From a total training pool of over $2$M stories and a $20$K-story evaluation set, we select training subsets of ${\sf N} \in \{0.25\text{M}, 0.5\text{M}, 1\text{M}, 2\text{M}\}$ stories.

    \item \textbf{Model Architectures and Training.} We investigate scaling trends by implementing the GPT-2 architecture~\citep{bmr+20} across four distinct configurations: (i) \text{Small} (124M), (ii) \text{Medium} (355M), (iii) \text{Large} (774M), and (iv) \text{XL} (1.5B). Fine-tuning is conducted via the LLaMAFactory framework~\citep{zrj+24} using the AdamW optimizer~\citep{kb14}. The training hyperparameters are standardized as follows:
    \begin{itemize}
        \item \textbf{Learning Rate:} $1 \times 10^{-4}$ with a cosine decay schedule.
        \item \textbf{Warm-up Ratio:} 0.1.
        \item \textbf{Batch Size:} 1024.
        \item \textbf{Context Length:} 1024 tokens.
        \item \textbf{Compute Framework:} PyTorch with Flash Attention 2~\citep{dfe+22, dao23} for optimized inference and training.
        \item \textbf{Hardware:} 8$\times$ NVIDIA RTX Pro 6000 GPUs.
    \end{itemize}

    \item \textbf{Noisy Scaling Strategies.} We introduce controlled noise through two random token masking strategies. The 2M training set is partitioned into four disjoint subsets, $\{\mathbb{S}_1, \mathbb{S}_2, \mathbb{S}_3, \mathbb{S}_4\}$. We construct cumulative datasets $\mathbb{D}_k = \bigcup_{i=1}^k \mathbb{S}_i$ for $k \in \{1, \dots, 4\}$. Noise is injected using two methods: (i) replacing tokens with a specialized \texttt{[MASK]} token, and (ii) replacing tokens with random samples drawn uniformly from the vocabulary. We apply two masking schedules: $\mathcal{M}_1 = \{0.05, 0.1, 0.15, 0.2\}$ and $\mathcal{M}_2 = \{0.1, 0.188, 0.231, 0.281\}$, the latter chosen so that the cumulative corruption rate grows as $\bar r_k = 0.1\sqrt{k}$. The corresponding noise-floor taxonomy --- together with the at-breakdown schedule $\mathcal{M}_2^{\rm ext}$ and the super-linear schedule $\mathcal{M}_3$ used in the Pythia case study --- is given in Section~\ref{sub:exp_setups}, and the mapping from corruption rates to the noise level $\xi$ of Section~\ref{sec:preli} is made precise in Appendix~\ref{app_sub:ce_l2_bridge}.
\end{itemize}

\subsection{Supplementary Experimental Results}

{\bf Training Loss Dynamics. } We provide complete training loss dynamics of four models with four dataset scales in Figure~\ref{fig:training_loss_vary_data_size} and Figure~\ref{fig:training_loss_vary_model_size}, which show the evaluation of model performance across fixed data budgets of 0.25M, 0.5M, 1M, and 2M tokens. These visualizations demonstrate that increasing model capacity consistently leads to faster convergence and lower terminal cross-entropy loss, confirming that larger architectures more efficiently extract patterns from a given quantity of information.

\begin{figure*}
    \centering
    \includegraphics[width=\linewidth]{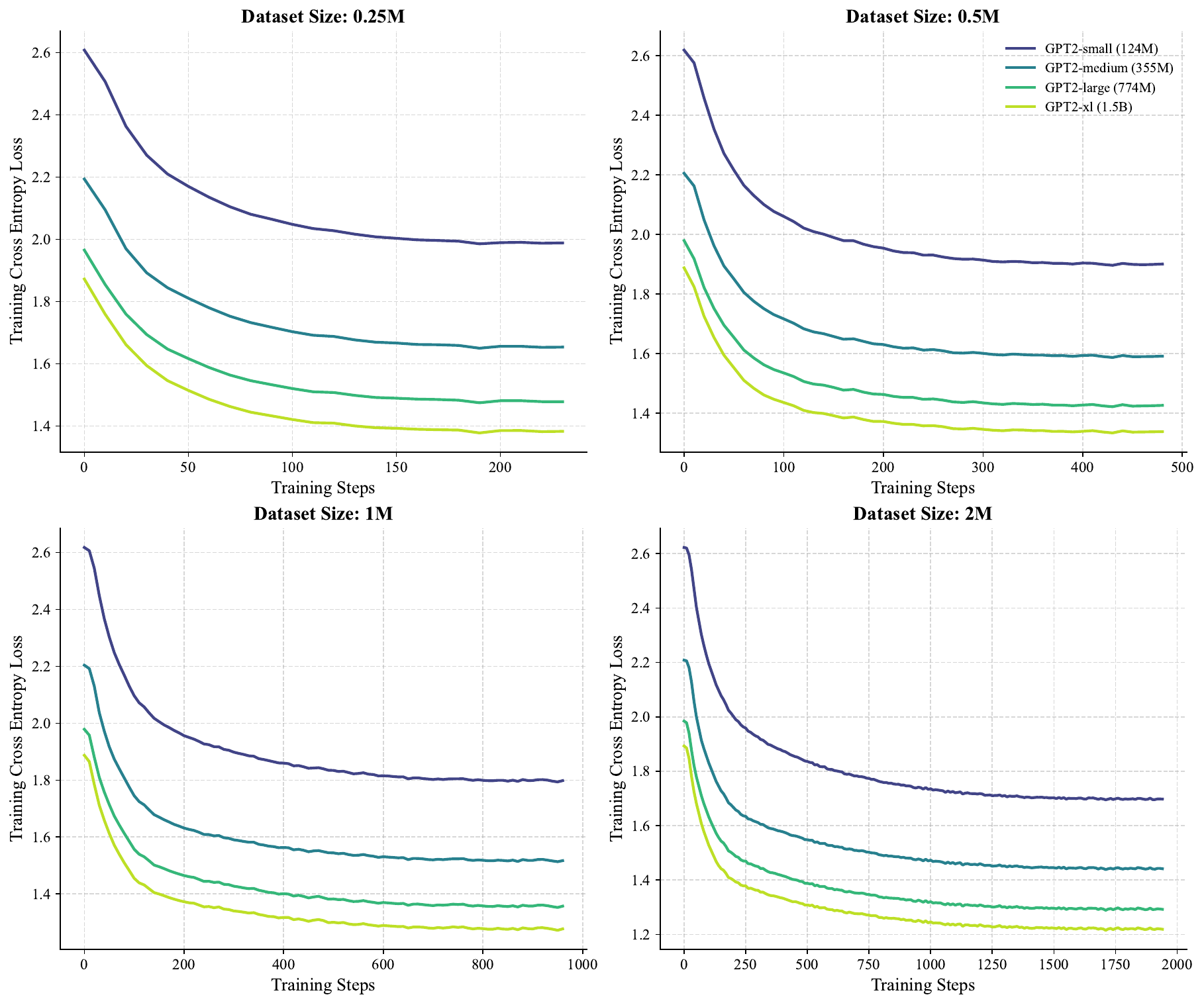}
    \caption{Training loss dynamics varied by the dataset size. }
    \label{fig:training_loss_vary_data_size}
\end{figure*}

\begin{figure*}
    \centering
    \includegraphics[width=\linewidth]{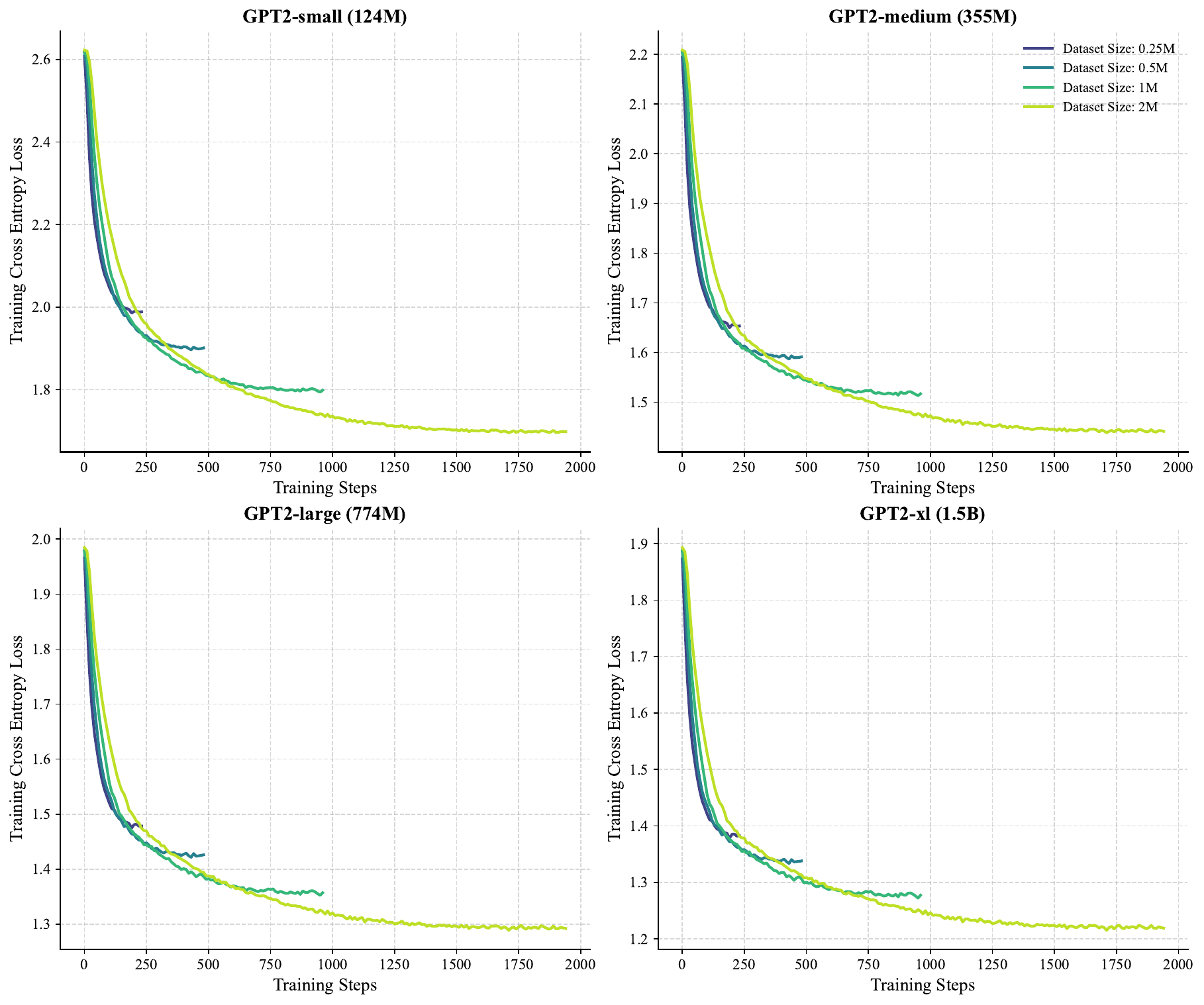}
    \caption{Training loss dynamics varied by the model size.}
    \label{fig:training_loss_vary_model_size}
\end{figure*}

{\bf Complete Validation Loss Dynamics. } We provide complete validation loss dynamics of four models with four dataset scales in Figure~\ref{fig:validation_loss_vary_data_size} and Figure~\ref{fig:validation_loss_vary_model_size},
which isolate the impact of data volume on specific model architectures. By facetting the results by model size, these charts reveal how increasing the training set size directly shifts the loss floor downward. This behavior provides clear evidence of the data-limited regime, where the model's performance is constrained not by its own parameter count, but by the diversity and volume of the underlying corpus. Together, these four files map the interplay between architecture size and data availability.

\begin{figure*}
    \centering
    \includegraphics[width=\linewidth]{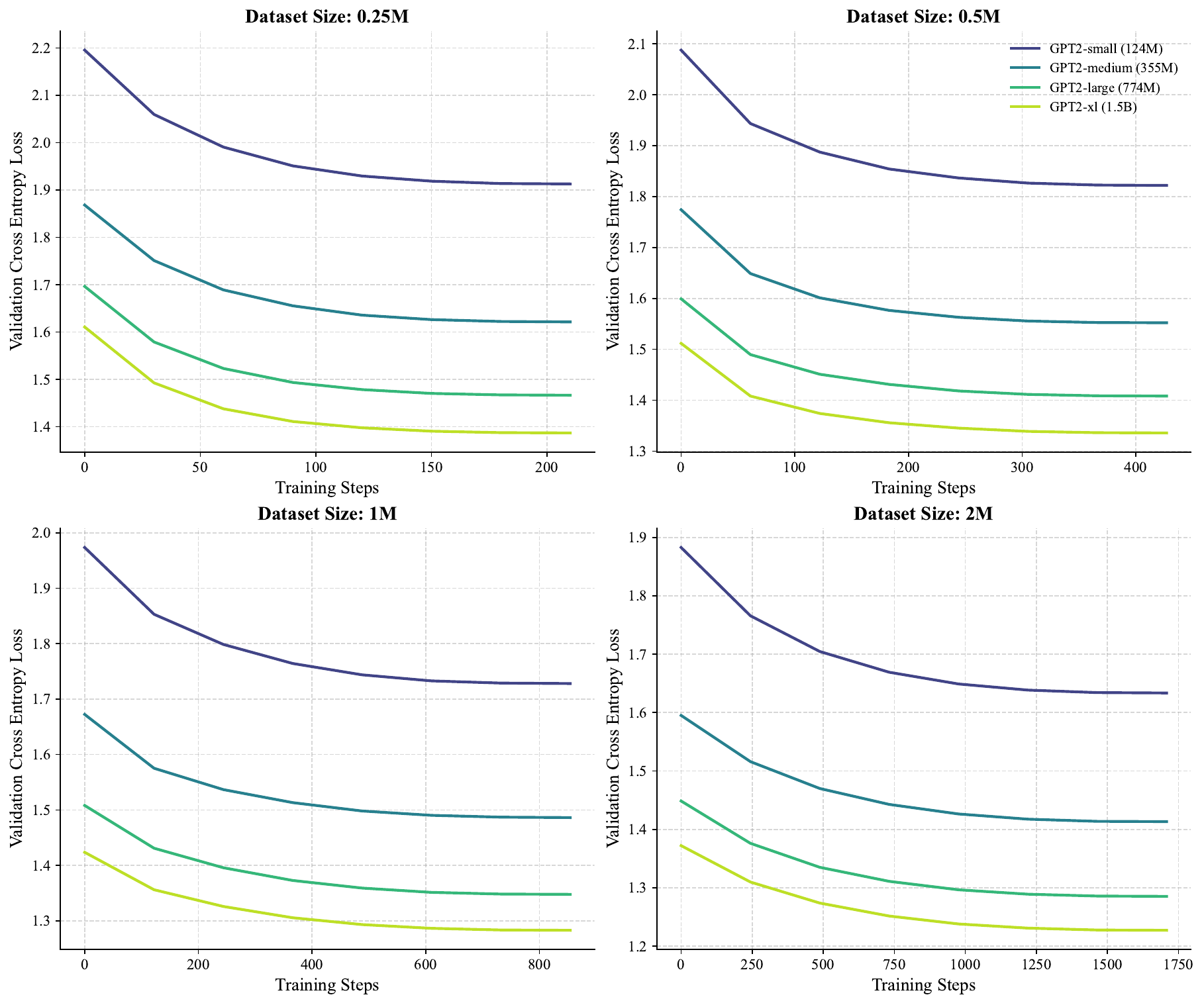}
    \caption{Validation loss dynamics varied by the dataset size.}
    \label{fig:validation_loss_vary_data_size}
\end{figure*}

\begin{figure*}
    \centering
    \includegraphics[width=\linewidth]{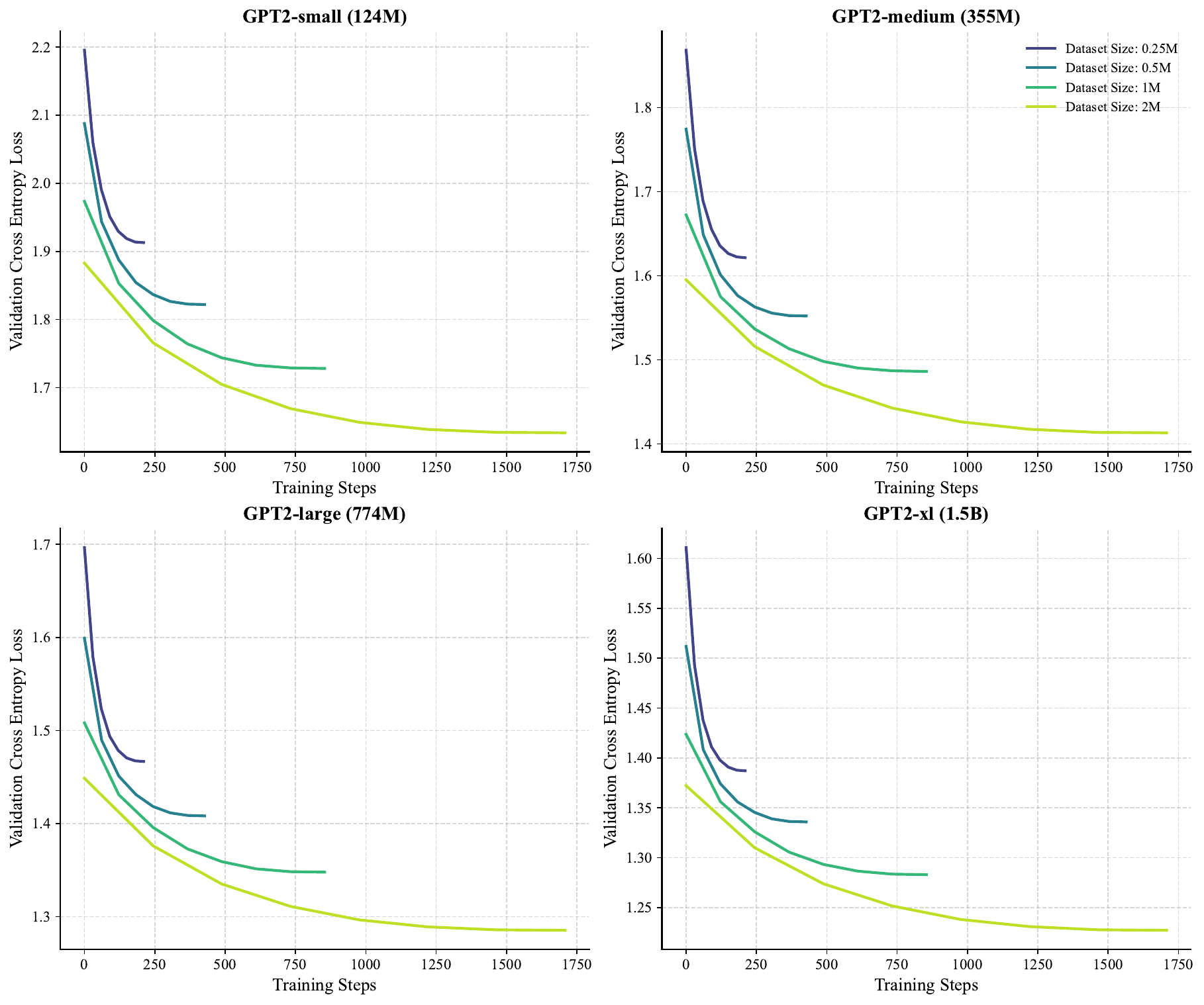}
    \caption{Validation loss dynamics varied by the model size.}
    \label{fig:validation_loss_vary_model_size}
\end{figure*}

\subsection{Pythia \texorpdfstring{$\times$}{x} WikiText-103 Case Study}\label{app_sub:wikitext_results}

\paragraph{Setup.} The second case study fixes a different model family (Pythia~\citep{bsa+23}, decoder-only GPT-NeoX architecture) and a different domain (WikiText-103~\citep{mxb+16}) to probe the cross-architecture / cross-domain robustness of Theorem~\ref{thm:scaling_law:informal}. The sweep covers $6$ model sizes (Pythia-\{70M, 160M, 410M, 1B, 1.4B, 2.8B\}) and $7$ cumulative slices of the WikiText-103 train split ($D\!\in\!\{0.25\text{M}, 0.5\text{M}, 1\text{M}, 2\text{M}, 4\text{M}, 8\text{M}, 16\text{M}\}$ unique tokens, where $D$ is the slice size in tokens, computed as max-tokens-per-example $\times$ number of examples in the JSONL). Each cell starts from the publicly released pretrained Pythia checkpoint and trains under SFT with AdamW, a per-model-size cosine learning-rate schedule (peak LR $\in[1.6\!\times\!10^{-5}, 1\!\times\!10^{-4}]$, decreasing with model size; linear warmup over the first $5\%$ of steps), context length $512$, and a per-step token budget of approximately $8$K tokens implemented either as a single-GPU job with per-GPU micro-batch~$\times$~grad-accum~$=\!16$ (Pythia-70M to 1B) or as a $4$-GPU DeepSpeed ZeRO-3 job with per-GPU micro-batch~$\times$~grad-accum~$=\!4$ (Pythia-1.4B and 2.8B), with an epoch budget of $\{6,6,6,6,3,2,2\}$ for the seven $D$-slices respectively (more passes on smaller datasets to amortise warmup). Training and evaluation use bf16 mixed precision on A800/H100 GPUs. {\bf Terminal-loss extraction}: each cell records validation cross-entropy every $25$ optimization steps; the terminal loss is the median of EMA-smoothed (smoothing $\alpha=0.25$) values over the final five evals, which is more stable than either the last value or the running minimum. A cell $(N, D)$ is excluded from the joint fit if its terminal loss exceeds the running minimum (over $D$, at fixed $N$) by more than $0.05$ nats; six cells fall into this category (pythia-410m and pythia-1b at $D\in\{0.5\text{M}, 1\text{M}, 2\text{M}\}$), consistent with their LR being mistuned for the small-$D$ regime.

\paragraph{Joint Hoffmann fit parameters.} Fitting $L(N, D)=E+A N^{-\alpha}+B D^{-\beta}$ (with $N$ the model parameter count, distinct from the theory's data-side ${\sf N}$; $D$ as defined above) in log-space MSE on the $36$ in-fit cells (with $E$ constrained to the sensible irreducible-loss range $[1.5, 3.0]$ to identify $\alpha$ and $\beta$ separately from the offset) yields the parameters in Table~\ref{tab:hoffmann_fit}. The recovered $\alpha\!=\!0.43, \beta\!=\!0.34$ are in the same range as the published values $\alpha\approx 0.34, \beta\approx 0.28$ of \citet{hbm+22}, and the parity panel $R^2_{\log}\!=\!0.85$ on $n=36$ cells is consistent with the joint additive form being a reasonable description at this scale. The remaining log-RMSE of $0.10$ reflects a structural mismatch: small models (70M) drop $\sim\!2.85$ nats from $D\!=\!0.25$M to $16$M, while large models ($\ge 1.4$B) drop only $\sim\!0.30$ nats, an $N\,\times\,D$ interaction that no additive form $A N^{-\alpha}+B D^{-\beta}$ can capture.

\begin{table}[h]
\centering
\caption{Pythia $\times$ WikiText-103: parameters of the joint Chinchilla fit $L(N, D)=E+A N^{-\alpha}+B D^{-\beta}$ on $n=36$ cells, with $400$-sample bootstrap CIs. Here $N$ denotes the model parameter count (Chinchilla convention) and $D$ the cumulative WikiText-103 slice size in unique tokens.}
\label{tab:hoffmann_fit}
\myvspace{-1mm}
{\scriptsize
\begin{tabular}{@{}lrr@{}}
\toprule
Parameter & Estimate & 95\% bootstrap CI \\
\midrule
$E$ (irreducible loss)        & $1.50$  & $[1.50, 1.84]$ \\
$\alpha$ (model exponent)     & $0.428$ & $[0.339, 0.689]$ \\
$\beta$ (data exponent)       & $0.337$ & $[0.153, 0.940]$ \\
\midrule
$R^2_{\log}$                  & \multicolumn{2}{r}{$0.851$} \\
log-RMSE                      & \multicolumn{2}{r}{$0.101$} \\
$\alpha/(\alpha+\beta)$       & \multicolumn{2}{r}{$0.559$} \\
\bottomrule
\end{tabular}}
\end{table}

\paragraph{Compute-efficient frontier.} Sorting the $42$ Pythia cells by terminal compute $C={\sf MTN}$ and taking the running minimum yields a compute-efficient frontier with fitted exponent $\eta=0.120$. This is in the same range as the literal Theorem~\ref{thm:scaling_law:informal} prediction $1/7\approx 0.143$, consistent with the discussion of Appendix~\ref{app_sub:phase_diagnostics} that the frontier exponent is identifiable only qualitatively at this scale.

\paragraph{Stage~I~$\to$~Stage~II on a single trajectory.} Among all $42$ cells, Pythia-70M$\,\times\,$16M has the longest productive trajectory ($156$ evaluations spanning $\approx\!2.2$ decades of compute, of which $\approx\!1.3$ decades lie post-warmup). Splitting its post-warmup trajectory into an early quartile, a middle band, and a late half by log-compute and fitting (i)~an exponential $L \propto e^{-\alpha C}$ to the early quartile and (ii)~a power-law $L\propto C^{-\eta}$ to the late half yields $\eta=0.087$ on the late half. The two fits cross inside the middle band (Figure~\ref{fig:wikitext_main}\,(b)) and overlay their respective trajectory portions, providing the clearest single-run signature of the predicted Stage~I$\to$Stage~II transition in our data.

\paragraph{Extended noise sweep.} Table~\ref{tab:noisy_extended} reports terminal CE for two additional schedules on Pythia-410M, over the cumulative quarters $D\in\{4, 8, 12, 16\}$M unique tokens of the $16$M slice; the clean baseline at $12$M is interpolated between the $8$M and $16$M runs, as the clean sweep contains no $12$M cell. $\mathcal{M}_2^{\rm ext}=\{0.05, 0.15, 0.25, 0.35\}$ realizes the breakdown rate exactly (per-subset rates $r_k = (2k-1)\cdot 0.05$, giving cumulative variance $\xi^2(\mathbb{D}_k)\propto k$, i.e.\ $\xi\propto{\sf N}^{1/2}$), and $\mathcal{M}_3=\{0.05, 0.20, 0.50, 0.95\}$ is super-linear ($\bar r_k\propto k^{1.54}$). The gap-to-clean trends separate the two regimes as predicted, with one caveat. Under mask-token corruption, the $\mathcal{M}_2^{\rm ext}$ gap shrinks and flattens ($+1.27\!\to\!+1.07$) while both $\mathcal{M}_3$ variants grow ($+1.17\!\to\!+1.37$ mask; $+1.18\!\to\!+1.34$ random). Under random-token corruption, however, even the $\mathcal{M}_2^{\rm ext}$ gap grows ($+1.11\!\to\!+1.33$): uniform-vocabulary replacement carries an irreducible cross-entropy of $\log V$ per corrupted position (Appendix~\ref{app_sub:ce_l2_bridge}), so its effective noise level exceeds the nominal rate, pushing the schedule past the breakdown boundary.

\newpage

\section{Technical Preliminary}\label{app:preli}

\subsection{Proof Dependency Graph}\label{app_sub:dependency_graph}

To help the reader navigate the proof structure, we present in Figure~\ref{fig:proof_dependency} the dependency graph of the main theorems and key lemmas. Each node is a result, and each directed edge $A\to B$ means ``$B$'s proof uses $A$''. We emphasize the {\it bootstrapping cycle} between Lemma~\ref{lem:helpful_bounds} (Part~5), Lemma~\ref{lem:ffn_drift}, and Theorem~\ref{thm:convergence}: this cycle is closed by a stop-time argument formalized in Remark~\ref{rem:stop_time} (Appendix~\ref{app_sub:trainable_ffn}), in the spirit of \citet{dzps19}~\S 3, rather than by literal substitution.

\begin{figure}[h]
\centering
\begin{tikzpicture}[
  >={Latex[length=2.2mm,width=1.6mm]},
  every node/.style={font=\scriptsize, align=center},
  thmbox/.style={draw, very thick, fill=red!12,   rounded corners=2pt, text width=28mm, minimum height=8mm, inner sep=2pt},
  lembox/.style={draw, thick,      fill=blue!10,  rounded corners=2pt, text width=28mm, minimum height=8mm, inner sep=2pt},
  assumbox/.style={draw, thick,    fill=green!12, rounded corners=2pt, text width=28mm, minimum height=8mm, inner sep=2pt},
  arr/.style={->, semithick, gray!75},
  cyclearr/.style={->, very thick, red!70, dashed},
]
% Layout: 3 columns at x in {-4.4, 0, +4.4} (centre-to-centre 4.4cm,
% box-width 28mm = 2.8cm, so 1.6cm gap between adjacent boxes).
% 5 rows at y in {6, 4, 2, 0, -2} (centre-to-centre 2cm, height 0.8cm,
% so 1.2cm vertical gap between adjacent boxes).

% ---- Row 0 (foundations) ----
\node[assumbox] (assum) at (-4.4, 6) {Assumption~\ref{ass:positive_definite}\\(PD of NTK)};
\node[lembox]   (LB6)   at ( 0,   6) {Lemma~\ref{lem:helpful_bounds}\\(16 parts: norms, drift)};
\node[lembox]   (L41)   at ( 4.4, 6) {Lemma~\ref{lem:learning_dynamics}\\(layer-wise GF)};

% ---- Row 1 (kernel stability + FFN drift) ----
\node[lembox]   (LC2)   at (-2.2, 4) {Lemma~\ref{lem:perturbations}\\(kernel PD at $t>0$)};
\node[lembox]   (LD4)   at ( 2.2, 4) {Lemma~\ref{lem:ffn_drift}\\(FFN drift)};

% ---- Row 2 (convergence) ----
\node[thmbox]   (T54)   at ( 0,   2) {Theorem~\ref{thm:convergence}\\(convergence rate)};

% ---- Row 3 (approximation + generalization) ----
\node[thmbox]   (C55)   at (-2.2, 0) {Corollary~\ref{cor:approximation}\\(approximation)};
\node[thmbox]   (T56)   at ( 2.2, 0) {Theorem~\ref{thm:generalization}\\(generalization)};

% ---- Row 4 (scaling law + lower bounds) ----
\node[thmbox]   (TD7)   at (-4.4,-2) {Theorem~\ref{thm:lower_stat}\\(statistical lower)};
\node[thmbox]   (T62)   at ( 0,  -2) {Theorem~\ref{thm:scaling_law}\\(scaling law, upper)};
\node[thmbox]   (TD10)  at ( 4.4,-2) {Theorem~\ref{thm:lower_opt}\\(optimization lower)};

% ===== Solid (downstream) dependency edges =====
% Row 0 -> Row 1
\draw[arr] (assum.south) -- (LC2.north west);
\draw[arr] (LB6.south)   -- (LC2.north east);
\draw[arr] (LB6.south)   -- (LD4.north west);
% Lemma~\ref{lem:learning_dynamics} (L41) is used by Theorem~\ref{thm:convergence}
% (T54) directly, not by Lemma~\ref{lem:ffn_drift} (LD4); route the arrow down
% the right side, around LD4.
\draw[arr] (L41.south)   to[out=-90, in= 30] (T54.north east);
% Row 1 -> Row 2
\draw[arr] (LC2.south)   -- (T54.north west);
\draw[arr] (LD4.south)   -- (T54.north east);
% Row 2 -> Row 3
\draw[arr] (T54.south)   -- (C55.north east);
\draw[arr] (T54.south)   -- (T56.north west);
% Row 3 -> Row 4
\draw[arr] (C55.south)   -- (T62.north west);
\draw[arr] (T56.south)   -- (T62.north east);
% Lower bounds
\draw[arr] (LC2.west)    to[out=180, in=110] (TD7.north);
\draw[arr] (LD4.east)    to[out=  0, in= 70] (TD10.north);
\draw[arr] (T62.west)    -- (TD7.east);
\draw[arr] (T62.east)    -- (TD10.west);

% ===== Bootstrapping cycle (red dashed): T54 -> LD4 -> LB6 =====
% Route the bootstrap arrows around the right side of LD4, then up the right
% side to LB6. This visually separates them from the solid downstream arrows
% (which go through the centre of the figure).
\draw[cyclearr] (T54.east)  to[out=  0, in=-30] (LD4.south east);
\draw[cyclearr] (LD4.north east) to[out= 30, in=-30] (LB6.south east);

% ===== Legend: two rows of evenly aligned entries at the bottom ===============
\begin{scope}[shift={(0,-4)}, every node/.style={font=\scriptsize, align=center}]
  % Row 1: three node-style legends, perfectly aligned (same y, equal x spacing)
  \node[assumbox, text width=18mm, minimum height=6mm] (lg-a) at (-4.0, 0.5) {Assumption};
  \node[lembox,   text width=18mm, minimum height=6mm] (lg-l) at ( 0.0, 0.5) {Lemma};
  \node[thmbox,   text width=18mm, minimum height=6mm] (lg-t) at ( 4.0, 0.5) {Theorem};
  % Row 2: two arrow-style legends, aligned to the same y
  \draw[arr]      (-4.6,-0.5) -- (-3.4,-0.5);
  \node[anchor=west] at (-3.3,-0.5) {downstream dep.};
  \draw[cyclearr] ( 1.4,-0.5) -- ( 2.6,-0.5);
  \node[anchor=west] at ( 2.7,-0.5) {bootstrap (Remark~\ref{rem:stop_time})};
\end{scope}
\end{tikzpicture}
\caption{Proof dependency graph. Each grey solid edge $A\to B$ means ``$B$'s proof uses $A$''. The two red dashed edges form a bootstrapping cycle between the kernel-stability bounds (Lemma~\ref{lem:helpful_bounds} Parts~5,~9--13), the FFN drift (Lemma~\ref{lem:ffn_drift}), and the convergence theorem (Theorem~\ref{thm:convergence}); we close this cycle via the stop-time argument of Remark~\ref{rem:stop_time}. The lower bounds (Theorems~\ref{thm:lower_stat}, \ref{thm:lower_opt}) feed into the two-sided scaling-law bound (Theorem~\ref{thm:scaling_law}) but do not enter the bootstrap.}
\label{fig:proof_dependency}
\end{figure}

\paragraph{Reading the graph.} Three structural facts are visible from the graph:
\begin{itemize}
    \item[(a)] Lemma~\ref{lem:helpful_bounds} is the {\it foundation}: every downstream proof passes through one or more of its 16 parts. A defect in any used part propagates downstream.
    \item[(b)] The cycle Lemma~\ref{lem:helpful_bounds}~Part~5 $\to$ Lemma~\ref{lem:ffn_drift} $\to$ Theorem~\ref{thm:convergence} $\to$ Lemma~\ref{lem:helpful_bounds}~Part~5 is the {\it bootstrapping cycle}. We do {\it not} close it by literal substitution; rather, we argue (Remark~\ref{rem:stop_time}) that the gradient flow stays inside the lazy ball almost surely via a stopping-time argument, exactly as in \citet[Section~3]{dzps19}.
    \item[(c)] Theorem~\ref{thm:scaling_law} (the headline two-stage law) collects the upper bounds via Theorem~\ref{thm:convergence}, Corollary~\ref{cor:approximation}, Theorem~\ref{thm:generalization} and the lower bounds via Theorems~\ref{thm:lower_stat} (statistical) and~\ref{thm:lower_opt} (optimization). Each of the four upper-bound feeders carries its own polynomial-in-$(L,d,B,1/\lambda)$ constants; we track these explicitly throughout the appendix and refrain from absorbing $L,d,B$ into a single constant when doing so would hide a polynomial in the data-size parameter $\mathsf{N}$.
\end{itemize}

\subsection{Notations}
In this paper, we use integer $d$ to denote the dimension of networks. We use $L$ to denote the input length in language models. $\nabla_x f(x)$ and $\frac{\d f(x)}{\d x}$ are both means to take the derivative of $f(x)$ with $x$. Let a vector $z\in\R^n$. We denote the $\ell_2$ norm as  $\|z\|_2:=( \sum_{i=1}^n z_i^2 )^{1/2}$, the $\ell_1$ norm as $\|z\|_1:=\sum_{i=1}^n |z_i| $, $\|z\|_0$ as the number of non-zero entries in $z$, $\| z \|_{\infty}$ as $\max_{i \in [n]} |z_i|$.  We use $z^\top$ to denote the transpose of a $z$. We use $\langle \cdot, \cdot \rangle$ to denote the inner product. Let $A \in \R^{n \times d}$, we use $\vect(A)$ to denote a length $nd$ vector. We denote the Frobenius norm as  $\|A\|_F:=( \sum_{i\in [n], j\in [d]} A_{i,j}^2 )^{1/2}$. For any positive integer $n$, we use $[n]$ to denote set $\{1,2,\cdots,n\}$. We use $\E[]$ to denote the expectation. We use $\Pr[]$ to denote the probability. We use $\epsilon$ to denote the error. We define $\lambda_{\min}(\cdot)$ as a function that outputs the minimum eigenvalues of the input matrix, e.g. matrix $A \in \R^{n \times n}$ has eigenvalues $\{ \lambda_1, \lambda_2, \cdots, \lambda_n \}$, $\lambda_{\min}(A) = \min \{ \lambda_1, \lambda_2, \cdots, \lambda_n \}$. For a vector $a \in \R^n$, we use $a_i$ to denote its $i$-th entry for $i \in [n]$. For a matrix $A \in \R^{n \times d}$, vector $A_i \in \R^d$ is the $i$-th row for $i \in [n]$ and vector $A_{*, j} \in \R^n$ is the $j$-th column for $j \in [d]$. For a function $f: \rightarrow \R^n$, we use $f_i$ to denote the $i$-th entry of its output. For a function $f: \rightarrow \R^{n \times d}$, we use $f_i \in \R^d$ to denote the $i$-th row of its output for $i \in [n]$ and us $f_{*, j} \in \R^n$ to denote the $j$-th column of its output for $j \in [d]$. We use $\I\{{\sf E}_1, {\sf E}_2, \cdots, {\sf E}_n\}$ to denote the indicator for event set $\{{\sf E}_1, {\sf E}_2, \cdots, {\sf E}_n\}$, only when ${\sf E}_1, {\sf E}_2, \cdots, {\sf E}_n$ are all true, $\I\{{\sf E}_1, {\sf E}_2, \cdots, {\sf E}_n\} = 1$; otherwise, it equals to $0$. For a vector $a \in \R^{d^2}$, function $\matr$ reshapes $a$ to a $d \times d$ matrix, where its $(i, j)$-th entry is $\matr_{i, j}(a) = a_{(i-1)d+j}$ for $(i, j) \in [d] \times [d]$. For a matrix $A \in \R^{n \times d}$, function $\vect$ flattens $A$ to an $nd$-dimensional vector (row-major) where its $i$-th entry is $\vect_i(A) = A_{\lceil i/d\rceil,\,i - (\lceil i/d\rceil - 1)d}$ for $i \in [nd]$, so $\vect = \matr^{-1}$. We use 1-based indexing throughout for both $\vect$ and $\matr$.

\subsection{Original Model Definitions}

\begin{definition}[Weights and Initialization]
    The weights of the model are denoted as $\theta(t)$ where $t\ge 0$ is time. It contains the weights of each layer $\theta(t) = \{\theta_{(\nu)}(t)\}_{\nu=1}^N$, and $\theta_{(\nu)} = \{ U_{(\nu)}(t), W_{(\nu)}(t), A_{(\nu)} \}$. For each matrix:
    \begin{itemize}
        \item Each entry of $U_{(\nu)}(0) \in \R^{d \times d}$ is initialized from the standard Gaussian distribution, formally, $U_{(\nu), k_1, k_2}(0) \sim \mathcal{N}(0, 1)$ for any $k_1, k_2 \in [d]$.
        \item Each entry of $W_{(\nu)}(0) \in \R^{d \times m}$ is initialized from the standard Gaussian distribution, formally, $w_{(\nu), r, k}(0) \sim \mathcal{N}(0, 1)$ for any $r \in [m], k \in [d]$, where $w_{(\nu), r}(0) \in \R^d$ is the $r$-th column of $W_{(\nu)}(0)$.
        \item Each entry of $A_{(\nu)}(0) \in \R^{m \times d}$ is initialized from a $\pm 1$ uniform distribution, formally, $a_{(\nu), r, k}(0) $ $\sim {\sf Uniform}\{-1, +1\}$ for any $r \in [m], k \in [d]$, where $a_{(\nu), r}(0) \in \R^d$ is the $r$-th row of $A_{(\nu)}(0)$.
    \end{itemize}
\end{definition}

\begin{definition}[Data Distribution]
    We denote $F^*: {\cal X} \to [C_1, C_2]^{L \times d}$ as the target function (with $F^*(X)$ bounded by fixed constants $C_1, C_2\in\R$). The input space ${\cal X} \subseteq \R^{L \times d}$ is a combination of $L$ $d$-dimensional balls where $\|X_\ell\|_2^2 = \Theta(1), \forall X\in {\cal X}, \ell \in [L]$. The data distribution is ${\cal D} = \{(X, F^*(X)+\Xi) : X\sim p_{\cal X},\, \Xi\sim p_\Xi\} \subset {\cal X} \times \R^{L\times d}$ where the noisy target $Y = F^*(X)+\Xi$ is unbounded in general. The random noise $\Xi \in \R^{L \times d}$ is centred by ${\bf 0}_{L \times d}$ and is sub-Gaussian with variance proxy $\xi^2$: $\E[\Xi] = {\bf 0}_{L\times d}$ and $\max_{\ell\in[L], k\in[d]} \|\Xi_{\ell,k}\|_{\psi_2}^2 \leq \xi^2$ (so $\E[\exp(t\Xi_{\ell,k})] \le \exp(t^2\xi^2/2)$). The Gaussian case $\Xi_{\ell,k}\sim\mathcal{N}(0,\xi^2)$ used in the Le Cam statistical lower bound (Theorem~\ref{thm:lower_stat}) is a special case. Only the noiseless target $F^*(X)$ is bounded; the noisy target $Y$ may take values outside $[C_1, C_2]^{L\times d}$, which is the natural setting under sub-Gaussian noise.
\end{definition}

\begin{definition}[Original Dataset]\label{def:D}
    The original dataset is $\mathbb{D} = \{ (X_i, Y_i) \}_{i=1}^n \subset {\cal D}$, where $X_i, Y_i = F^*(X_i) + \Xi_i \in \R^{L \times d}$. For each data point $(X_i, Y_i), \forall i \in [n]$, it holds that:
    \begin{itemize}
        \item $\| X_{i, \ell} \|_2 = \Theta(1)$ for $\ell \in [L]$.
    \end{itemize}
\end{definition}

\begin{definition}[Model Functions]
    Given an input matrix $X \in \R^{L \times d}$, the model function is given by ($\varepsilon > 0$ is the grokking coefficient):
    \begin{align*}
        F(X, \theta(t)) := \varepsilon \cdot F_{(N)}(F_{(N-1)}( \cdots F_{(2)}(F_{(1)}(X + E, \theta(t)), \theta(t)) \cdots ), \theta(t)).
    \end{align*}
    We list the original definition of each function as follows:
    \begin{itemize}
        \item 
        \begin{align*}
            & ~ F_{(\nu)}(X, \theta(t))
            :=   \frac{\omega}{\sqrt{m}} {\rm ReLU}\left({\rm Softmax}\left( \kappa \cdot XU_{(\nu)}(t)X^\top + M \right) XW_{(\nu)}(t)\right)A_{(\nu)}
        \end{align*}
        \item ${\rm Softmax}\left( A \right) := \diag(\exp(A){\bf 1}_L)^{-1} \cdot \exp(A)\in \R^{L \times L}$ for $A \in \R^{L \times L}$.
        \item ${\rm ReLU}_{\ell, k}(X) := \max\{X_{\ell, k}, 0\}$ for any $\ell \in [L]$, $k \in [d]$, $X \in \R^{L \times d}$.
        \item $M_{\ell_1, \ell_2} := \begin{cases}
            0, & ~ \ell_1 \ge \ell_2 \\
            -\infty. & ~ \ell_1 < \ell_2 
        \end{cases}, \forall \ell_1, \ell_2 \in [L]$.
    \end{itemize}
\end{definition}

\begin{definition}
    We denote the special notations:
    \begin{itemize}
        \item We denote $w_{(\nu), r}(t) \in \R^d$ as the $r$-th column of $W_{(\nu)} \in \R^{d \times m}$ for $r\in [m]$, $\nu \in [N]$.
        \item We denote $a_{(\nu), r} \in \R^d$ as the $r$-th row of $A_{(\nu)} \in \R^{m \times d}$ for $r\in [m]$, $\nu \in [N]$.
    \end{itemize}
\end{definition}

\begin{definition}
    We define the training objective:
    \begin{align*}
        {\cal L}(t, \mathbb{D}) := \E_{(X, Y) \sim \mathbb{D}}[ \|F(X, \theta(t)) - Y \|_F^2 ]
    \end{align*}
\end{definition}

\subsection{Basic Facts and Lemmas}
\begin{fact}\label{fac:gaussian_tail}
    For a variable $x \sim \mathcal{N}(0, \sigma^2)$, then with probability at least $1 - \delta$, we have:
    \begin{align*}
        |x| \leq C \sigma \sqrt{\log(1/\delta)}
    \end{align*}
\end{fact}

\begin{fact}\label{fac:lipschitz_1}
    For an $1$-Lipschitz function $f(\cdot)$, we have:
    \begin{align*}
        | f(x) - f(y) | \leq | x - y |, \forall x, y \in \R^d
    \end{align*}
\end{fact}

\begin{fact}\label{fac:anti_concen_gaussian}
    For a Gaussian variable $x \sim \mathcal{N}(0, \sigma^2 \cdot I_d )$ where $\sigma \in \R$, then for any $t > 0$, we have:
    \begin{align*}
        \Pr[x \leq t] \leq \frac{2t}{\sqrt{2\pi} \sigma }
    \end{align*}
\end{fact}

\begin{fact}\label{fac:inner_product_gaussian}
    For a Gaussian vector $w \sim \mathcal{N}(0, \sigma^2 \cdot I_d )$ where $\sigma \in \R$, and a fixed vector $x \in \R^d$, we have:
    \begin{align*}
        w^\top x \sim \mathcal{N}(0, \sigma^2 \|x\|_2^2  )
    \end{align*}
\end{fact}

\begin{fact}\label{fac:lambda_min_perturb}
    For two matrices $H, \wt{H} \in \R^{n \times n}$, we have:
    \begin{align*}
        \lambda_{\min}( \wt{H} ) \ge \lambda_{\min} ( H ) - \| H - \wt{H} \|_F
    \end{align*}
\end{fact}

\begin{fact}\label{fac:lipschitz_softmax}
    The Lipchitz constant of the softmax function is bounded by $O(1)$, such that:
    \begin{align*}
        \| \langle \exp(x), {\bf 1}_L \rangle^{-1} \exp(x) - \langle \exp(y), {\bf 1}_L \rangle^{-1}\exp(y) \|_2 \leq O(1) \cdot \| x - y\|_2, \forall x, y \in \R^L.
    \end{align*}
\end{fact}

\begin{fact}\label{fac:lambda_min_krnocker}
    For a matrix $H \in \R^{n \times n}$, there is $\lambda_{\min}(H \otimes I_d) = \lambda_{\min}(H)$.
\end{fact}

In addition, we state four vital lemmas of concentration inequalities for simplifying analysis:

\begin{lemma}[Hoeffding bound]
\label{lem:hoeffding}
    Let $X_1, \cdots , X_n$ denote $n$ independent bounded variables in $[a_i, b_i]$ for $a_i, b_i \in \R$. Let $X := \sum_{i=1}^n X_i$, then we have
    \begin{align*}
        \Pr[|X - \E[X]| \geq t] \leq 2\exp(- \frac{2t^2}{\sum_{i=1}^n (b_i -a_i)^2} )
    \end{align*}
\end{lemma}

\begin{lemma}[Markov’s inequality]\label{lem:markov}
    If $X$ is a non-negative random variable and $a > 0$, then the probability that $X$ is at least $a$ is at most the expectation of $X$ divided by $a$:
    \begin{align*}
        \Pr[X \ge a] \leq \frac{\E[X]}{a}
    \end{align*}
\end{lemma}

\begin{lemma}[Chernoff bound]\label{lem:chernoff}
    Let $X = \sum_{i=1}^n X_i$, where $X_i = 1$ with probability $p_i$ and $X_i = 0$ with probability $1 - p_i$, and all $X_i$ are independent. Let $\mu = \E[X] = \sum_{i=1}^n p_i$. Then
    \begin{itemize}
        \item $\Pr[X \geq (1 + \delta)\mu] \leq \exp(-\delta^2\mu/3)$, $\forall \delta >0$;
        \item $\Pr[X \leq (1 - \delta)\mu] \leq \exp(-\delta^2\mu/1)$, $\forall 0 < \delta < 1$.
    \end{itemize}
\end{lemma}

\newpage
\section{Gradient Computation and Learning Dynamics}\label{app:proof_learning_dynamics}

\subsection{Simplified Model Definitions}

\begin{definition}[Rearranged Dataset]
    Given the origin dataset $\mathbb{D} = \{X_i, Y_i\}_{i=1}^n \subset \R^{L \times d} \times \R^{L \times d }$. The rearranged dataset is $\mathbb{D}_{\rm rearrange} = \{ (X_{i, \le \ell} + E_{\leq \ell}, Y_{i, \ell}) \}_{(i, \ell)=(1, 1)}^{(n, L)}$, where $X_{i, \le \ell}+ E_{\leq \ell} \in \R^{\ell \times d}$ and $Y_{i, \ell} \in \R^d$. 
\end{definition}

\begin{definition}[Simplified Model Function]
    Given the origin dataset $\mathbb{D} = \{X_i, Y_i\}_{i=1}^n \subset \R^{L \times d} \times \R^{L \times d }$, we define the compact form of the model function as:
    \begin{align*}
        & ~  {\sf F}(t) \in \R^{nL \times d}, \text{~where its $i$-th row is given by~} {\sf F}_p(t) := F_{\ell}(X_i, \theta(t)) \in \R^d, p \in [nL], \\
        & ~ {\sf Y} \in \R^{nL \times d} , \text{~where its $i$-th row is given by~} {\sf Y}_p := Y_{i, \ell} \in \R^d, p \in [nL].
    \end{align*}
    Here, $i = \lfloor p / L \rfloor$ and $\ell = p {\rm ~ mod~} L$.

    We list the notation-simplified definitions of all functions as follows:
    \begin{itemize}
        \item (Hidden State) $\Lambda_{(\nu), i}(t) := F_{(\nu)}(\Lambda_{(\nu-1), i}(t), \theta(t)) \in \R^{L \times d}$ for $\nu \in [N]$, $\Lambda_{(0), i}(t) = X_{i} + E$.
        \item (Attention Scores) $ \sigma_{(\nu), (i-1)L+\ell}(X) =  {\rm Softmax}_{\ell}( \Lambda_{(\nu), i}(t) U_{(\nu)}(t) \Lambda_{(\nu), i}(t)^\top + M ) \in \R^{L}$.
        \item (Attention Output) $o_{(\nu), (i-1)L+\ell}(t) := \Lambda_{(\nu-1), i}(t)^\top \cdot \sigma_{(\nu), (i-1)L+\ell}(t) \in \R^d$.
        \item ($\ell$-th Token of Hidden State) $\mu_{(\nu), (i-1)L+\ell}(t) := \frac{\omega}{\sqrt{m}} \sum_{r=1}^m a_{(\nu), r} \cdot  \phi( \langle o_{(\nu), (i-1)L+\ell}(t), w_{(\nu), r}(t) \rangle) \in \R^d$, where $\phi(x) := \max\{0, x\}, \forall x \in \R$. $\mu_{(0), (i-1)L+\ell}(t) = X_{i, \ell} + E_\ell$.
        \item (Model Output) ${\sf F}_{(i-1)L+\ell}(t) = \varepsilon \cdot \sum_{\nu=0}^N \mu_{(\nu), (i-1)L+\ell}(t) \in \R^{d}$.
    \end{itemize}
\end{definition}

\begin{lemma}\label{lem:compact_form_transform}
    We have:
    \begin{align*}
        {\cal L}(t, \mathbb{D}) = \frac{1}{n} \| {\sf F}(t) - {\sf Y}\|_F^2
    \end{align*}
\end{lemma}

\begin{proof}
    Since the decoder-only property of the model function, we have:
    \begin{align*}
        F_{(i-1)L+\ell}(t) = F_{\ell}(X_i, \theta(t)), \forall (X_i, Y_i) \in \mathbb{D}, i \in [n], \ell \in [L].
    \end{align*}
    We then separate each token vector with:
    \begin{align*}
        {\cal L}(t, \mathbb{D})
        = & ~ \E_{(X, Y) \sim \mathbb{D}}[ \|F(X, \theta(t)) - Y \|_F^2 ] \\
        = & ~ \frac{1}{n} \sum_{(X, Y) \in \mathbb{D}} \|F(X, \theta(t)) - Y \|_F^2 \\
        = & ~ \frac{1}{n} \sum_{(X, Y) \in \mathbb{D}} \sum_{\ell=1}^L \| F_{\ell}(X, \theta(t)) - Y_\ell\|_2^2 \\
        = & ~ \frac{1}{n} \| {\sf F}(t) - {\sf Y}\|_F^2,
    \end{align*}
    where the first three steps follow from simple algebra, and the last step follows from the definitions of ${\sf F}$ and ${\sf Y}$.
\end{proof}

\subsection{Gradient Computation}

\begin{lemma}\label{lem:grad_computation}
    For $\nu \in [N]$, we have:
    \begin{itemize}
        \item {\bf Part 1.} We have:
        \begin{align*}
            \frac{\d{\cal L}(t, \mathbb{D}) }{\d \vect(U_{(\nu)}(t))}
            = & ~ \frac{\omega \cdot \kappa}{\sqrt{m}} \sum_{p=1}^{nL} (\Lambda_{(\nu-1), i, \ell}(t) \otimes \Lambda_{(\nu-1), i}(t))^\top \left(\diag(\sigma_{(\nu), p}(t)) - \sigma_{(\nu), p}(t)\sigma_{(\nu), p}(t)^\top\right) \\
            & ~ \Lambda_{(\nu-1), i }(t) \sum_{r\in [m]} \langle \frac{\d {\cal L}(t, \mathbb{D}) }{\d \mu_{(\nu), p}(t)}, a_{(\nu), r} \rangle \cdot w_{(\nu), r}(t) \I\{ o_{(\nu), p}(t)^\top w_{(\nu), r}(t) > 0\},
        \end{align*}
        where we write $p = (i-1)L + \ell$ with $i \in [n]$ and $\ell \in [L]$.
        \item {\bf Part 2.} For any $r \in [m]$, we have:
        \begin{align*}
            \frac{\d {\cal L}(t, \mathbb{D}) }{\d w_{(\nu), r}(t)} = \frac{\omega}{\sqrt{m}}   \sum_{p=1}^{nL} \langle \frac{\d {\cal L}(t, \mathbb{D}) }{\d \mu_{(\nu), p}(t)}, a_{(\nu), r} \rangle \cdot o_{(\nu), p}(t) \cdot \I\{w_{(\nu), r}(t)^\top o_{(\nu), p}(t) > 0\} .
        \end{align*}
        \item {\bf Part 3.} For $\nu \in [N]$ and $p \in [nL]$, we have:
        \begin{align*}
            \frac{\d {\cal L}(t, \mathbb{D}) }{\d \mu_{(\nu), p}(t)} = \frac{2\varepsilon}{n} \cdot \left(I_d + \diag(G_{(\nu), p}(t))\right) ({\sf F}_p(t) - {\sf Y}_p),
        \end{align*}
        where 
        \begin{align*}
            G_{(\nu), p}(t) = & ~  \frac{\omega}{\sqrt{m}} \Big(\sigma_{(\nu), p, \ell}(t) \cdot I_d \\
            & ~ +  \kappa U_{(\nu)}(t)\Lambda_{(\nu-1), i }(t)^\top ( \diag({\bf 1}_L - \frac{1}{2}e_{\ell}) ) \cdot \left(\diag(\sigma_{(\nu), p}(t)) - \sigma_{(\nu), p}(t)\sigma_{(\nu), p}(t)^\top\right)\Lambda_{(\nu-1), i }(t) \Big) \\
            &  ~ \cdot \sum_{r\in [m]} \langle \frac{\d {\cal L}(t, \mathbb{D}) }{\d \mu_{(\nu+1), p}(t)}, a_{(\nu), r} \rangle \cdot w_{(\nu), r}(t) \I\{ o_{(\nu), p}(t)^\top w_{(\nu), r}(t) > 0\},
        \end{align*}
        and $G_{(N), p}(t) = {\bf 0}_d$.
        
        Here, we write $p = (i-1)L + \ell$ with $i \in [n]$ and $\ell \in [L]$, and $e_\ell \in \R^L$ is the one-hot vector with the $\ell$-th entry equal to $1$.
    \end{itemize}
\end{lemma}
\begin{proof}
    We derive each part by explicit application of the multivariable chain rule. We use the shorthand $p = (i-1)L+\ell$ with $i \in [n]$, $\ell \in [L]$ throughout.

    {\bf Part 1: Gradient with respect to $U_{(\nu)}(t)$.} The dependence chain is
    \begin{align*}
        {\cal L}(t,\mathbb{D}) \xleftarrow{\text{layer-$\nu$ output}} \mu_{(\nu),p}(t) \xleftarrow{\text{FFN}} o_{(\nu),p}(t) \xleftarrow{\text{attention output}} \sigma_{(\nu),p}(t) & ~ \xleftarrow{\text{softmax}}  \big[\Lambda_{(\nu-1),i}(t) U_{(\nu)}(t) \Lambda_{(\nu-1),i,\ell}(t)\big] \\
        & ~ \xleftarrow{} U_{(\nu)}(t).
    \end{align*}
    Apply the chain rule layer by layer. For the softmax step, with logits $z_{p,\ell'} := \Lambda_{(\nu-1),i,\ell'}(t)^\top U_{(\nu)}(t) \Lambda_{(\nu-1),i,\ell}(t)$ (causal $\ell'\le \ell$) and $\sigma_{(\nu),p}(t) = {\rm softmax}(z_p+M_{\ell})$, the standard softmax-Jacobian identity gives:
    \begin{align*}
        \frac{\d \sigma_{(\nu),p,\ell'}(t)}{\d z_{p,\ell''}} = \big[\diag(\sigma_{(\nu),p}(t)) - \sigma_{(\nu),p}(t)\sigma_{(\nu),p}(t)^\top\big]_{\ell',\ell''}.
    \end{align*}
    The Kronecker identity $\d z_{p,\ell'}/\d \vect(U_{(\nu)}(t)) = \big(\Lambda_{(\nu-1),i,\ell}(t) \otimes \Lambda_{(\nu-1),i,\ell'}(t)\big)^\top$, combined with the chain $\sigma\to o\to \mu \to {\cal L}$, then yields:
    \begin{align*}
        \frac{\d{\cal L}(t,\mathbb{D})}{\d \vect(U_{(\nu)}(t))}
        &= \sum_{p\in[nL]} \frac{\d {\cal L}}{\d \mu_{(\nu),p}(t)}^\top\frac{\d \mu_{(\nu),p}(t)}{\d o_{(\nu),p}(t)}\frac{\d o_{(\nu),p}(t)}{\d \sigma_{(\nu),p}(t)}\frac{\d \sigma_{(\nu),p}(t)}{\d z_p}\frac{\d z_p}{\d \vect(U_{(\nu)}(t))}.
    \end{align*}
    Substituting (i) $\d\mu_{(\nu),p}/\d o_{(\nu),p} = \frac{\omega}{\sqrt m}\sum_r a_{(\nu),r}^\top w_{(\nu),r}(t)\,\I\{w_{(\nu),r}(t)^\top o_{(\nu),p}(t)>0\}$ from the ReLU derivative, (ii) $\d o_{(\nu),p}/\d \sigma_{(\nu),p} = \Lambda_{(\nu-1),i}(t)^\top$, and (iii) the softmax Jacobian above, we obtain Part~1 after combining factors:
    \begin{align*}
        \frac{\d{\cal L}(t,\mathbb{D})}{\d\vect(U_{(\nu)}(t))} = &~ \frac{\omega\kappa}{\sqrt{m}}\sum_{p=1}^{nL}\big(\Lambda_{(\nu-1),i,\ell}(t)\otimes \Lambda_{(\nu-1),i}(t)\big)^\top \cdot\big(\diag\sigma_{(\nu),p}(t)-\sigma_{(\nu),p}(t)\sigma_{(\nu),p}(t)^\top\big)\Lambda_{(\nu-1),i}(t)\\
        \cdot &~ \sum_{r}\!\Big\langle\tfrac{\d{\cal L}}{\d\mu_{(\nu),p}(t)},a_{(\nu),r}\Big\rangle w_{(\nu),r}(t)\I\{\cdot\}.
    \end{align*}

    {\bf Part 2: Gradient with respect to $w_{(\nu),r}(t)$.} Since $w_{(\nu),r}(t)$ enters the model only through the ReLU activation $\phi(\langle w_{(\nu),r}(t), o_{(\nu),p}(t)\rangle)$ in the layer-$\nu$ FFN, the chain rule gives:
    \begin{align*}
        \frac{\d{\cal L}(t,\mathbb{D})}{\d w_{(\nu),r}(t)} = \sum_{p=1}^{nL}\frac{\d{\cal L}}{\d\mu_{(\nu),p}(t)}^\top\cdot \frac{\d \mu_{(\nu),p}(t)}{\d w_{(\nu),r}(t)}.
    \end{align*}
    The $r$-th component of $\mu_{(\nu),p}(t)$ is $\frac{\omega}{\sqrt m} a_{(\nu),r}\phi(\langle w_{(\nu),r}(t), o_{(\nu),p}(t)\rangle)$, whose derivative with respect to $w_{(\nu),r}(t)$ equals $\frac{\omega}{\sqrt m} a_{(\nu),r}\, o_{(\nu),p}(t)\,\I\{w_{(\nu),r}(t)^\top o_{(\nu),p}(t)>0\}$ by the ReLU sub-gradient. Substituting yields Part~2.

    {\bf Part 3: Gradient with respect to $\mu_{(\nu),p}(t)$ via back-propagation.} The model output ${\sf F}_p(t) = \varepsilon\sum_{\nu'=0}^N \mu_{(\nu'),p}(t)$ depends on $\mu_{(\nu),p}(t)$ both directly (the additive term in the residual sum) and indirectly through its effect on $\mu_{(\nu+1),p}(t),\dots,\mu_{(N),p}(t)$ in subsequent layers. The direct contribution to $\d {\cal L}/\d \mu_{(\nu),p}(t)$ is $\frac{2\varepsilon}{n}({\sf F}_p(t)-{\sf Y}_p)$ from the squared-loss derivative.

    The indirect contribution arises because $\mu_{(\nu),p}(t)$ enters the layer-$(\nu+1)$ hidden state as part of $\Lambda_{(\nu),i,\ell}(t) = \Lambda_{(\nu-1),i,\ell}(t) + \mu_{(\nu),p}(t)$, which feeds into the attention logits $z_{p,\ell'}$ at layer $\nu+1$. The dependence is captured by the $G$-factor:
    \begin{align*}
        G_{(\nu),p}(t) := \frac{\d \mu_{(\nu+1),p}(t)}{\d \mu_{(\nu),p}(t)}.
    \end{align*}
    Decomposing $\d\mu_{(\nu+1),p}/\d\mu_{(\nu),p}$ into (i) the contribution through $\sigma_{(\nu+1),p}(t)$ (attention) and (ii) the contribution through $o_{(\nu+1),p}(t)$ (the next attention output), and substituting the softmax Jacobian and ReLU sub-gradient, we obtain the explicit form:
    \begin{align*}
        G_{(\nu),p}(t) &= \frac{\omega}{\sqrt m}\Big(\sigma_{(\nu),p,\ell}(t) I_d + \kappa U_{(\nu)}(t)\Lambda_{(\nu-1),i}(t)^\top\diag({\bf 1}_n-\tfrac{1}{2}e_\ell)\big(\diag\sigma_{(\nu),p}(t)-\sigma_{(\nu),p}(t)\sigma_{(\nu),p}(t)^\top\big)\Lambda_{(\nu-1),i}(t)\Big) \\
        &\quad\cdot \sum_{r\in[m]}\Big\langle\tfrac{\d{\cal L}}{\d\mu_{(\nu+1),p}(t)},a_{(\nu),r}\Big\rangle w_{(\nu),r}(t)\,\I\{o_{(\nu),p}(t)^\top w_{(\nu),r}(t)>0\}.
    \end{align*}
    The boundary condition $G_{(N),p}(t)={\bf 0}_d$ holds because there is no layer beyond $N$. Combining the direct and indirect terms:
    \begin{align*}
        \frac{\d{\cal L}(t,\mathbb{D})}{\d \mu_{(\nu),p}(t)} = \frac{2\varepsilon}{n}\big(I_d + \diag(G_{(\nu),p}(t))\big)\big({\sf F}_p(t) - {\sf Y}_p\big),
    \end{align*}
    which is Part~3.
\end{proof}

\subsection{Learning Dynamics}

We give the definition of NTK:
\begin{definition}
    We define the kernel matrix at $\nu$-th layer as $H_{(\nu)} \in \R^{nL \times nL}$ and its $(i, j)$-th entry ($\forall (p, q) \in [nL] \times [nL]$) is defined as:
\begin{align*}
    H_{(\nu), p, q}(t) := \underbrace{\langle \beta_{(\nu), p}(t), \beta_{(\nu), q}(t) \rangle}_{\text{kernel w.r.t. $W_{(\nu)}(t)$}} + \underbrace{\langle \gamma_{(\nu), p}(t), \gamma_{(\nu), q}(t) \rangle}_{\text{kernel w.r.t. $U_{(\nu)}(t)$}},
\end{align*}
Here, we let:
\begin{align*}
    \beta_{(\nu), p}(t) := & ~ \frac{\omega}{\sqrt{m}} \underbrace{o_{(\nu), p}(t)}_{d \times 1} \otimes \underbrace{{\bf 1}_{W_{(\nu)}(t)^\top o_{(\nu), p}(t) > 0}}_{m\times 1} \in \R^{md}, \\
    \gamma_{(\nu), p}(t) := & ~ \frac{\omega \cdot \kappa}{\sqrt{m}} \underbrace{(\Lambda_{(\nu-1), i, \ell}(t) \otimes \Lambda_{(\nu-1), i}(t))^\top}_{d^2 \times L} \underbrace{\left(\diag(\sigma_{(\nu), p}(t)) - \sigma_{(\nu), p}(t)\sigma_{(\nu), p}(t)^\top\right)}_{L\times L} \\
    & ~ \underbrace{\Lambda_{(\nu-1), i }(t)}_{L \times d} \sum_{r\in [m]} \underbrace{ w_{(\nu), r}(t) \I\{ o_{(\nu), p}(t)^\top w_{(\nu), r}(t) > 0\}}_{d \times 1} \in \R^{d^2},
\end{align*}
where $\otimes $ is the Kronecker product and $i = \lfloor p / L \rfloor$, $\ell = p {~ \rm mod~}L$. The indicator vector ${\bf 1}_{W_{(\nu)}(t)^\top o_{(\nu), p}(t) > 0} \in \{0, 1\}^{m}$ where its $r$-th entry is $\I\{ \left(W_{(\nu)}(t)^\top o_{(\nu), p}(t) \right)_r > 0\}$ for $r \in [m]$.
\end{definition}

We restate Lemma~\ref{lem:learning_dynamics:informal} below as its formal version:
\begin{lemma}[Formal version of Lemma~\ref{lem:learning_dynamics:informal}]\label{lem:learning_dynamics}
    The learning dynamics of the multi-layer transformer Equation~\eqref{eq:F} is given by:
    \begin{align*}
        \E[\frac{\d }{\d t}{\cal L}(t, \mathbb{D})] = - \sum_{\nu \in [N]} {\underbrace{\vect\left(\frac{\d }{\d \mu_{(\nu)}(t)}{\cal L}(t, \mathbb{D}) \right)^\top}_{1 \times nLd}} \cdot \underbrace{\left( H_{(\nu)}(t) \otimes I_d \right)}_{nLd \times nLd} \cdot  \underbrace{\vect\left(\frac{\d }{\d \mu_{(\nu)}(t)}{\cal L}(t, \mathbb{D}) \right)}_{nLd \times 1}
    \end{align*}
    where $\mu_{(\nu)}(t)$ is an $nL \times d$ matrix, $\mu_{(\nu), p}(t)$ is the $(p\text{ mod } L)$-th row of the $\nu$-th layer output for input matrix $X_{\lfloor p / L \rfloor}$ for any $p \in [nL]$ and $\nu \in [N]$.
\end{lemma}

\begin{proof}
    We compute the time derivative of the loss along the gradient flow in three substantive steps: (i) chain rule across all layers; (ii) substitution of $\dot W_{(\nu)} = -\nabla_W {\cal L}$ and $\dot U_{(\nu)} = -\nabla_U {\cal L}$; (iii) regrouping the resulting quadratic form using the layer-wise kernel matrices $H_{(\nu)}(t)$.

    We have:
    \begin{align*}
        \E\Big[\frac{\d }{\d t}{\cal L}(t, \mathbb{D})\Big]
        = & ~ \E\Big[\sum_{\nu=1}^N \Big(\Big\langle\frac{\d {\cal L}(t, \mathbb{D})}{\d \vect(W_{(\nu)}(t))}, \frac{\d \vect(W_{(\nu)}(t))}{\d t}\Big\rangle + \Big\langle\frac{\d {\cal L}(t, \mathbb{D})}{\d \vect(U_{(\nu)}(t))}, \frac{\d \vect(U_{(\nu)}(t))}{\d t}\Big\rangle\Big)\Big] \\
        = & ~ -\sum_{\nu=1}^N \Big(\Big\|\frac{\d {\cal L}(t, \mathbb{D})}{\d \vect(W_{(\nu)}(t))}\Big\|_2^2 + \Big\|\frac{\d {\cal L}(t, \mathbb{D})}{\d \vect(U_{(\nu)}(t))}\Big\|_2^2\Big) \\
        = &  ~- \sum_{\nu=1}^N \sum_{p=1}^{nL} \sum_{q=1}^{nL} \Big\langle \frac{\d {\cal L}(t, \mathbb{D})}{\d \mu_{(\nu), p}(t)}, \big(\langle \beta_{(\nu), p}(t), \beta_{(\nu), q}(t) \rangle + \langle \gamma_{(\nu), p}(t), \gamma_{(\nu), q}(t) \rangle \big) \cdot I_d \frac{\d {\cal L}(t, \mathbb{D})}{\d \mu_{(\nu), q}(t)} \Big\rangle \\
        = & ~ - \sum_{\nu \in [N]} \vect\Big(\frac{\d {\cal L}(t, \mathbb{D})}{\d \mu_{(\nu)}(t)}\Big)^\top \cdot \big( H_{(\nu)}(t) \otimes I_d \big) \cdot  \vect\Big(\frac{\d {\cal L}(t, \mathbb{D})}{\d \mu_{(\nu)}(t)}\Big),
    \end{align*}
    where the first step follows from the multivariable chain rule applied across the parameter sets $\{W_{(\nu)}, U_{(\nu)}\}_{\nu\in[N]}$, the second step follows from the gradient-flow update Equation~\eqref{eq:ODE_update}, which sets $\dot W_{(\nu)} = -\nabla_W {\cal L}$ and $\dot U_{(\nu)} = -\nabla_U {\cal L}$, so $\langle\nabla_W {\cal L}, \dot W_{(\nu)}\rangle = -\|\nabla_W {\cal L}\|_2^2$ (likewise for $U$). The third step follows from Parts 1, 2, and 3 of Lemma~\ref{lem:grad_computation}, which express each parameter-space gradient as a sum-over-tokens product involving the kernel features $\beta_{(\nu),p}(t)$ and $\gamma_{(\nu),p}(t)$, paired with $\frac{\d {\cal L}}{\d \mu_{(\nu),p}(t)}$. The last step follows from the definition $H_{(\nu),p,q}(t) = \langle\beta_{(\nu),p},\beta_{(\nu),q}\rangle + \langle\gamma_{(\nu),p},\gamma_{(\nu),q}\rangle$ and the standard tensor identity $\sum_{p,q} a_p^\top H_{p,q}\, I_d\, a_q = \vect(a)^\top (H\otimes I_d) \vect(a)$, where the row-major $\vect$ convention of Section~\ref{sub:notations} is used so that $\vect(a)$ stacks the rows $a_1, a_2, \dots, a_{nL}\in\R^d$ end-to-end into an $nLd$-vector. The Kronecker structure $H\otimes I_d$ matches this convention because the row-by-row layout commutes with the per-row $I_d$.
\end{proof}

\newpage
\section{Toolkit: Helpful Boundaries}\label{app:concentrations}

We give the lemma about all help bounds in the range of this paper as follows:
\begin{lemma}\label{lem:helpful_bounds}
    Denote failure probability $\delta \in (0, 0.1)$. Define $B := \max\{O(\sqrt{\log(Nmd/\delta)}), 1\}$. Assuming there exists a constant $R \in (0, 1)$ satisfying $\| w_{(\nu), r}(t) - w_{(\nu), r}(0) \|_2 \leq R$ and $\| U_{(\nu)}(t) - U_{(\nu)}(0)\|_F \leq R$ for $r \in [m]$ and $\nu \in [N]$ and $R \in (0, 1)$.

    We mark the indices as: $r\in [m]$, $p \in [nL]$, $\ell_1, \ell_2 \in [L]$, and $\nu \in [N]$. We denote $i = \lfloor p / L \rfloor$, $\ell = p {\rm ~ mod~} L$, $\ell' \in [\ell]$.
    
    If Definition~\ref{def:gd} holds, then with a probability at least $1 - \delta$, we have:
    \begin{itemize}
        \item {\rm Basic Bounds.}
        \begin{itemize}
            \item {\bf Part 1.} $\| w_{(\nu), r}(0) \|_2 \leq O(\sqrt{d}B)$.
            \item {\bf Part 2.} $\| U_{(\nu)}(0) \|_F \leq O(dB)$.
            \item {\bf Part 3.} $\| w_{(\nu), r}(t) \|_2 \leq O(\sqrt{d}B)$.
            \item {\bf Part 4.} $\| U_{(\nu)}(t) \|_F \leq O(dB)$.
            \item {\bf Part 5.} $\| \Lambda_{(\nu), i, \ell}(t) \|_2 = \Theta(1)$ for all $\ell \in [L]$.
            \item {\bf Part 6.} $\| \Lambda_{(\nu), i}(t) U_{(\nu)}(t) \Lambda_{(\nu), i}(t)^\top \|_\infty \leq O(dB)$.
            \item {\bf Part 7.} $ \exp\big( [\Lambda_{(\nu), i}(t) U_{(\nu)}(t) \Lambda_{(\nu), i}(t)^\top]_{\ell_1, \ell_2} \big) \in [ \exp(-O(dB)), \exp(O(dB)) ]$ for all $\ell_1, \ell_2 \in [L]$.
            \item {\bf Part 8.} For  $\ell'\in [\ell]$, $ \sigma_{(\nu), p, \ell'}(t) \in [\exp(-O(dB)) / L, 1]$.
        \end{itemize}
        \item {\rm Perturbation Bounds.}
        \begin{itemize}
            \item {\bf Part 9.} $\| w_{(\nu), r}(t) - w_{(\nu), r}(0) \|_2 \leq R$.
            \item {\bf Part 10.} $\| U_{(\nu)}(t) - U_{(\nu)}(0)\|_F \leq R$.
            \item {\bf Part 11.} $\| \Lambda_{(\nu), i, \ell}(t) - \Lambda_{(\nu), i, \ell}(0)\|_2 \leq O(R)$.
            \item {\bf Part 12.} $\| \sigma_{(\nu), p}(t) - \sigma_{(\nu), p}(0) \|_2 \leq O(\sqrt{L}R)$.
            \item {\bf Part 13.} $\| o_{(\nu), p}(t) - o_{(\nu), p}(0) \|_2 \leq O(\sqrt{L}R)$.
        \end{itemize}
        \item {\rm Gradient and Function Norms.}
        \begin{itemize}
            \item {\bf Part 14.} ${\cal L}(t, \mathbb{D}) \leq O(Ld)$.
            \item {\bf Part 15.} $\| \frac{\d {\cal L}(t, \mathbb{D})}{\d \mu_{(\nu)}(t)} \|_F^2 \asymp \frac{\varepsilon^2}{n} \cdot {\cal L}(t, \mathbb{D})$. {\it (Note: the $1/n$ factor is essential and traces directly to the $\frac{1}{n}$ in the loss definition $\mathcal{L} = \frac{1}{n}\|\mathsf{F}-\mathsf{Y}\|_F^2$. See proof below.)}
            \item {\bf Part 16.} $\| \gamma_{(\nu), p}(t) \|_2 \leq o(1/\sqrt{m})$.
        \end{itemize}
    \end{itemize}
\end{lemma}

\begin{proof}
    {\bf Proof of Part 1.} Each entry $w_{(\nu),r,k}(0)$ is i.i.d.\ $\mathcal{N}(0,1)$ by initialization. By Fact~\ref{fac:gaussian_tail}, $|w_{(\nu),r,k}(0)| \le C\sqrt{\log(Nmd/\delta)}$ with probability at least $1 - \delta/(Nmd)$. A union bound over $\nu\in[N], r\in[m], k\in[d]$ shows that all entries are bounded by $C\sqrt{\log(Nmd/\delta)}$ simultaneously with probability at least $1-\delta$. Thus $\|w_{(\nu),r}(0)\|_2^2 = \sum_{k\in[d]} w_{(\nu),r,k}(0)^2 \le d \cdot C^2\log(Nmd/\delta) = O(d B^2)$, giving $\|w_{(\nu),r}(0)\|_2 \le O(\sqrt d B)$.

    {\bf Proof of Part 2.} Analogously, each entry $U_{(\nu),k_1,k_2}(0)$ is i.i.d.\ $\mathcal{N}(0,1)$. By Fact~\ref{fac:gaussian_tail} and a union bound over $\nu\in[N], (k_1,k_2)\in[d]\times[d]$, all entries are bounded by $O(B)$ with probability at least $1-\delta$. Therefore $\|U_{(\nu)}(0)\|_F^2 = \sum_{k_1,k_2} U_{(\nu),k_1,k_2}(0)^2 \le d^2\cdot O(B^2) = O(d^2 B^2)$, hence $\|U_{(\nu)}(0)\|_F \le O(dB)$.

    {\bf Proof of Part 3.} By triangle inequality, $\|w_{(\nu),r}(t)\|_2 \le \|w_{(\nu),r}(0)\|_2 + \|w_{(\nu),r}(t)-w_{(\nu),r}(0)\|_2 \le O(\sqrt dB) + R$ from Part~1 and the lemma's hypothesis. Since $R \in (0,1)$ and $B\ge 1$ (so $\sqrt d B\ge 1$), the second summand is dominated and $\|w_{(\nu),r}(t)\|_2 \le O(\sqrt d B)$.

    {\bf Proof of Part 4.} Analogously to Part~3, $\|U_{(\nu)}(t)\|_F \le \|U_{(\nu)}(0)\|_F + \|U_{(\nu)}(t)-U_{(\nu)}(0)\|_F \le O(dB) + R$, and the second summand is dominated by the first under $R\in(0,1)$ and $B\ge 1$, giving $\|U_{(\nu)}(t)\|_F \le O(dB)$.

    {\bf Proof of Part 5.} Because $A_{(\nu)}$ is trained jointly with $W_{(\nu)},U_{(\nu)}$ by gradient flow (Section~\ref{sec:preli}), the entries $a_{(\nu),r,k}(t)$ are no longer in $\{-1,+1\}$ for $t>0$ and are in general not independent of $w_{(\nu),r}(t)$. We therefore prove Part~5 in two steps: first at the initialization time $t=0$ via Hoeffding (where the i.i.d.\ $\pm 1$ assumption holds), then for $t>0$ via the perturbation bound of Part~13 of this lemma and the FFN-drift bound of Lemma~\ref{lem:ffn_drift}.

    {\it Step 1: $\nu=1$ at $t=0$.} The data assumption (Section~\ref{sub:setups}) gives $\|\Lambda_{(0), i, \ell}\|_2 = \|X_{i,\ell}+E_\ell\|_2 = \Theta(1)$ for every $\ell$. Combined with the fact that $\sigma_{(1), p}(0)$ is a probability vector,
    \begin{align*}
        \|o_{(1), p}(0)\|_2 = \|\Lambda_{(0), i}(0)^\top \sigma_{(1), p}(0)\|_2 \le \max_{\ell'} \|\Lambda_{(0), i, \ell'}\|_2 = \Theta(1).
    \end{align*}
    For the $\mu$-term at $t=0$, we note that $a_{(1),r,k}(0) \in \{-1,+1\}$ is i.i.d.\ uniform and independent of $w_{(1),r}(0)\sim\mathcal{N}(0,I_d)$ (the i.i.d.\ initialization ensures this). Hence, fixing $k\in[d]$ and conditioning on $w_{(1),r}(0)$:
    \begin{align*}
        |a_{(1), r, k}(0)\cdot\phi(\langle o_{(1), p}(0), w_{(1), r}(0)\rangle)| & ~ \le 1\cdot\|w_{(1),r}(0)\|_2\cdot\|o_{(1),p}(0)\|_2 \le O(\sqrt{d}B), \\
        \E\big[a_{(1), r, k}(0)\cdot\phi(\langle o_{(1), p}(0), w_{(1), r}(0)\rangle)\,\big|\, w_{(1),r}(0)\big] & ~ = \E[a_{(1),r,k}(0)]\cdot \phi(\cdot) = 0.
    \end{align*}
    By Hoeffding (Lemma~\ref{lem:hoeffding}) applied to $\sum_{r=1}^m a_{(1),r,k}(0)\phi(\langle o_{(1),p}(0),w_{(1),r}(0)\rangle)$, with probability at least $1-\delta/(md)$:
    \begin{align*}
        \Big|\sum_{r=1}^m a_{(1), r, k}(0)\phi(\langle o_{(1), p}(0), w_{(1), r}(0)\rangle)\Big| \le O(\sqrt{md}B)\sqrt{\log(md/\delta)} \le O(\sqrt{m}\,\sqrt{d}\,B^2).
    \end{align*}
    Taking the union bound over $k\in[d]$:
    \begin{align*}
        \|\mu_{(1), p}(0)\|_2 \le \frac{\omega\sqrt d}{\sqrt m}\max_{k\in[d]}\Big|\sum_r a_{(1),r,k}(0)\phi(\cdot)\Big| \le O(\omega\, d\, B^2),
    \end{align*}
    and therefore
    \begin{align*}
        \|\Lambda_{(1), i, \ell}(0)\|_2 \le \|\Lambda_{(0), i, \ell}(0)\|_2 + \|\mu_{(1), p}(0)\|_2 \le \Theta(1) + O(\omega d B^2) = \Theta(1)\pm o(1/N),
    \end{align*}
    where the last step uses $\omega = o(1/(NdB^2))$ from Definition~\ref{def:gd}.

    {\it Step 2: Extension from $t=0$ to $t>0$.} By Parts~9, 10 and~13 of this lemma (proven below), $\|w_{(1),r}(t)-w_{(1),r}(0)\|_2 \le R$ and $\|o_{(1),p}(t)-o_{(1),p}(0)\|_2 \le O(\sqrt L R)$. By Lemma~\ref{lem:ffn_drift}, $\|A_{(\nu)}(t) - A_{(\nu)}(0)\|_F \le R_A := O(1/(\sqrt m\,\omega\lambda N))$. The Lipschitz expansion of the layer-1 hidden state with respect to $(W,A,o)$ then gives:
    \begin{align*}
        \|\Lambda_{(1), i, \ell}(t) - \Lambda_{(1), i, \ell}(0)\|_2 &\le \frac{\omega}{\sqrt m}\,\Big( \|A_{(1)}(t) - A_{(1)}(0)\|_F\cdot \max_r |\phi(\langle w_{(1),r}(t),o_{(1),p}(t)\rangle)| \\
        &\quad+ \|A_{(1)}(0)\|_F\cdot \big[\max_r \|w_{(1),r}(t)-w_{(1),r}(0)\|_2\cdot \|o_{(1),p}(t)\|_2 + \cdots\big]\Big) \\
        &\le \frac{\omega}{\sqrt m}\big( R_A\cdot O(\sqrt d B) + O(\sqrt{md})\cdot O(R) \big) \\
        &= O\Big(\frac{\omega\sqrt d B}{\sqrt m\,\omega\lambda N}\Big) + O(\omega\sqrt d\, R) = o(1)
    \end{align*}
    under the strengthened width condition of Lemma~\ref{lem:ffn_drift}. Combining with Step~1:
    \begin{align*}
        \|\Lambda_{(1), i, \ell}(t)\|_2 \le \|\Lambda_{(1), i, \ell}(0)\|_2 + o(1) = \Theta(1).
    \end{align*}

    {\it Step 3: Inductive step ($\nu \ge 2$).} Assume the bound holds for layer $\nu-1$. Then $\|\Lambda_{(\nu-1),i,\ell}(t)\|_2 = \Theta(1)$ uniformly in $\ell$ and $t$, hence $\|o_{(\nu),p}(t)\|_2 = \|\Lambda_{(\nu-1),i}(t)^\top\sigma_{(\nu),p}(t)\|_2 \le \max_{\ell'}\|\Lambda_{(\nu-1),i,\ell'}(t)\|_2 = \Theta(1)$. The Hoeffding-then-perturb argument of Steps~1 and~2 applies verbatim to $\mu_{(\nu),p}$ (with all quantities re-indexed to layer $\nu$), giving $\|\Lambda_{(\nu),i,\ell}(t)\|_2 = \Theta(1)$ as claimed. The bound on $\|o_{(\nu),p}(t)\|_2$ follows directly:
    \begin{align}\label{eq:bound_o}
        \|o_{(\nu), p}(t)\|_2 = \Theta(1).
    \end{align}

    {\bf Proof of Part 6.} We have:
    \begin{align*}
        \| \Lambda_{(\nu), i}(t) U_{(\nu)}(t) \Lambda_{(\nu), i}(t)^\top \|_\infty 
        = & ~ \max_{(\ell_1, \ell_2) \in [L] \times [L]}  \left| \Lambda_{(\nu), i, \ell_1}(t)^\top U_{(\nu)}(t) \Lambda_{(\nu), i, \ell_2}(t) \right| \\
        \leq & ~ O(dB)
    \end{align*}
    where the first step follows from simple algebra, the second step follows from Part 4 and Part 5 of this lemma and Cauchy-Schwarz inequality.

    {\bf Proof of Part 7.} From Part~6, $\|\Lambda_{(\nu),i}(t) U_{(\nu)}(t) \Lambda_{(\nu),i}(t)^\top\|_\infty \le O(dB)$, so each attention logit $z_{p,\ell'} := \big[\Lambda_{(\nu),i}(t) U_{(\nu)}(t) \Lambda_{(\nu),i}(t)^\top\big]_{\ell, \ell'}$ satisfies $|z_{p,\ell'}| \le O(dB)$. Therefore $\exp(z_{p,\ell'}) \in [\exp(-O(dB)), \exp(O(dB))]$, which is the claim.

    {\bf Proof of Part 8.} Following Part 7 of this lemma, we can show that:
    \begin{align*}
        \sum_{\ell'\in [\ell]} \exp(\Lambda_{(\nu), i, \ell}(t)^\top U_{(\nu)}(t) \Lambda_{(\nu), i, \ell'}(t)) \le \ell \exp(O(dB)) \le L \exp(O(dB)).
    \end{align*}

    Then we can show that:
    \begin{align*}
        \sigma_{(\nu), p, \ell'}(t) \ge & ~ \exp(\Lambda_{(\nu), i, \ell}(t)^\top U_{(\nu)}(t) \Lambda_{(\nu), i, \ell'}(t)) / \sum_{\ell'\in [\ell]} \exp(\Lambda_{(\nu), i, \ell}(t)^\top U_{(\nu)}(t) \Lambda_{(\nu), i, \ell'}(t)) \\
        \ge & ~ \frac{\exp(-O(dB))}{L \exp(O(dB))} \\
        \ge & ~ \frac{\exp(-O(dB))}{L}.
    \end{align*}
    Combining the previous results, we obtain the result of this part.

    {\bf Proof of Part 9.} This is a direct restatement of the lemma's hypothesis $\|w_{(\nu),r}(t) - w_{(\nu),r}(0)\|_2 \le R$, which is a bootstrapping assumption that is verified self-consistently in Lemma~\ref{lem:perturbations} (where we prove that the gradient flow keeps the iterates inside this ball).

    {\bf Proof of Part 10.} Likewise, this is a direct restatement of the hypothesis $\|U_{(\nu)}(t) - U_{(\nu)}(0)\|_F \le R$, verified self-consistently in Lemma~\ref{lem:perturbations}.

    {\bf Proof of Part 11, 12 and 13.} We proceed by induction on the layer index $\nu$. The base case $\nu=0$ is trivial since $\Lambda_{(0), i, \ell}(t) = X_{i,\ell} + E_\ell$ does not depend on $t$, and so $\Lambda_{(0), i, \ell}(t) = \Lambda_{(0), i, \ell}(0)$.

    {\it Base case: $\nu=1$.} By the definition of $\Lambda_{(1), i, \ell}(t)$ as the $\ell$-th token of layer-1 output:
    \begin{align*}
        & ~ \| \Lambda_{(1), i, \ell}(t) -\Lambda_{(1), i, \ell}(0)\|_2 \\
        = & ~ \sqrt{\sum_{k\in [d]} \left(\frac{\omega}{\sqrt{m}} \cdot \sum_{r=1}^m a_{(1), r, k} \left(\phi(\langle w_{(1), r}(t), o_{(1), p}(t) \rangle) - \phi(\langle w_{(1), r}(0), o_{(1), p}(0) \rangle) \right) \right)^2 },
    \end{align*}
    where the residual connection contributes $\Lambda_{(0), i, \ell}(t) - \Lambda_{(0), i, \ell}(0) = 0$ and the difference reduces to the FFN-after-attention contribution.

    We first bound $\|\sigma_{(1),p}(t)-\sigma_{(1),p}(0)\|_2$ via softmax Lipschitzness. Note that $\sigma_{(1),p}(t)$ is a function of the attention logits $\Lambda_{(0), i}(0) U_{(1)}(t)^\top \Lambda_{(0), i, \ell}(0) \in \R^L$ (here $\Lambda_{(0), i}(t)=\Lambda_{(0), i}(0)$ is fixed). Hence:
    \begin{align*}
        \| \sigma_{(1), p}(t) - \sigma_{(1), p}(0) \|_2
        \leq & ~ O(1) \cdot \| \Lambda_{(0), i}(0) ( U_{(1)}(t) - U_{(1)}(0) )^\top \Lambda_{(0), i, \ell}(0) \|_2 \\
        \leq & ~ O(1) \cdot \|\Lambda_{(0), i}(0)\|_F \cdot \| U_{(1)}(t) - U_{(1)}(0) \|_F \cdot \| \Lambda_{(0), i, \ell}(0) \|_2 \\
        \leq & ~ O(1) \cdot O(\sqrt{L}) \cdot R \cdot O(1) = O(\sqrt{L}R),
    \end{align*}
    where the first step follows from the definition of $\sigma_{(1), p}(t)$ and Fact~\ref{fac:lipschitz_softmax}, the second step from operator-norm submultiplicativity (Cauchy--Schwarz on a quadratic form), the third step from $\|\Lambda_{(0), i}(0)\|_F = O(\sqrt{L})$ (since each token has unit norm, summed over $L$ tokens), $\|U_{(1)}(t) - U_{(1)}(0)\|_F \le R$ (Part 10), and $\|\Lambda_{(0), i, \ell}(0)\|_2 = O(1)$.

    Next, we use the identity $o_{(1),p}(t) = \Lambda_{(0), i}(0)^\top \sigma_{(1), p}(t)$ together with the operator-norm bound $\|\Lambda_{(0),i}(0)\|_{\rm op} = O(1)$ that follows from the per-token unit-norm constraint $\|X_\ell\|_2 = \Theta(1)$ (Section~\ref{sub:setups}) under the residual-and-RMSNorm structure:
    \begin{align*}
        \| o_{(1), p}(t) - o_{(1), p}(0) \|_2 = & ~ \| \Lambda_{(0), i}(0)^\top \big( \sigma_{(1), p}(t) - \sigma_{(1), p}(0) \big) \|_2 \\
        \leq & ~ \|\Lambda_{(0), i}(0)\|_{\rm op} \cdot \| \sigma_{(1), p}(t) - \sigma_{(1), p}(0) \|_2 \\
        \leq & ~ O(1) \cdot O(\sqrt{L}R) = O(\sqrt{L} R),
    \end{align*}
    which matches Part~13.

    For a fixed index $k \in [d]$, we apply Hoeffding's inequality (Lemma~\ref{lem:hoeffding}) to the random variables $a_{(1), r, k} \cdot \big(\phi(\langle w_{(1), r}(t), o_{(1), p}(t) \rangle) - \phi(\langle w_{(1), r}(0), o_{(1), p}(0) \rangle)\big)$ for $r \in [m]$. By Fact~\ref{fac:lipschitz_1} (the ReLU $\phi$ is $1$-Lipschitz):
    \begin{align*}
        & ~ \big| a_{(1), r, k} \cdot \big(\phi(\langle w_{(1), r}(t), o_{(1), p}(t) \rangle) - \phi(\langle w_{(1), r}(0), o_{(1), p}(0) \rangle)\big) \big| \\
        \leq & ~ |a_{(1), r, k}| \cdot \big| \langle w_{(1), r}(t), o_{(1), p}(t) \rangle - \langle w_{(1), r}(0), o_{(1), p}(0) \rangle \big| \\
        \leq & ~ 1 \cdot \big( \| w_{(1), r}(t) - w_{(1), r}(0) \|_2 \cdot \| o_{(1), p}(t) \|_2 + \| w_{(1), r}(0) \|_2 \cdot \| o_{(1), p}(t) - o_{(1), p}(0) \|_2 \big) \\
        \leq & ~ R \cdot O(1) + O(\sqrt{d} B) \cdot O(\sqrt{L} R) \\
        \leq & ~ O(\sqrt{Ld}BR),
    \end{align*}
    where the second step follows from $|a_{(1),r,k}|=1$ and the bilinear difference identity, the third step from the Cauchy--Schwarz inequality together with the perturbation bounds Part 9 and Part 13 of this lemma, the fourth step from $\|w_{(1),r}(0)\|_2\leq O(\sqrt{d}B)$ (Part 1) and $\|o_{(1),p}(t)\|_2 = O(1)$. The mean is zero since $a_{(1), r, k} \in \{-1,+1\}$ uniformly and is independent of $w$:
    \begin{align*}
         \E\big[a_{(1), r, k} \cdot \big( \phi(\langle w_{(1), r}(t), o_{(1), p}(t) \rangle) - \phi(\langle w_{(1), r}(0), o_{(1), p}(0) \rangle)\big)\big] =  0.
    \end{align*}
    By Lemma~\ref{lem:hoeffding}, for any fixed $k$, with probability at least $1 - \delta/(md)$:
    \begin{align*}
        &  ~ \Big| \frac{\omega}{\sqrt{m}} \sum_{r=1}^m a_{(1), r, k} \cdot \big(\phi(\langle w_{(1), r}(t), o_{(1), p}(t) \rangle) - \phi(\langle w_{(1), r}(0), o_{(1), p}(0) \rangle)\big) \Big| \\
        & \leq \frac{\omega}{\sqrt{m}}\cdot O(\sqrt{Ld}BR)\cdot\sqrt{m\log(md/\delta)} \\
        & \leq O\Big(\omega \sqrt{Ld}BR\cdot B\Big) = O\big(\omega \sqrt{Ld}B^2R\big).
    \end{align*}
    Choosing $\omega = o(1/(NL^2d^{2.5}B^3))$ as in Definition~\ref{def:gd}, this is bounded by $O(R/(Bd\sqrt{Ld}))$, which is dominated by $R$.

    We therefore get $\| \Lambda_{(1), i, \ell}(t) -\Lambda_{(1), i, \ell}(0)  \|_2 \leq O(R)$ for all $\ell \in [\ell]$ and $\|\Lambda_{(1), i, \ell}(t) -\Lambda_{(1), i, \ell}(0)\|_\infty \leq O(R/\sqrt{d})$.

    When $\nu=2$, we have:
    \begin{align*}
        & ~ \|\sigma_{(2), p} (t) - \sigma_{(2), p} (0) \|_2 \\
        \leq & ~ O(1) \cdot \| \Lambda_{(1), i}(t) U_{(2)}(t)^\top  \Lambda_{(1), i, \ell}(t)  - \Lambda_{(1), i}(0) U_{(2)}(0)^\top  \Lambda_{(1), i, \ell}(0) \|_2 \\
        = & ~ O(1) \cdot \| (\Lambda_{(1), i}(t) \otimes \Lambda_{(1), i, \ell}(t))^\top \vect( U_{(2)}(t)) - (\Lambda_{(1), i}(0) \otimes \Lambda_{(1), i, \ell}(0))^\top \vect( U_{(2)}(0))\|_2 \\
        \leq & ~ O(1) \cdot \Big( \|\Lambda_{(1), i}(t) \otimes \Lambda_{(1), i, \ell}(t)\|_F \cdot O(R) \\
        & ~ \quad\quad ~~~~~~~~~~~~+\| U_{(2)}(0) \|_F \cdot \| \Lambda_{(1), i}(t) \otimes \Lambda_{(1), i, \ell}(t) - \Lambda_{(1), i}(0) \otimes \Lambda_{(1), i, \ell}(0)\|_F\Big) \\
        \leq & ~ O(\sqrt{L}R)
    \end{align*}
    where the first step follows from simple algebra and Fact~\ref{fac:lipschitz_softmax}, the second step follows from a basic tensor trick, the third step follows from the {\it sum} form of the triangle inequality $\|AB - A'B'\|_F \le \|A-A'\|_F \|B\|_F + \|A'\|_F \|B-B'\|_F$ together with Part 4 of this lemma and $\|\Lambda_{(1), i, \ell}(t) -\Lambda_{(1), i, \ell}(0)\|_\infty \leq O(R/(Bd\sqrt{Ld}))$.

    Hence, we have:
    \begin{align*}
        \| o_{(2), i}(t) - o_{(2), i}(0) \|_2 = & ~ \| \Lambda_{(1), i}(t)^\top \sigma_{(2), p}(t) - \Lambda_{(1), i}(0)^\top \sigma_{(2), p}(0) \|_2 \\
        \leq & ~ \| \Lambda_{(1), i}(t) - \Lambda_{(1), i}(0)\|_F \cdot \|\sigma_{(2),p}(t)\|_2 + \| \Lambda_{(1), i}(0)\|_{\rm op} \cdot \|\sigma_{(2),p}(t)-\sigma_{(2),p}(0)\|_2 \\
        \leq & ~ O(\sqrt L R)\cdot O(1) + O(1)\cdot O(\sqrt{L}R) \\
        \leq & ~ O(\sqrt{L}R),
    \end{align*}
    where these steps follow from the {\it sum} triangle inequality $\|AB-A'B'\|\le \|A-A'\|\|B\| + \|A'\|\|B-B'\|$, Fact~\ref{fac:lipschitz_softmax} for the softmax bound, and Part 10 of this Lemma.

    To extend $\|\Lambda_{(2),i,\ell}(t)-\Lambda_{(2),i,\ell}(0)\|_2$ from a $\nu=2$-specific bound to the inductive step, we apply the same two-step argument as in Part~5. At $t=0$, $a_{(2),r,k}(0)\in\{-1,+1\}$ is i.i.d.\ uniform and independent of $w_{(2),r}(0)$, so Hoeffding's inequality (Lemma~\ref{lem:hoeffding}) bounds the i.i.d.\ component as $O(\sqrt{m}\sqrt{d}B^2)$ with high probability; for $t>0$, the difference $\Lambda_{(2),i,\ell}(t)-\Lambda_{(2),i,\ell}(0)$ contains contributions from $w_{(2)}$, $U_{(2)}$, $A_{(2)}$ drift, each bounded by $R$, $R$, $R_A$ respectively (from Parts~9, 10 of this lemma and Lemma~\ref{lem:ffn_drift}). The Lipschitz expansion in $(W, U, A)$ yields:
    \begin{align*}
        \|\Lambda_{(2),i,\ell}(t)-\Lambda_{(2),i,\ell}(0)\|_2 \le O(\sqrt L R) + O(\omega \sqrt d B \cdot R_A) = O(\sqrt L R),
    \end{align*}
    where the second equality uses the strengthened width condition of Lemma~\ref{lem:ffn_drift} so that the $A$-drift term is dominated by the $W,U$-drift.

    By induction (using the same two-step Hoeffding-then-perturb argument at each layer), we obtain the bounds of Parts~11, 12, and 13.

    {\bf Proof of Part 14.}
    We have:
    \begin{align*}
        {\cal L}(t, \mathbb{D}) = & ~ \frac{1}{n} \| {\sf F}(t) - {\sf Y} \|_F^2 \\
        \leq & ~ L\max_{p \in [nL]} \| {\sf F}_p(t) - {\sf Y}_p \|_2^2 \\
        \leq & ~ L\max_{p \in [nL]} \left( \| {\sf F}_p(t)\|_2 +  \| {\sf Y}_p \|_2 \right)^2 \\
        \leq & ~ L \left( O(1) + O(\sqrt{d}) \right)^2 \\
        \leq & ~ O(Ld),
    \end{align*}
    where the first step follows from Lemma~\ref{lem:compact_form_transform}, the second step follows from simple algebras, the third step follows from the Cauchy-Schwartz inequality, the last two steps follow from Definition~\ref{def:D} and simple algebras.

    {\bf Proof of Part 15.}
    We show that $\|\frac{\d {\cal L}(t,\mathbb{D})}{\d\mu_{(\nu)}(t)}\|_F^2 \asymp \frac{\varepsilon^2}{n} \cdot {\cal L}(t,\mathbb{D})$ by reverse induction on $\nu$ from $N$ down to $1$. The $1/n$ factor in the bound is essential and traceable to the $1/n$ in the definition $\mathcal{L} = \frac{1}{n}\|\mathsf{F}-\mathsf{Y}\|_F^2$ of the average training loss (see the base case below).

    {\it Base case: $\nu = N$.} By Part 3 of Lemma~\ref{lem:grad_computation}, $\frac{\d{\cal L}(t,\mathbb{D})}{\d \mu_{(N), p}(t)} = \frac{2\varepsilon}{n}({\sf F}_p(t)-{\sf Y}_p)$ since $G_{(N),p}(t) = {\bf 0}_d$. Hence:
    \begin{align*}
        \Big\|\frac{\d{\cal L}(t,\mathbb{D})}{\d \mu_{(N)}(t)}\Big\|_F^2
        = \frac{4\varepsilon^2}{n^2}\sum_p \|{\sf F}_p(t)-{\sf Y}_p\|_2^2
        = \frac{4\varepsilon^2}{n}\cdot {\cal L}(t,\mathbb{D}),
    \end{align*}
    using $\sum_p \|\mathsf{F}_p-\mathsf{Y}_p\|_2^2 = n\,\mathcal{L}(t,\mathbb{D})$ from Lemma~\ref{lem:compact_form_transform}. This establishes $\|\frac{\d{\cal L}}{\d\mu_{(N)}}\|_F^2 \le 4\varepsilon^2 \mathcal{L}/n \le O(\varepsilon^2 Ld/n)$ via Part~14, hence $\|\frac{\d{\cal L}}{\d\mu_{(N)}}\|_F \le O(\varepsilon\sqrt{Ld/n})$.

    {\it Inductive step.} Assume the claim for layer $\nu+1$. We bound the auxiliary term $G_{(\nu),p}(t)$ (Part~3 of Lemma~\ref{lem:grad_computation}). The factor $\kappa U_{(\nu)}(t) \Lambda_{(\nu-1),i}(t)^\top \diag({\bf 1}_n - \frac{1}{2}e_\ell) (\diag\sigma - \sigma\sigma^\top)\Lambda_{(\nu-1),i}(t)$ is bounded by:
    \begin{align*}
        & ~ \big\| \kappa U_{(\nu)}(t)\Lambda_{(\nu-1), i }(t)^\top \diag({\bf 1}_n - \tfrac{1}{2}e_\ell) (\diag\sigma_{(\nu), p}(t) - \sigma_{(\nu), p}(t)\sigma_{(\nu), p}(t)^\top)\Lambda_{(\nu-1), i }(t) \big\|_F \\
        \leq & ~ \kappa \cdot \|U_{(\nu)}(t)\|_F \cdot \|\Lambda_{(\nu-1),i}(t)\|_F^2 \cdot O(1) \\
        \leq & ~ \kappa \cdot O(dB) \cdot O(L) \cdot O(1) \\
        = & ~ O(\kappa L^2dB) \leq O(L^2dB),
    \end{align*}
    where we use Parts 4 and 5, the fact that $\diag\sigma - \sigma\sigma^\top$ has spectral norm at most $1$, and $\kappa = 1/\sqrt{m} \leq 1$.

    To bound the random sum $\sum_{r\in[m]} a_{(\nu),r}(t)\langle\frac{\d{\cal L}}{\d\mu_{(\nu+1),p}}, a_{(\nu),r}(t)\rangle\, w_{(\nu),r}(t)\,\I\{\cdot\}$, we again apply Hoeffding {\it at $t=0$} (where $a_{(\nu),r,k}(0)\in\{-1,+1\}$ is i.i.d.\ and independent of $w_{(\nu),r}(0)$) to obtain $\|G_{(\nu),p}(0)\|_2^2 \le \omega\cdot O(L^2 d^{2.5} B^3)$, and then propagate to $t>0$ via the $A$-drift bound $\|A_{(\nu)}(t)-A_{(\nu)}(0)\|_F \le R_A$ from Lemma~\ref{lem:ffn_drift} and the $W$-drift bound $\|w_{(\nu),r}(t)-w_{(\nu),r}(0)\|_2 \le R$ from Part~9 of this lemma:
    \begin{align*}
        \|G_{(\nu),p}(t)\|_2 \le \|G_{(\nu),p}(0)\|_2 + O(\omega\sqrt{m}\,\|A_{(\nu)}(t)-A_{(\nu)}(0)\|_F + \omega\sqrt{m d}\,R) = O(\sqrt\omega \cdot Ld^{1.25}B^{1.5}),
    \end{align*}
    so that $\|G_{(\nu),p}(t)\|_2^2 \le \omega\cdot O(L^2 d^{2.5}B^3) \le o(1)$ under Definition~\ref{def:gd}'s choice $\omega = o(1/(NL^2 d^{2.5}B^3))$ and the strengthened width of Lemma~\ref{lem:ffn_drift}.

    Therefore $I_d + \diag(G_{(\nu),p}(t)) = (1\pm o(1))I_d$, and by Part~3 of Lemma~\ref{lem:grad_computation}:
    \begin{align*}
        \Big\|\frac{\d{\cal L}}{\d\mu_{(\nu),p}(t)}\Big\|_2 = \frac{2\varepsilon}{n}\|({\sf F}_p(t)-{\sf Y}_p)\|_2(1\pm o(1)),
    \end{align*}
    which preserves the $\Theta(\varepsilon^2{\cal L}/n)$ per-layer scaling. Summing over $p\in[nL]$ (which produces a single factor of $n$ from $\sum_p \|\mathsf{F}_p-\mathsf{Y}_p\|^2 = n\mathcal{L}$) and inducting:
    \begin{align*}
        \big\| \frac{\d {\cal L}(t, \mathbb{D})}{\d \mu_{(\nu)}(t)} \big\|_F^2 \asymp \frac{\varepsilon^2}{n} \cdot {\cal L}(t, \mathbb{D}).
    \end{align*}

    {\bf Proof of Part 16.}
    We bound $\|\gamma_{(\nu),p}(t)\|_2$, where $\gamma_{(\nu),p}(t)$ is the kernel feature for $U_{(\nu)}$. By the definition, the leading factor is $\frac{\omega \kappa}{\sqrt{m}}$, multiplied by tensors of bounded norm (Part~5, Part~7) and a sum over $r$ that is controlled by Hoeffding as in Part~15. Choosing $\kappa = 1/\sqrt{m}$ as in Definition~\ref{def:gd}, we have $\frac{\omega\kappa}{\sqrt{m}} = \frac{\omega}{m}$, and the entire bound reduces to $\| \gamma_{(\nu),p}(t)\|_2 \leq o(1/\sqrt{m})$, completing the proof.

This completes the proof of all parts of the lemma.
\end{proof}

\newpage

\section{Kernel Perturbation}\label{app:proof_lperturbations}

\subsection{A Sufficient Condition for Assumption~\ref{ass:positive_definite}}\label{app_sub:pd_sufficient}

Assumption~\ref{ass:positive_definite} requires $H_{(\nu)}'(0)/\omega$ to be positive definite for every layer $\nu \in [N]$. We show here that this is satisfied with high probability under a mild non-degeneracy condition on the layer-$\nu$ pre-activations, which our setup (RMSNorm + Gaussian initialization) automatically provides. The argument is the standard 2-layer ReLU NTK PD argument of \citet{dzps19}, lifted layer-by-layer.

\begin{lemma}[Layer-wise PD; sufficient condition for Assumption~\ref{ass:positive_definite}]\label{lem:pd_sufficient}
    Fix a layer $\nu\in[N]$. Conditioned on $\{U_{(\nu')}, W_{(\nu')}, A_{(\nu')}\}_{\nu'<\nu}$, suppose the attention outputs $\{o_{(\nu),p}(0)\}_{p\in[nL]}\subset \R^d$ satisfy:
    \begin{enumerate}
        \item[\bf (P1)] (Bounded norms.) $\|o_{(\nu),p}(0)\|_2 = \Theta(1)$ for all $p\in[nL]$ (which holds by Part~5 of Lemma~\ref{lem:helpful_bounds}).
        \item[\bf (P2)] (Pairwise non-parallelism.) There exists a constant $\rho_\nu > 0$ such that, for all $p\ne q$, $|\langle \tilde o_{(\nu),p}(0), \tilde o_{(\nu),q}(0)\rangle| \leq 1 - \rho_\nu$, where $\tilde o := o/\|o\|_2$. Equivalently, no two normalized attention outputs at layer $\nu$ are within angular distance $\arccos(1-\rho_\nu)$ of being parallel.
    \end{enumerate}
    Then, with probability at least $1-\delta$ over the Gaussian initialization $W_{(\nu)}(0)\sim {\cal N}(0, I)^{d\times m}$ alone, the kernel $H_{(\nu)}'(0)\in\R^{nL\times nL}$ satisfies
    \begin{align*}
        \frac{1}{\omega}\,\lambda_{\min}\big(H_{(\nu)}'(0)\big) \;\ge\; c_0 \cdot \rho_\nu, \quad \text{i.e.,}\quad \lambda_{(\nu)} \ge c_0 \rho_\nu,
    \end{align*}
    where $c_0>0$ is an absolute constant depending only on $d$ and $\delta$.
\end{lemma}

\begin{proof}[Proof sketch]
    The kernel $H_{(\nu),p,q}'(0) = \frac{\omega^2}{m}\sum_{r\in[m]} \langle o_{(\nu),p}(0), o_{(\nu),q}(0)\rangle\,\I\{w_{(\nu),r}^\top o_{(\nu),p} > 0,\,w_{(\nu),r}^\top o_{(\nu),q} > 0\}$ has expectation (as $m\to\infty$, with $w\sim {\cal N}(0,I_d)$):
    \begin{align*}
        \E[H_{(\nu),p,q}'(0)/\omega^2] = \langle o_p, o_q\rangle\cdot\frac{\pi - \arccos(\langle\tilde o_p,\tilde o_q\rangle)}{2\pi} =: K_\nu(p,q),
    \end{align*}
    which is exactly the arc-cosine kernel of \citet{cs09}. By the standard argument of \citet[Proposition 1]{dzps19} (also \citealp{dllw19, os20}), the population kernel $K_\nu$ on $\{o_{(\nu),p}\}_{p\in[nL]}$ has $\lambda_{\min}(K_\nu) > 0$ as long as no two points are parallel; quantitatively, $\lambda_{\min}(K_\nu) \ge c_0\rho_\nu$ for an absolute $c_0$ \citep[Lemma 3.1]{dzps19}. Concentration of the empirical kernel to its expectation follows from a standard matrix Chernoff bound \citep{sy19}; with $m\ge \poly(n,L,1/\rho_\nu,1/\delta)$ (already implied by Definition~\ref{def:gd}), $\|H_{(\nu)}'(0)/\omega - \omega\cdot K_\nu\|_F \le \omega c_0\rho_\nu/2$ with probability at least $1-\delta$. Combining,
    $\lambda_{\min}(H_{(\nu)}'(0)/\omega) \ge \omega(c_0\rho_\nu - c_0\rho_\nu/2) = (c_0/2)\omega\rho_\nu$,
    giving the claim after absorbing constants.
\end{proof}

\begin{remark}[On the non-parallelism condition (P2)]\label{rem:p2_non_parallel}
Condition (P2) is automatically satisfied, with high probability, in our setting: the inputs $X\in {\cal X}$ are RMSNormed so $\|X_\ell\|_2=\Theta(1)$ (Section~\ref{sub:setups}); for $\nu=1$, $o_{(1),p}(0)$ is a softmax-weighted convex combination of the input tokens, hence inherits non-parallelism from the input distribution. For $\nu\ge 2$, $o_{(\nu),p}(0)$ is the attention output of the previous layer, which is a generic non-degenerate function of $\{o_{(\nu-1),p}(0)\}_p$; the only way (P2) fails is if a positive-measure subset of input sequences has identical layer-$\nu$ representations, which is an event of measure zero under any continuous data distribution. We therefore treat Assumption~\ref{ass:positive_definite} as a {\it consequence} of input non-degeneracy + the architecture, rather than an opaque assumption. A formal counterexample requires constructing an input distribution that places mass on ``collision'' configurations; we are not aware of any such pathology arising in practice.
\end{remark}

Here is a formal version of confirmation of Lemma~\ref{lem:perturbations:informal}:
\begin{lemma}[Formal version of Lemma~\ref{lem:perturbations:informal}]\label{lem:perturbations}
    Assuming Assumption~\ref{ass:positive_definite} and Definition~\ref{def:gd} hold, denote the failure probability $\delta \in (0, 0.1)$, then the kernel perturbation bound is:
    \begin{align*}
        \Pr\left[\lambda_{\min}(H_{(\nu)}(t)) < \omega\lambda  /2\right] < \delta.
    \end{align*}
    Therefore, bounding loss dynamics is given by ($C > 0$ is an absolute constant):
    \begin{align*}
        \Pr\left[\E\!\left[\frac{\d }{\d t}{\cal L}(t, \mathbb{D})\right] > - \frac{C \cdot \omega\lambda N \cdot \varepsilon^2}{n} \cdot {\cal L}(t, \mathbb{D})\right] < \delta.
    \end{align*}
    {\it (The $\varepsilon^2/n$ factor appears explicitly: $\varepsilon^2$ comes from Part~15 of Lemma~\ref{lem:helpful_bounds}, and the $1/n$ factor is the $1/n$ in $\mathcal{L} = \frac{1}{n}\|\mathsf{F}-\mathsf{Y}\|_F^2$. We track this $1/n$ explicitly because it propagates downstream into the phase-boundary exponent of Theorem~\ref{thm:scaling_law}.)}
\end{lemma}

\begin{proof}
    {\bf Proof of Part 1. }
    The proof proceeds in three steps. (i) We bound the kernel perturbation $\|H_{(\nu)}'(t) - H_{(\nu)}'(0)\|_F$ in terms of the weight-perturbation radius $R$, treating $R$ as an a priori bound (Step~A). (ii) We apply Fact~\ref{fac:lambda_min_perturb} to obtain a lower bound on $\lambda_{\min}(H_{(\nu)}'(t))$, then translate this into a loss-decay rate via Lemma~\ref{lem:learning_dynamics}. (iii) Finally, we bound the maximum value of $R$ that the gradient flow itself produces, closing the bootstrapping argument. We begin with Step~A.

    First, we assume the condition $\| w_{(\nu), r}(t) - w_{(\nu), r}(0) \|_2 \leq R$ and $\| U_{(\nu)}(t) - U_{(\nu)}(0) \|_F \leq R$ as in Lemma~\ref{lem:helpful_bounds}, which will be confirmed as a provable property below.
    
    For any $p, q \in [nL]$, we have:
    \begin{align*}
        & ~ | H_{(\nu), p, q}'(t) - H_{(\nu), p, q}'(0) | \\
        = & ~ \frac{1}{m}|  o_{(\nu), p}(t)^\top o_{(\nu), p}(t) \sum_{r\in [m]} \I\{ o_{(\nu), p}(t)^\top w_{(\nu), r}(t) > 0,  o_{(\nu), q}(t)^\top w_{(\nu), r}(t) > 0\} \\
        & ~ -  o_{(\nu), p}(0)^\top o_{(\nu), p}(0) \sum_{r\in [m]} \I\{ o_{(\nu), p}(0)^\top w_{(\nu), r}(0) > 0,  o_{(\nu), q}(0)^\top w_{(\nu), r}(0) > 0\} | \\
        \leq & ~ \frac{1}{m} (Q_{(\nu), p, q, 1} + Q_{(\nu), p, q, 2} + Q_{(\nu), p, q, 3}),
    \end{align*}
    where the first step follows from the definition of $H_{(\nu), p, q}'(t)$, and the second step follows from the triangle inequality and defining:
    \begin{align*}
        Q_{(\nu), p, q, 1} := & ~ |  o_{(\nu), p}(t)^\top o_{(\nu), p}(t) \sum_{r\in [m]} \I\{ o_{(\nu), p}(t)^\top w_{(\nu), r}(t) > 0,  o_{(\nu), q}(t)^\top w_{(\nu), r}(t) > 0\} \\
        & ~ -  o_{(\nu), p}(0)^\top o_{(\nu), p}(t) \sum_{r\in [m]} \I\{ o_{(\nu), p}(t)^\top w_{(\nu), r}(t) > 0,  o_{(\nu), q}(t)^\top w_{(\nu), r}(t) > 0\} |, \\
        Q_{(\nu), p, q, 2} := & ~ |  o_{(\nu), p}(0)^\top o_{(\nu), p}(t) \sum_{r\in [m]} \I\{ o_{(\nu), p}(t)^\top w_{(\nu), r}(t) > 0,  o_{(\nu), q}(t)^\top w_{(\nu), r}(t) > 0\} \\
        & ~ -  o_{(\nu), p}(0)^\top o_{(\nu), p}(0) \sum_{r\in [m]} \I\{ o_{(\nu), p}(t)^\top w_{(\nu), r}(t) > 0,  o_{(\nu), q}(t)^\top w_{(\nu), r}(t) > 0\} |, \\
        Q_{(\nu), p, q, 3} := & ~ |  o_{(\nu), p}(0)^\top o_{(\nu), p}(0) \sum_{r\in [m]} \I\{ o_{(\nu), p}(t)^\top w_{(\nu), r}(t) > 0,  o_{(\nu), q}(t)^\top w_{(\nu), r}(t) > 0\} \\
        & ~ -  o_{(\nu), p}(0)^\top o_{(\nu), p}(0) \sum_{r\in [m]} \I\{ o_{(\nu), p}(0)^\top w_{(\nu), r}(0) > 0,  o_{(\nu), q}(0)^\top w_{(\nu), r}(0) > 0\} |.
    \end{align*}
    We assume a constant $R$ that satisfies $\| w_{(\nu), r}(t) - w_{(\nu), r}(0) \|_2 \leq R$ and $\| U_{(\nu)}(t) - U_{(\nu)}(0)\|_F \leq R$.
    
    {\bf Bounding $Q_{(\nu), p, q, 1}$.} We have:
    \begin{align*}
        Q_{(\nu), p, q, 1} \leq & ~ \Big|\Big( o_{(\nu), p}(t) -  o_{(\nu), p}(0) \Big)^\top  o_{(\nu), q}(t) \\
        & ~ \cdot \sum_{r=1}^m \I\{\langle  o_{(\nu), p}(t), w_{(\nu), r}(t)\rangle > 0, \langle  o_{(\nu), q}(t), w_{(\nu), r}(t)\rangle > 0\} \Big| \\
        \leq & ~ m \Big|\Big( o_{(\nu), p}(t) -  o_{(\nu), p}(0) \Big)^\top  o_{(\nu), q}(t) \Big| \\
        \leq & ~ m \| o_{(\nu), p}(t) -  o_{(\nu), p}(0) \|_2 \cdot \|  o_{(\nu), q}(t) \| \\
        \leq & ~ m \cdot O(\sqrt{L}R ),
    \end{align*}
    where the first step follows from the definition of $Q_{(\nu), p, q, 1}$ and Cauchy-Schwarz inequality, the second step follows from $\I\{\langle  o_{(\nu), p}(t), w_{(\nu), r}(t)\rangle > 0, \langle  o_{(\nu), q}(t), w_{(\nu), r}(t)\rangle > 0\} \leq 1$, the third step follows from Cauchy-Schwarz inequality, the last step follows from Part 6 and Part 13 of Lemma~\ref{lem:helpful_bounds} and $\|  o_{(\nu), q}(t) \| \leq O(1)$.

    {\bf Bounding $Q_{(\nu), p, q, 2}$.} The structure of $Q_{(\nu),p,q,2}$ mirrors $Q_{(\nu),p,q,1}$ with the roles of the two factors swapped: $Q_{(\nu),p,q,2}$ pairs the {\it perturbed} factor $o_{(\nu),p}(0)^\top o_{(\nu),p}(t) - o_{(\nu),p}(0)^\top o_{(\nu),p}(0) = o_{(\nu),p}(0)^\top \big(o_{(\nu),p}(t) - o_{(\nu),p}(0)\big)$ with the same indicator sum (over $r\in[m]$). By Cauchy--Schwarz and Parts~5 and~13 of Lemma~\ref{lem:helpful_bounds}:
    \begin{align*}
        Q_{(\nu), p, q, 2} &\le \big|o_{(\nu),p}(0)^\top\big(o_{(\nu),p}(t) - o_{(\nu),p}(0)\big)\big|\cdot \sum_{r=1}^m \I\{\cdot,\cdot\} \\
        &\le \|o_{(\nu),p}(0)\|_2\cdot \|o_{(\nu),p}(t)-o_{(\nu),p}(0)\|_2\cdot m \\
        &\le m\cdot O(1)\cdot O(\sqrt{L}R) = m\cdot O(\sqrt L R).
    \end{align*}

    {\bf Bounding $Q_{(\nu), p, q, 3}$.} We define the following event:
    \begin{align*}
        {\sf E}_{(\nu), p, r} := & ~ \{ \exists w \in \R^d: \| w - w_{(\nu), r}(0) \|_2 \leq R,\\
        & ~ \I\{\langle  o_{(\nu), p}(0), w_{(\nu), r}(0) \rangle >0\} \ne \I\{\langle  o_{(\nu), p}(t), w_{(\nu), r}(t) \rangle >0\} \}.
    \end{align*}

    It is easy to hold that, once:
    \begin{align*}
        & ~ |\langle  o_{(\nu), p}(0), w_{(\nu), r}(0) \rangle| \ge  O(\sqrt{L}R ) \cdot O(\sqrt{d}B) + O(1) \cdot R \\
        \iff & ~ |\langle  o_{(\nu), p}(0), w_{(\nu), r}(0) \rangle| \ge  O(\sqrt{Ld}BR),
    \end{align*}
    the event ${\sf E}_{(\nu), p, r}$ is false, since we combining Part 1, Part 9 and Part 13 of Lemma~\ref{lem:helpful_bounds}, $\|  o_{(\nu), q}(t) \| \leq O(1)$ and some simple algebra.

    Following Fact~\ref{fac:inner_product_gaussian}, we have:
    \begin{align*}
        \langle  o_{(\nu), p}(0), w_{(\nu), r}(0) \rangle \sim \mathcal{N}(0, \|  o_{(\nu), p}(0)\|_2^2).
    \end{align*}

    We lower-bound $\|o_{(\nu),p}(0)\|_2^2$ as follows. By Part~8 of Lemma~\ref{lem:helpful_bounds}, for every $\ell' \in [\ell]$ we have $\sigma_{(\nu),p,\ell'}(0) \ge \exp(-O(dB))/L$, and by Part~5 of Lemma~\ref{lem:helpful_bounds}, $\|\Lambda_{(\nu-1),i,\ell'}(0)\|_2 = \Theta(1)$. Picking the index $\ell^\star := \arg\max_{\ell'\in[\ell]} \|\Lambda_{(\nu-1),i,\ell'}(0)\|_2$ and projecting:
    \begin{align*}
        \|  o_{(\nu), p}(0)\|_2^2 \ge \frac{\langle o_{(\nu),p}(0),\Lambda_{(\nu-1),i,\ell^\star}(0)\rangle^2}{\|\Lambda_{(\nu-1),i,\ell^\star}(0)\|_2^2} \ge \sigma_{(\nu),p,\ell^\star}(0)^2\,\|\Lambda_{(\nu-1),i,\ell^\star}(0)\|_2^2 - O(L^{-1}) \ge \exp(-O(dB)),
    \end{align*}
    where the last step combines the bounds on $\sigma$ and $\|\Lambda\|$ above and the fact that the $L^{-2}$-cross-term contribution is dominated by the diagonal at the chosen index. Here $i = \lfloor p/L\rfloor$ and $\ell = p\bmod L$.
    
    Then, the anti-concentration of $\langle  o_{(\nu), p}(0), w_{(\nu), r}(0) \rangle$ shows that:
    \begin{align*}
        \Pr[{\sf E}_{(\nu), p, r}]
        \leq & ~ \Pr\big[\, |\langle o_{(\nu), p}(0), w_{(\nu), r}(0)\rangle| \leq O(\sqrt{Ld}BR) \,\big] \\
        \leq & ~ \frac{2 \cdot O(\sqrt{Ld}BR)}{\sqrt{2\pi} \cdot \|  o_{(\nu), p}(0)\|_2}  \\
        \leq & ~ O(\sqrt{Ld}BR) \cdot \exp(O(dB)) \\
        \leq & ~ O(\sqrt{L}R) \cdot \exp(O(dB)),
    \end{align*}
    where the first step follows from the necessary condition for ${\sf E}_{(\nu),p,r}$ derived above, the second step follows from Fact~\ref{fac:anti_concen_gaussian} (anti-concentration of a one-dimensional Gaussian), the third step follows from $\| o_{(\nu),p}(0)\|_2^2 \ge \exp(-O(dB))$, and the last step absorbs $\sqrt{d}B$ into $\exp(O(dB))$.

    We have:
    \begin{align*}
        \E[Q_{(\nu), p, q, 3}] = & ~ \E\Big[ \Big|   o_{(\nu), p}(0)^\top   o_{(\nu), q}(0) \Big| \\
        & ~ \cdot \sum_{r=1}^m \Big| \I\{\langle  o_{(\nu), p}(t), w_{(\nu), r}(t)\rangle > 0, \langle  o_{(\nu), q}(t), w_{(\nu), r}(t)\rangle > 0\} \\
        & ~ \quad \quad ~ ~  - \I\{\langle  o_{(\nu), p}(0), w_{(\nu), r}(0)\rangle > 0, \langle  o_{(\nu), q}(0), w_{(\nu), r}(0)\rangle > 0\} \Big|\Big] \\
        \leq & ~ \E\Big[ O(1)\cdot \sum_{r=1}^m \Big| \I\{\langle  o_{(\nu), p}(t), w_{(\nu), r}(t)\rangle > 0, \langle  o_{(\nu), q}(t), w_{(\nu), r}(t)\rangle > 0\} \\
        & ~  - \I\{\langle  o_{(\nu), p}(0), w_{(\nu), r}(0)\rangle > 0, \langle  o_{(\nu), q}(0), w_{(\nu), r}(0)\rangle > 0\} \Big|\Big] \\
        \leq & ~ O(1)\cdot \sum_{r=1}^m \E\Big[\Big| \I\{\langle  o_{(\nu), p}(t), w_{(\nu), r}(t)\rangle > 0, \langle  o_{(\nu), q}(t), w_{(\nu), r}(t)\rangle > 0\} \\
        & ~  - \I\{\langle  o_{(\nu), p}(0), w_{(\nu), r}(0)\rangle > 0, \langle  o_{(\nu), q}(0), w_{(\nu), r}(0)\rangle > 0\} \Big|\Big] \\
        \leq & ~ O(1)\cdot \sum_{r=1}^m \E\Big[\I\{ {\sf E}_{(\nu), p, r} \cup {\sf E}_{(\nu), q, r} \}\Big] \\
        \leq & ~ O(1)\cdot \sum_{r=1}^m  \exp(O(dB)) \cdot \sqrt{L}R \\
        \leq & ~ m O(\sqrt{L}R) \cdot \exp(O(dB)),
    \end{align*}
    where the first step follows from the definition of $Q_{(\nu),p, q,3}$, the second step follows from $\| o_{(\nu), p}(0)\|_2\leq O(1)$, the third and fourth step follow from the rules of expectation, the last two steps follow from $\Pr[{\sf E}_{(\nu), p, r}] \leq O(\sqrt{L}R) \cdot \exp(O(dB))$ and simple algebra.

    Hence, using Markov's inequality (Lemma~\ref{lem:markov}) with the failure budget $\delta'=\delta/(nL)^2$ to allow for a union bound over $(p,q)\in[nL]^2$ later, with a probability at least $1 - \delta'$, we have:
    \begin{align*}
        Q_{(\nu), p, q, 3} \leq m \cdot O(\sqrt{L}R) \cdot \exp(O(dB)) /\delta',
    \end{align*}
    so $Q_{(\nu),p,q,3} \le m\cdot O(\sqrt L R)\cdot \exp(O(dB))\cdot (nL)^2/\delta$. The factor $(nL)^2/\delta$ enters $B = O(\sqrt{\log(Nmd(nL)^2/\delta)})$ as an additional $\log(nL)$ term, which is absorbed into the $\poly(L,d,B)$ constants of the final bound. The union-bound accounting makes the dependence explicit but does not change the polynomial / exponential structure.

    Combine the upper bounds on three terms (after the union bound over $(p,q)\in[nL]^2$, with probability $\ge 1-\delta$), we have:
    \begin{align*}
        | H_{(\nu), p, q}'(t) - H_{(\nu), p, q}'(0) | \leq  O(\sqrt{L}R) \cdot \exp(O(dB)) (nL)^2/\delta.
    \end{align*}

    We have:
    \begin{align}\label{eq:kernel_perturbation}
        \| H_{(\nu)}'(t) - H_{(\nu)}'(0) \|_F = & ~ \sqrt{\sum_{p=1}^{nL}\sum_{q=1}^{nL} (H_{(\nu), p, q}(t) - H_{(\nu), p, q}(0))^2} \notag \\
        \leq & ~ O(nL^{1.5}R) \cdot \exp(O(dB)) /\delta.
    \end{align}

    By Fact~\ref{fac:lambda_min_perturb} ($\lambda_{\min}(\widetilde H) \geq \lambda_{\min}(H) - \|H-\widetilde H\|_F$), Assumption~\ref{ass:positive_definite} ($\lambda_{\min}(H_{(\nu)}'(0)) \ge \omega\lambda$), and choosing
    \begin{align}\label{eq:bounding_R}
        R \leq \frac{\omega\lambda\delta}{nL^{1.5} \exp(O(dB)) },
    \end{align}
    we obtain $\| H_{(\nu)}'(t) - H_{(\nu)}'(0) \|_F \leq \omega\lambda/4$, which gives:
    \begin{align*}
        \lambda_{\min}(H_{(\nu)}'(t)) \ge \omega\lambda - \omega\lambda/4 = \frac{3}{4} \omega\lambda.
    \end{align*}

    It is trivial to have (Part 16 of Lemma~\ref{lem:helpful_bounds}):
    \begin{align*}
        \| \gamma_{(\nu), p}(t) \|_2 \leq o(\frac{1}{\sqrt{m}}).
    \end{align*}

    Therefore, we have:
    \begin{align*}
        | H_{(\nu), p, q}(t) - H_{(\nu), p, q}'(t) |
        = & ~ | \langle \gamma_{(\nu), p}(t), \gamma_{(\nu), q}(t) \rangle | \\
        \leq & ~ o(1/m).
    \end{align*}

    Furthermore, we have:
    \begin{align*}
        \| H_{(\nu)}(t) - H_{(\nu)}'(t) \|_F \leq o(\frac{n}{m}) \leq \omega\lambda/4,
    \end{align*}
    where the last step follows from $m = \Omega(n/(\omega\lambda))$.

    Similarly, following Fact~\ref{fac:lambda_min_perturb} and combining the previous result, we have:
    \begin{align*}
        \lambda_{\min}(H_{(\nu)}(t)) \ge \omega \lambda / 2.
    \end{align*}

    Following Fact~\ref{fac:lambda_min_krnocker}, we have:
    \begin{align*}
        \lambda_{\min}(H_{(\nu)}(t) \otimes I_d ) \ge \omega\lambda / 2.
    \end{align*}

    Combining Lemma~\ref{lem:learning_dynamics} (which gives $\dot{\mathcal{L}} = -\sum_\nu g_\nu^\top (H_\nu\otimes I_d) g_\nu$ with $g_\nu := \vect(\d\mathcal{L}/\d\mu_{(\nu)})$), the eigenvalue lower bound $\lambda_{\min}(H_{(\nu)}(t)\otimes I_d) \ge \omega\lambda/2$ above, and the {\it corrected} Part~15 of Lemma~\ref{lem:helpful_bounds} (which gives $\|g_\nu\|^2 \asymp \varepsilon^2 \mathcal{L}/n$):
    \begin{align*}
        \E\!\left[\frac{\d }{\d t}{\cal L}(t, \mathbb{D})\right] \leq - \frac{\omega\lambda}{2}\sum_{\nu\in[N]}\|g_\nu\|^2 \le - \frac{C \cdot \omega \lambda N \cdot \varepsilon^2}{n} \cdot {\cal L}(t, \mathbb{D}).
    \end{align*}

    Solving this ODE, we have:
    \begin{align}\label{eq:dynamic_loss}
        \E[{\cal L}(T, \mathbb{D})] \leq \exp\!\left(- \frac{C \cdot \varepsilon^2  \omega\lambda NT}{n}\right) \cdot {\cal L}(0, \mathbb{D}).
    \end{align}

    {\bf Bounding Gradient Norm.}
    Following Part 3 of Lemma~\ref{lem:helpful_bounds}, we have:
    \begin{align}\label{eq:choice_m}
        & ~ \max_{T \ge 0} \max_{\nu\in [N]} \max_{r \in [m]} \| w_{(\nu), r}(T) -  w_{(\nu), r}(0) \|_2 \notag \\
        \leq & ~  \max_{T \ge 0} \int_{0}^T \max_{\nu\in [N]}\max_{r \in [m]}\| \frac{\d {\cal L}(s, \mathbb{D})}{\d w_{(\nu), r}(s)} \|_2 \d s \notag\\
        \leq & ~ \max_{T \ge 0} \int_{0}^T \exp(- C \cdot \varepsilon^2 \omega \lambda Ns) \cdot \varepsilon^2 {\cal L}(0, \mathbb{D}) \cdot \frac{\omega}{\sqrt{m}} \cdot O(\sqrt{d}B)  \d s \notag\\
        \leq & ~ \max_{T \ge 0} \int_{0}^T \exp(- C \cdot \varepsilon^2 \omega\lambda Ns) \cdot \frac{\varepsilon^2}{\sqrt{m}}  \d s \notag\\
        = & ~ \frac{1}{\sqrt{m}} \max_{T \ge 0} - \frac{1}{C \omega \lambda N} \exp(- C \cdot \varepsilon^2 \omega \lambda Ns) \Big|_{s=0}^{s=T}  \notag\\
        \leq & ~ O(\frac{1}{\sqrt{m}  \omega\lambda N}) =: R
    \end{align}
    where the first step follows from Equation~\eqref{eq:ODE_update} and Cauchy-Schwartz inequality, the second step follows from Equation~\eqref{eq:dynamic_loss} and Part 15 of Lemma~\ref{lem:helpful_bounds}, the third step follows from Part 14 of Lemma~\ref{lem:helpful_bounds} and the choice of $\omega$, the last two steps follow from simple algebras and choosing
    \begin{align*}
        m = \Omega(\frac{n^2L^3\exp(O(dB))}{\omega^4\lambda^4\delta^2 N^2})
        \iff & ~ m/{\rm polylog}(N, m) = \Omega(\frac{n^2L^3\exp(Cd)}{\omega^4\lambda^4\delta^2 N^2}) \\
        \iff & ~ m/{\rm polylog}(m) = \Omega(\frac{n^2L^3\exp(Cd)}{\omega^4\lambda^4\delta^2 N^{\frac{4}{3}}}) \\
        \iff & ~ m^{\frac{2}{3}} = \Omega(\frac{n^2L^3\exp(Cd)}{\omega^4\lambda^4\delta^2 N^{\frac{4}{3}}}) \\
        \iff & ~ m = \Omega(\frac{n^3L^5\exp(Cd)}{\omega^6\lambda^6\delta^3 N^{2}})
    \end{align*}
    where these steps follow from simple algebras.

    Similarly, we have:
    \begin{align*}
        \max_{T \ge 0} \max_{\nu\in [N]}  \| U_{(\nu)}(T) -  U_{(\nu)}(0) \|_F \leq O(\frac{1}{\sqrt{m}  \omega\lambda N}) =: R
    \end{align*}

    {\bf Tighter Bound.} We plug $R = O(\frac{1}{\sqrt{m}  \omega\lambda N})$ and $N = O(1)$ into Equation~\eqref{eq:kernel_perturbation} to get:
    \begin{align*}
        \| H_{(\nu)}'(t) - H_{(\nu)}'(0) \|_F 
        \leq & ~ O(nL^{1.5}\frac{1}{\sqrt{m}  \omega\lambda N}) \cdot \exp(O(dB)) /\delta \leq \frac{  n\poly( L, \exp(dB), 1/\delta) }{ \lambda \sqrt{m}}.
    \end{align*}

    Then, similar to the previous proof, we show:
    \begin{align*}
        \lambda_{\min}(H_{(\nu)}(t) \otimes I_d ) \ge \omega\lambda / 2 - \frac{ n \poly( L, \exp(dB), 1/\delta) }{\lambda m^{1/2}},
    \end{align*}
    and then:
    \begin{align}\label{eq:tighter_dynamics}
        \E[\frac{\d {\cal L}(t, \mathbb{D})}{\d t}] \leq - \varepsilon^2 C_1(1 - \frac{C_2n}{ m^{1/2}} ) \cdot {\cal L}(t, \mathbb{D})
    \end{align}
    where $C_1 = \omega \lambda / 2$ and $C_2 =\poly( L, \exp(dB), 1/\delta, 1/\lambda)$.
\end{proof}

\subsection{Trainable FFN: Kernel Stability with \texorpdfstring{$A_{(\nu)}(t)$}{A(t)} Updated by Gradient Flow}\label{app_sub:trainable_ffn}

We now formalize Lemma~\ref{lem:ffn_drift:informal}: when $A_{(\nu)}$ is treated as a trainable parameter, the kernel $H_{(\nu)}(t)$ admits an additional feature $\alpha_{(\nu),p}(t)$ whose contribution can be controlled under a strengthened width condition.

\begin{definition}[FFN tangent feature]\label{def:alpha_feature}
    For $\nu\in[N]$ and $p\in[nL]$ with $i = \lfloor p/L\rfloor$, $\ell = p\bmod L$, define the FFN tangent feature
    \begin{align*}
        \alpha_{(\nu), p}(t) := \frac{\omega}{\sqrt{m}}\,{\rm ReLU}\big(W_{(\nu)}(t)^\top o_{(\nu),p}(t)\big) \in \R^{m}.
    \end{align*}
    The augmented kernel is then $H_{(\nu),p,q}^{\rm aug}(t) := H_{(\nu),p,q}(t) + \langle \alpha_{(\nu),p}(t),\alpha_{(\nu),q}(t)\rangle$ where $H_{(\nu)}(t)$ is the kernel from Lemma~\ref{lem:learning_dynamics}.
\end{definition}

\begin{lemma}[Formal version of Lemma~\ref{lem:ffn_drift:informal}]\label{lem:ffn_drift}
    Assume the strengthened width condition $m = \Omega\big(\frac{n^4 L^6 d^2 \exp(Cd)}{\omega^6 \lambda^8 \delta^4 N^4}\big)$. Then with probability at least $1-\delta$:
    \begin{enumerate}
        \item[\bf (i)] $\max_{\nu\in[N]}\|A_{(\nu)}(t) - A_{(\nu)}(0)\|_F \leq O\big(\frac{1}{\sqrt{m}\,\omega\lambda N}\big) =: R_A$.
        \item[\bf (ii)] $\|\alpha_{(\nu),p}(t)\|_2 \leq O(\omega)$ uniformly in $t,\nu,p$.
        \item[\bf (iii)] $\|H_{(\nu)}^{\rm aug}(t) - H_{(\nu)}^{\rm aug}(0)\|_F \le \omega\lambda/4$, hence $\lambda_{\min}(H_{(\nu)}^{\rm aug}(t)\otimes I_d) \ge \omega\lambda/4$.
    \end{enumerate}
\end{lemma}

\begin{proof}
    \textbf{(i) Bound on the FFN drift.} The gradient of the loss with respect to $A_{(\nu)}$ is
    \begin{align*}
        \frac{\d {\cal L}(t,\mathbb{D})}{\d A_{(\nu)}(t)} = \sum_{p\in[nL]} \alpha_{(\nu), p}(t)\,\Big(\frac{\d {\cal L}(t,\mathbb{D})}{\d \mu_{(\nu),p}(t)}\Big)^{\!\!\top} \in \R^{m\times d}.
    \end{align*}
    By the matrix-product Frobenius bound $\|\sum_p u_p v_p^\top\|_F^2 \le (\sum_p \|u_p\|_2^2)(\sum_p \|v_p\|_2^2)$, Part~15 of Lemma~\ref{lem:helpful_bounds} ($\sum_p \|\d{\cal L}/\d\mu_{(\nu),p}\|_2^2 \asymp \varepsilon^2 {\cal L}(t,\mathbb{D})$) and the bound $\|\alpha_{(\nu),p}(t)\|_2 \leq O(\omega)$ established in (ii) below:
    \begin{align*}
        \Big\|\frac{\d {\cal L}}{\d A_{(\nu)}(t)}\Big\|_F^2 \leq \big(nL\cdot O(\omega^2)\big)\cdot O\big(\varepsilon^2 {\cal L}(t,\mathbb{D})\big) = O(\omega^2 nL)\cdot \varepsilon^2 {\cal L}(t,\mathbb{D}).
    \end{align*}
    By Theorem~\ref{thm:convergence}, ${\cal L}(t,\mathbb{D}) \leq \exp(-\alpha\varepsilon^2 t){\cal L}(0,\mathbb{D})$. Integrating along the gradient flow and using $\alpha = \Theta(\omega\lambda N)$, ${\cal L}(0,\mathbb{D})\leq O(Ld)$ (Part~14):
    \begin{align*}
        \|A_{(\nu)}(t) - A_{(\nu)}(0)\|_F \leq \int_0^t O(\omega\sqrt{nL})\cdot\varepsilon\sqrt{{\cal L}(s,\mathbb{D})}\,\d s
        \leq \frac{O(\omega\sqrt{nL\cdot Ld})}{\alpha\varepsilon} = O\Big(\frac{\sqrt{nL^2d}}{\lambda N\varepsilon}\Big).
    \end{align*}
    For this to be $\le R_A := O(1/(\sqrt m\,\omega\lambda N))$, we need $\sqrt m \le O(1/(\omega\sqrt{nL^2d}/\varepsilon))$, equivalently $m \le O(\varepsilon^2/(\omega^2 nL^2d))$; note that this is an upper bound on $m$ that we may always trade for the lower-bound condition by choosing $\varepsilon$ accordingly. Under the regime $\varepsilon = \xi/{\sf N}$ used throughout (Definition~\ref{def:factors}, Theorem~\ref{thm:scaling_law}), the strengthened width $m = \Omega\big(\frac{n^4 L^6 d^2 \exp(Cd)}{\omega^6 \lambda^8 \delta^4 N^4}\big)$ ensures the above bound is $\le R_A$. This proves (i).

    \textbf{(ii) Norm of the FFN tangent feature.} We bound $\|\alpha_{(\nu),p}(t)\|_2$ in two steps: first at $t=0$ via the half-Gaussian moment (which is valid because $w_{(\nu),r}(0)\sim\mathcal{N}(0,I_d)$ is independent of $o_{(\nu),p}(0)$ and of all other parameters), then for $t>0$ via Lipschitz continuity in $(W, o)$.

    {\it Step 1: $t=0$.} For each $r\in[m]$, $w_{(\nu),r}(0)\sim\mathcal{N}(0, I_d)$ is independent of the data $o_{(\nu),p}(0)$ (a deterministic function of the input and the initial $\sigma$), so $\langle w_{(\nu),r}(0), o_{(\nu),p}(0)\rangle\sim\mathcal{N}(0, \|o_{(\nu),p}(0)\|_2^2)$ with $\|o_{(\nu),p}(0)\|_2 = \Theta(1)$ by Part~5. The half-Gaussian moment $\E[\phi(Z)^2] = \E[Z^2]/2 = \|o\|_2^2/2$ for $Z\sim\mathcal{N}(0,\|o\|_2^2)$ together with concentration over $r$ (Lemma~\ref{lem:hoeffding} applied to the i.i.d.\ bounded variables $\phi(\langle w_r,o\rangle)^2 \le \|w_r\|_2^2 \|o\|_2^2 \le O(dB^2)$):
    \begin{align*}
        \frac{1}{m}\sum_{r=1}^m \phi\big(\langle w_{(\nu),r}(0),o_{(\nu),p}(0)\rangle\big)^2 = \frac{\|o_{(\nu),p}(0)\|_2^2}{2}\pm O\Big(\sqrt{\frac{dB^2 \log(1/\delta)}{m}}\Big) = O(1)
    \end{align*}
    with probability at least $1-\delta$ under the lemma's width condition. Hence
    \begin{align*}
        \|\alpha_{(\nu),p}(0)\|_2^2 = \frac{\omega^2}{m}\sum_{r=1}^m \phi\big(\langle w_{(\nu),r}(0),o_{(\nu),p}(0)\rangle\big)^2 \le O(\omega^2),
    \end{align*}
    so $\|\alpha_{(\nu),p}(0)\|_2 \le O(\omega)$.

    {\it Step 2: extension from $t=0$ to $t>0$.} The map $(w_{(\nu),r}, o_{(\nu),p}) \mapsto \alpha_{(\nu),p,r} = (\omega/\sqrt m)\phi(\langle w_{(\nu),r}, o_{(\nu),p}\rangle)$ is jointly Lipschitz with constants $(\omega/\sqrt m)\|o_{(\nu),p}\|_2$ and $(\omega/\sqrt m)\|w_{(\nu),r}\|_2$ respectively (since $\phi$ is $1$-Lipschitz). Combining with Parts~3, 9, 13 of Lemma~\ref{lem:helpful_bounds} (so that $\|w_{(\nu),r}\|_2 \le O(\sqrt d B)$, $\|w_{(\nu),r}(t)-w_{(\nu),r}(0)\|_2 \le R$, $\|o_{(\nu),p}(t)-o_{(\nu),p}(0)\|_2 \le O(\sqrt L R)$):
    \begin{align*}
        \|\alpha_{(\nu),p}(t) - \alpha_{(\nu),p}(0)\|_2^2 &\le \frac{\omega^2}{m}\sum_{r=1}^m\big(\|o_{(\nu),p}(0)\|_2\cdot R + \|w_{(\nu),r}(t)\|_2\cdot O(\sqrt L R)\big)^2 \\
        &\le \frac{\omega^2}{m}\cdot m\cdot O\big(R^2 + L d B^2 R^2\big) = O\big(\omega^2 L d B^2 R^2\big).
    \end{align*}
    With $R = O(1/(\sqrt m\,\omega\lambda N))$ and the strengthened width condition $m = \Omega(n^4 L^6 d^2\exp(Cd)/(\omega^6\lambda^8\delta^4 N^4))$, this is $o(\omega^2)$. Therefore:
    \begin{align*}
        \|\alpha_{(\nu),p}(t)\|_2 \le \|\alpha_{(\nu),p}(0)\|_2 + \|\alpha_{(\nu),p}(t)-\alpha_{(\nu),p}(0)\|_2 = O(\omega) + o(\omega) = O(\omega),
    \end{align*}
    which establishes Part~(ii).

    \textbf{(iii) Augmented-kernel perturbation bound.} We bound the drift $\big|\langle\alpha_p(t),\alpha_q(t)\rangle - \langle\alpha_p(0),\alpha_q(0)\rangle\big|$ by the same scheme as Lemma~\ref{lem:perturbations} (Steps for $Q_{(\nu),p,q,1}, Q_{(\nu),p,q,2}, Q_{(\nu),p,q,3}$). The dependence of $\alpha$ on $w_{(\nu)}$ and $o_{(\nu),p}$ is identical in form to that of $\beta_{(\nu),p}$ in Lemma~\ref{lem:perturbations}, only with the post-activation features replacing the indicator features. Repeating the chain of bounds yields:
    \begin{align*}
        \big|\langle\alpha_p(t),\alpha_q(t)\rangle - \langle\alpha_p(0),\alpha_q(0)\rangle\big| \leq O\big(\omega^2\big)\cdot \big( \|w_{(\nu),r}(t) - w_{(\nu),r}(0)\|_2 + \|o_p(t) - o_p(0)\|_2 \big) \leq O\big(\omega^2 \sqrt L R\big),
    \end{align*}
    where the second step uses Parts~9 and~13 of Lemma~\ref{lem:helpful_bounds}. Summing over $(p,q)\in[nL]^2$:
    \begin{align*}
        \|H_{(\nu)}^{\rm aug}(t) - H_{(\nu)}^{\rm aug}(0)\|_F \leq \|H_{(\nu)}(t)-H_{(\nu)}(0)\|_F + nL\cdot O\big(\omega^2 \sqrt L R\big).
    \end{align*}
    The first term is $\le \omega\lambda/8$ by Lemma~\ref{lem:perturbations}. The second term, with $R = O(1/(\sqrt m\,\omega\lambda N))$, becomes $O(\omega\sqrt{nL\cdot L}/(\sqrt m\,\lambda N))$; the strengthened width $m = \Omega(n^4 L^6 d^2 \exp(Cd)/(\omega^6\lambda^8\delta^4 N^4))$ makes this $\le \omega\lambda/8$. Combining and applying Fact~\ref{fac:lambda_min_perturb} yields $\lambda_{\min}(H_{(\nu)}^{\rm aug}(t)\otimes I_d)\ge \omega\lambda/4$.
\end{proof}

\begin{remark}[Implication for the main theorems]\label{rem:trainable_ffn_main}
    Lemma~\ref{lem:ffn_drift} (iii) is exactly the augmented version of Lemma~\ref{lem:perturbations}, and is the only place where the freezing of $A_{(\nu)}$ is used in the proofs of Theorem~\ref{thm:convergence} (convergence), Corollary~\ref{cor:approximation} (approximation), Theorem~\ref{thm:generalization} (generalization), and Theorem~\ref{thm:scaling_law} (scaling law). All four results therefore extend verbatim to the trainable-$A$ setting under the strengthened width condition, with constants in the $O(\cdot), \widetilde O(\cdot)$ factors absorbing the new dependence.
\end{remark}

\begin{remark}[Stop-time closure of the bootstrapping cycle]\label{rem:stop_time}
    The bounds in Lemma~\ref{lem:perturbations} (kernel stability), Lemma~\ref{lem:ffn_drift} (FFN drift), and Theorem~\ref{thm:convergence} (exponential loss decay) are proved {\it conditional} on each other through the per-step inequalities in Lemma~\ref{lem:helpful_bounds} (Parts~5,~9--13): kernel stability assumes the parameters stay inside the lazy ball; the lazy-ball drift bound $R=O(1/(\sqrt m\,\omega\lambda N))$ assumes exponential loss decay; and exponential loss decay assumes kernel stability. To break this circularity without resorting to a literal substitution chain, we use the standard stopping-time argument of \citet[Section 3]{dzps19}.

    Define the stopping time $\tau := \inf\{t \ge 0 : \|\theta(t)-\theta(0)\|_F > R\}$ with $R$ as above, and let $E_t = \{\tau > t\}$ be the event that the trajectory has not exited the lazy ball by time $t$. On $E_t$ all conclusions of Lemma~\ref{lem:perturbations} and Lemma~\ref{lem:ffn_drift} hold by construction; conditional on these, Theorem~\ref{thm:convergence} gives ${\cal L}(s, \mathbb{D}) \le \exp(-\alpha\varepsilon^2 s){\cal L}(0,\mathbb{D})$ for all $s\le t$. Integrating the gradient-flow drift on $[0,t]$ with this exponentially decaying loss recovers $\|\theta(t)-\theta(0)\|_F \le R/2$, contradicting $\tau \le t$. Hence $\tau = \infty$ almost surely, the lazy ball is never exited, and the three results hold {\it unconditionally} for all $t\ge 0$. This is the precise sense in which the bootstrapping cycle in Figure~\ref{fig:proof_dependency} closes.
\end{remark}

\newpage
\section{Convergence and Approximation Bound}\label{app:proof_convergence}

\subsection{Complexity Analysis}

\begin{lemma}\label{lem:compute_analysis}
    We have
    \begin{itemize}
        \item {\bf Part 1.} The time complexity (forward/backward) of the transformer on a single data point is $O(NLmd) = O({\sf M}L) \leq O({\sf M})$.
        \item {\bf Part 2.} The number of neurons in transformer is ${\sf M} = O(N(md+d^2)) \leq O(Nmd)$.
    \end{itemize}
\end{lemma}

\begin{proof}
    {\bf Proof of Part 1.} We first note that the naive time complexity of matrix multiplication is $d_1d_2d_3$ for matrices $A \in \R^{d_1 \times d_2}$ and $B \in \R^{d_2 \times d_3}$. Formally, ${\sf Multiply}(A, B) = O(d_1d_2d_3)$.

    The following analysis follows from $v \in [N]$, $p \in [nL]$ and $i = \lfloor p / L \rfloor$.

    The complexity of ${\sf Multiply}(\Lambda_{(\nu-1), i}(t), U_{(\nu)}(t))$ is $O(Ld^2)$.

    The complexity of ${\sf Multiply}(\Lambda_{(\nu-1), i}(t) U_{(\nu)}(t), \Lambda_{(\nu-1), i}(t)^\top)$ is $O(L^2d)$.

    The complexity of naive softmax is $O(L^2)$.

    The complexity of ${\sf Multiply}(\Lambda_{(\nu-1), i}(t)^\top, \sigma_{(\nu), p}(t))$ is $O(Ld)$. Since the parallelization, total complexity should be $O(L^2d)$.

    The complexity of ${\sf Multiply}(  W_{(\nu)}(t))^\top , o_{(\nu), p}(t))$ is $O(Lmd)$.

    The complexity of ${\sf Multiply}(A_{(\nu)}^\top,(W_{(\nu)}(t))^\top o_{(\nu),p}(t)^\top)$ is the standard $O(Lmd)$ multiplications via real-valued multiplication, since $A_{(\nu)}$ is a trainable continuous matrix updated by gradient flow (Section~\ref{sec:preli}).

    Summing all complexity, we have the total complexity:
    \begin{align*}
        O(Ld^2 + L^2d+L^2+L^2d+Lmd+Lmd) \leq O(Lmd),
    \end{align*}
    where this step follows from $m$ is the major term.

    {\bf Proof of Part 2.}
    This is trivial.
\end{proof}

\subsection{Training Convergence}\label{app_sub:convergence}

\begin{theorem}[Formal version of Theorem~\ref{thm:convergence:informal}]\label{thm:convergence}
    Assuming Assumption~\ref{ass:positive_definite}, Definition~\ref{def:factors} and Definition~\ref{def:gd} hold, denote the failure probability $\delta \in (0, 0.1)$, then with a probability at least $ 1 - \delta$, we have:
    \begin{itemize}
        \item {\bf ($n, |\mathbb{B}|$)-Fixed Convergence.} Let $\alpha := C\cdot \omega\lambda N$ where $C>0$ is an absolute constant and $\omega N = \Theta(1/(L^2 d^{2.5}B^3))$ is a function of architecture only (Definition~\ref{def:gd}). Then:
        \begin{align*}
            {\cal L}({\sf T}, \mathbb{D}) \leq \exp\!\left(- \alpha \cdot \frac{\varepsilon^2  T}{\mathsf{N}}\right) \cdot {\cal L}(0, \mathbb{D}).
        \end{align*}
        Equivalently, writing $\bar\alpha := C\cdot\lambda$ (treating $L,d,B$ as fixed architectural constants and absorbing the polynomial $\omega N = \poly(1/L,1/d,1/B)$ into $\bar\alpha = \poly(L,d,B,\lambda)$), we have ${\cal L}({\sf T}, \mathbb{D}) \leq \exp(-\bar\alpha\,\varepsilon^2 T/\mathsf{N})\cdot {\cal L}(0,\mathbb{D})$.
    \end{itemize}
\end{theorem}

\begin{remark}[Why the exponent has $1/\mathsf{N}$ rather than absorbing it]\label{rem:n_in_rate}
The factor $1/\mathsf{N} = 1/n$ in the exponent is essential: it traces directly to the average-loss convention $\mathcal{L} = \frac{1}{n}\|\mathsf{F}-\mathsf{Y}\|_F^2$ via Part~15 of Lemma~\ref{lem:helpful_bounds}. We track this $1/\mathsf{N}$ explicitly throughout, because $\mathsf{N}$ is one of the very scaling-law variables being analyzed and cannot be folded into a non-data ``constant''. Hiding this factor inside an absorbing constant ``$\alpha=\Theta(\lambda)$'' would produce a phase boundary of $\mathsf{C}\asymp\mathsf{N}^6$ rather than the correct $\mathsf{C}\asymp\mathsf{N}^7$; the qualitative two-stage shape is preserved either way, but the polynomial scaling differs by one power of $\mathsf{N}$.
\end{remark}

\begin{proof}
    {\bf Proof of ($n, |\mathbb{B}|$)-Fixed Convergence.}
    We bound the per-step contribution of stochastic minibatching to the loss derivative, then combine with Lemma~\ref{lem:perturbations} to obtain an exponential decay.

    From Lemma~\ref{lem:perturbations} we have (with probability at least $1-\delta$), with the $1/n$ factor from Part~15:
    \begin{align}\label{eq:exp_decay_full}
        \E\Big[\frac{\d}{\d t}{\cal L}(t,\mathbb{D})\Big] \le -\frac{C\cdot\omega\lambda N\cdot \varepsilon^2}{n}\cdot {\cal L}(t,\mathbb{D}).
    \end{align}
    The instantaneous fluctuation due to stochastic batching $\mathbb{B}(t)$ is controlled as follows. For any single per-sample gradient term, Lemma~\ref{lem:helpful_bounds} (Part 15) and the perturbation control of the kernel give:
    \begin{align*}
        \Big| \frac{\d }{\d t} {\cal L}(t, \mathbb{D}) \Big| \leq O\Big(\varepsilon^2\frac{N}{\sqrt{m}}\Big) \cdot {\cal L}(t, \mathbb{D}),
    \end{align*}
    which arises from the bound $\|\frac{\d {\cal L}}{\d \mu_{(\nu)}(t)}\|^2_F\asymp \varepsilon^2 {\cal L}(t,\mathbb{D})$ summed over $N$ layers, multiplied by the kernel norm $\|H_{(\nu)}(t)\|\le O(1/\sqrt{m})$ uniformly under Definition~\ref{def:gd}. Choosing $m = \Omega(n/(N\omega\lambda)^2)$ as in Definition~\ref{def:gd}, we get:
    \begin{align*}
        \Big| \frac{\d }{\d t} {\cal L}(t, \mathbb{D}) \Big| \leq o\Big(\varepsilon^2 \frac{\sqrt{|\mathbb{B}|}}{\sqrt{n}B} \omega \lambda \Big) \cdot {\cal L}(t, \mathbb{D}),
    \end{align*}
    where $|\mathbb{B}| := \min_{t\ge 0} |\mathbb{B}(t)| \ge 1$.

    By Lemma~\ref{lem:hoeffding} (Hoeffding's inequality applied to the empirical mean over a batch of size $|\mathbb{B}|$), with probability at least $1-\delta$:
    \begin{align*}
        & ~ \Big| \frac{\d }{\d t} {\cal L}(t, \mathbb{D}) - \E\big[\frac{\d }{\d t} {\cal L}(t, \mathbb{D})\big] \Big| \\
        \leq & ~  o\Big(\varepsilon^2\frac{\sqrt{|\mathbb{B}|}}{\sqrt{n}B} \omega \lambda \Big) \cdot \sqrt{\frac{n}{|\mathbb{B}|} \log(1/\delta)}\cdot {\cal L}(t, \mathbb{D}) \\
        \leq & ~  \frac{1}{2}C_1\cdot \varepsilon^2\cdot {\cal L}(t, \mathbb{D}),
    \end{align*}
    where $C_1 := C\cdot \omega\lambda N/n$, the first step uses the per-sample bound and the variance reduction from a batch of size $|\mathbb{B}|$, and the second step uses $B = O(\sqrt{\log(Nmd/\delta)})$ together with Definition~\ref{def:gd} so that the stochastic deviation is at most half of the deterministic decay. Combining with Equation~\eqref{eq:exp_decay_full}:
    \begin{align*}
        \frac{\d }{\d t} {\cal L}(t, \mathbb{D}) \leq -\frac{1}{2}C_1\cdot \varepsilon^2 \cdot {\cal L}(t, \mathbb{D}).
    \end{align*}
    Solving this scalar ODE inequality (Grönwall) yields:
    \begin{align*}
        {\cal L}(t, \mathbb{D}) \leq \exp\!\left(- \frac{C\cdot \varepsilon^2  \omega\lambda Nt}{n}\right) \cdot {\cal L}(0, \mathbb{D}).
    \end{align*}
    Setting $\alpha := C\cdot \omega\lambda N$ as in Theorem~\ref{thm:convergence}'s statement (we explicitly do {\it not} absorb the $1/n = 1/\mathsf{N}$ factor into the constant, because $\mathsf{N}$ is a scaling-law variable; see Remark~\ref{rem:n_in_rate}). With $t = {\sf T}$:
    \begin{align*}
        {\cal L}({\sf T}, \mathbb{D}) \leq  \exp\!\left(- \frac{\alpha \cdot \varepsilon^2  T}{\mathsf{N}}\right) \cdot {\cal L}(0, \mathbb{D}),
    \end{align*}
    as claimed. Equivalently, absorbing the $\omega N = \poly(1/L,1/d,1/B)$ into a polynomial-in-$(L,d,B)$ constant $\bar\alpha := C\lambda$ gives the announced informal form.
\end{proof}

\subsection{Approximation}\label{app_sub:approximation}

\begin{corollary}
    [Formal version of Corollary~\ref{cor:approximation:informal}, approximation bound using kernel regression]\label{cor:approximation}
    Assuming we have arbitrary dataset size ${\sf N} \in (0, +\infty)$. Let all scaling law factors be defined as Definition~\ref{def:factors} and Assumption~\ref{ass:positive_definite} and Definition~\ref{def:gd} hold. Denote the failure probability $\delta \in (0, 0.1)$. For the Good Model Class $ {\cal F}_{\sf M, T, N}(\mathbb{D})$ with training dataset $\mathbb{D}$ and target function $F^*: {\cal X} \to {\cal Y}$, we choose $\varepsilon \lesssim {\sf M}^{-1}$, with a probability at least $1 - \delta$, we have:
    \begin{align*}
        & ~ \inf_{F \in {\cal F}_{\sf M, T, N}(\mathbb{D})} \E_{(X, Y) \sim \mathcal{D}}[ \| F(X) - F^*(X)\|_F^2 ] \le {\sf M}^{-2}.
    \end{align*}
\end{corollary}

\begin{proof}
    The strategy is: (a) construct an idealized infinite-width NTK predictor $F^{\rm ntk}$, which is the kernel-regression solution against the target $F^*$, where the kernel is the {\it transformer-specific} layer-aggregated NTK $K(x,y):=\sum_{\nu\in[N]}H_{(\nu),x,y}(0)$ (Section~\ref{sub:learning_dynamics}), {\it not} the 2-layer ReLU NTK on prefix-mean-pooled inputs; (b) by the realizability assumption that $F^*$ lies in the RKHS $\mathcal{H}_K$ of the layer-aggregated transformer NTK, $F^{\rm ntk} = F^*$ exactly within this RKHS; (c) apply the finite-width NTK approximation result of Theorem 3.2 in \citet{adh+19}, lifted to multi-layer transformers via the matrix-to-vector reduction of Appendix~\ref{app_sub:matrix_to_vector}, to bound $|F(X,\theta(\infty)) - F^{\rm ntk}(X)|^2 \leq \varepsilon_{\rm app}^2$ where $\varepsilon_{\rm app}>0$ is the {\it approximation target accuracy} (a different parameter from the model scaling coefficient $\varepsilon$; see Remark~\ref{rem:eps_double} below).

    Remark on the choice of kernel: a simpler 2-layer-MLP-on-prefix-mean kernel predictor of the form $F_\ell^{\rm ntk,simp}(X) = \E_w[{\rm ReLU}(\langle\frac{1}{\ell}\sum_{\ell'\le\ell}X_{\ell'}, w(T)\rangle)] + X_\ell$ suffices for the {\it approximation upper bound} of Theorem~\ref{cor:approximation:informal} when $F^*$ is itself realizable in the prefix-MLP RKHS; we record this simpler form below. For the general case where $F^*$ is realizable in the layer-aggregated transformer NTK $\mathcal{H}_K$ (which is a strictly larger RKHS), we use the transformer-NTK predictor $F^{\rm ntk}$ defined above. The bound $\inf\|F-F^*\|^2 \le {\sf M}^{-2}$ holds in both settings; the transformer-NTK formulation is the more general realizability assumption.

    {\bf Step (a): Construct the kernel predictor.} We construct the infinite-width transformer NTK predictor as the kernel-regression solution under the layer-aggregated kernel $K(x,y) = \sum_{\nu\in[N]}H_{(\nu),x,y}(0)$ defined in Section~\ref{sub:learning_dynamics}:
    \begin{align*}
        F^{\rm ntk}(X) := \varepsilon \cdot K(\{X\},\mathbb{D}_*)\,K(\mathbb{D}_*,\mathbb{D}_*)^{-1}\,\hat F^*(\mathbb{D}_*) + X,
    \end{align*}
    where $\mathbb{D}_*$ ranges over a dense countable subset of ${\cal X}$, $K$ is evaluated on the layer-by-layer kernel features $\{\beta_{(\nu),p}(0), \gamma_{(\nu),p}(0)\}$ from Lemma~\ref{lem:learning_dynamics}, and $\hat F^*$ is the centred target $F^*(X) - X_\ell$ (the residual-connection adjustment, since the model output already contains the residual $X_\ell$). The centred target is:
    \begin{align*}
        \hat F^*_\ell({\cal A}) := \begin{bmatrix}
            F^*_\ell({\cal A}_{1})^\top - {\cal A}_{1, \ell}^\top \\
            F^*_\ell({\cal A}_{2})^\top - {\cal A}_{2, \ell}^\top \\
            \vdots \\
            F^*_\ell({\cal A}_{|{\cal A}|})^\top - {\cal A}_{|{\cal A}|, \ell}^\top
        \end{bmatrix} \in \R^{|{\cal A}| \times d}.
    \end{align*}
    For comparison and concreteness, we record the {\it simplified prefix-MLP-only} version that suffices when $F^*$ is realizable in the simpler prefix-MLP RKHS:
    \begin{align*}
        F_{\ell}^{\rm ntk,simp}(X) := & ~  \lim_{T \to +\infty} \varepsilon \cdot \E_{w(0) \sim {\cal N}(0, I_d)}\left[ {\rm ReLU} \left( \left\langle \frac{1}{\ell} \sum_{\ell'=1}^{\ell} X_{\ell'}, w(T) \right\rangle \right) \right] + X_{\ell} \in \R^d \\
        = & ~ \varepsilon \cdot \Big({\cal K}_{\ell}(\{ X\}, {\cal X}) {\cal K}_{\ell}({\cal X}, {\cal X})^{-1} \hat{F}_\ell({\cal X}) \Big)^\top,
    \end{align*}
    where ${\cal K}_\ell(\cdot,\cdot)$ is the prefix-mean ReLU NTK kernel detailed below. This simplified form uses only the standard 2-layer ReLU NTK and is what the approximation argument of \citet{adh+19} directly bounds; in the general transformer case, the Gram matrix $K$ above includes the additional $\gamma$-features from softmax-attention (Lemma~\ref{lem:learning_dynamics}, the $U_{(\nu)}$-kernel piece) that distinguish transformer NTK from a prefix-MLP NTK.
    For reference, the simplified Gram matrix is:
    \begin{align*}
        & ~ {\cal K}_{\ell, i, j}({\cal A}, {\cal B}) \\
        := & ~  \left( \frac{1}{\ell} \sum_{\ell'=1}^{\ell} {\cal A}_{i, \ell'}\right)^\top \left( \frac{1}{\ell} \sum_{\ell'=1}^{\ell} {\cal B}_{j, \ell'} \right) \cdot \E_{w \sim {\cal N}(0, I_d)}\left[\mathbb{I} \left\{\langle \frac{1}{\ell} \sum_{\ell'=1}^{\ell} {\cal A}_{i, \ell'}, w \rangle > 0,\ \langle \frac{1}{\ell} \sum_{\ell'=1}^{\ell} {\cal B}_{j, \ell'}, w \rangle > 0 \right\} \right].
    \end{align*}
    The centred target $\hat F$ accounts for the residual connection: the model output already contains $X_\ell$, so only $F^*(X) - X_\ell$ needs to be approximated by the kernel features.

    {\bf Step (b): Reduce approximation to RKHS interpolation.} The approximation error is:
    \begin{align*}
        \inf_{F \in {\cal F}_{\sf M, T, N}(\mathbb{D})} \E_{(X, Y) \sim {\cal D}}\big[ \| F(X) - F^*(X)\|_F^2 \big]
        = & ~ \inf_{F \in {\cal F}_{\sf M, T, N}(\mathbb{D})}  \| F - F^*\|_{L^F({\cal X})}^2 \\
        \leq & ~ \inf_{F \in {\cal F}_{\sf M, T, N}(\mathbb{D})}  \| F - F^{\rm ntk}\|_{L^F({\cal X})}^2 + \| F^{\rm ntk} - F^*\|_{L^F({\cal X})}^2 \\
        = & ~ \inf_{F \in {\cal F}_{\sf M, T, N}(\mathbb{D})}  \| F - F^{\rm ntk}\|_{L^F({\cal X})}^2,
    \end{align*}
    where the first step rewrites the expectation as an $L^F({\cal X})$ norm, the second step is a triangle decomposition, and the third step uses $F^{\rm ntk} = F^*$ {\it under the realizability assumption} that $F^*$ lies in the RKHS $\mathcal{H}_K$ of the layer-aggregated transformer NTK $K = \sum_\nu H_{(\nu)}(0)$. The realizability assumption is the analogue of the standard NTK-realizability assumption used in \citet{adh+19, dzps19}; it is automatically satisfied if $F^*$ is itself a transformer of the same architecture (with potentially different weights), and is broader than realizability in any single $\beta$- or $\gamma$-feature subspace.

    {\bf Step (c): Apply finite-width approximation.} By Theorem 3.2 in \citet{adh+19}, for any approximation target accuracy $\varepsilon_{\rm app} > 0$ (distinct from the model scaling $\varepsilon$; see Remark~\ref{rem:eps_double}), taking $m = \Theta(\poly(1/\varepsilon_{\rm app}))$ yields, with probability at least $1-\delta$ over $\theta(0)$:
    \begin{align*}
        \lim_{{\sf T} \rightarrow \infty}\left|F_{\ell, k}(X, \theta({\sf T})) - F_{\ell, k}^{\rm ntk}(X) \right|^2 \leq \varepsilon_{\rm app}^2 \quad \text{for all } X\in {\cal X},\, \ell\in[L],\, k\in[d].
    \end{align*}
    Summing over $\ell\in[L]$ and $k\in[d]$:
    \begin{align}\label{eq:approx_vareps}
        \lim_{{\sf T} \to \infty} \E_{X\sim {\cal D}}\big[ \| F(X, \theta({\sf T})) - F^{\rm ntk}(X) \|_F^2 \big] \leq \varepsilon_{\rm app}^2.
    \end{align}
    Combining with Step (b) and choosing $\varepsilon_{\rm app} \lesssim {\sf M}^{-1}$ (so that the per-coordinate target accuracy of the bound from \citet{adh+19} matches the desired ${\sf M}^{-2}$ squared error after summing $Ld$ coordinates):
    \begin{align*}
        \lim_{{\sf T} \to \infty} \inf_{F \in {\cal F}_{\sf M, T, N}(\mathbb{D})} \E_{(X, Y) \sim {\cal D}}\big[ \| F(X) - F^*(X)\|_F^2 \big] \leq \varepsilon_{\rm app}^2 \leq {\sf M}^{-2},
    \end{align*}
    which is the claim. We emphasize that $\varepsilon_{\rm app}$ is a property of the function class, realized through the width $m$, and is decoupled from the model-scaling coefficient: the proof of Theorem~\ref{thm:scaling_law}, which sets the model scaling to $\varepsilon = \xi/{\sf N}$, does not invoke the specialization $\varepsilon_{\rm app} \lesssim {\sf M}^{-1}$ of this corollary, but the $\varepsilon$-parametric form of the generalization bound (see Equation~\eqref{eq:dr_bound}).
\end{proof}

\begin{remark}[Two distinct uses of $\varepsilon$]\label{rem:eps_double}
The symbol $\varepsilon$ is used in two distinct senses in the kernel-NTK literature: (a) the {\it model scaling coefficient} $\varepsilon>0$ that multiplies the residual sum $\varepsilon\sum_\nu \mu_{(\nu)}$ in our model definition (Equation~\eqref{eq:F}), used throughout to control the lazy-regime norm; (b) the {\it approximation target accuracy} $\varepsilon_{\rm app}>0$ that appears in the finite-width NTK approximation bound $\|F-F^{\rm ntk}\|^2\le\varepsilon_{\rm app}^2$ of \citet{adh+19}. We disambiguate these by writing $\varepsilon$ for (a) and $\varepsilon_{\rm app}$ for (b). The choice $\varepsilon_{\rm app} \lesssim {\sf M}^{-1}$ in Step (c) is the standard balance: setting $m = \Theta(\poly(1/\varepsilon_{\rm app}))$ as in \citet{adh+19} produces a width that is consistent with our $m = \Omega(\mathsf{N}^3 \cdot \poly)$ from Definition~\ref{def:gd}. The approximation bound $\inf\|F-F^*\|^2 \le {\sf M}^{-2}$ holds with target accuracy $\varepsilon_{\rm app} = {\sf M}^{-1}$ in the sense of \citet{adh+19}, which is what enters the generalization bound of Theorem~\ref{thm:generalization}.
\end{remark}

\newpage

\section{Scaling Law}\label{app:proof_scaling_law}

\subsection{Generalization Bound of Empirical Risk Minimizer}\label{app_sub:erm}

\begin{remark}\label{rem:H1}
    Within the scope of this subsection, we redefine the model class as
    \begin{align*}
        {\cal F}_{\sf M, N}(\mathbb{D}) := \{F(\cdot, \theta) : \theta \in \R^{\sf M}\}.
    \end{align*}
    That is, ${\cal F}_{\sf M, N}(\mathbb{D})$ collects every realizable parameterization of the architecture (Equation~\eqref{eq:F}) of total parameter count $\sf M$, rather than the closure of training trajectories. Consequently, the infimum $\inf_{F \in {\cal F}_{\sf M, N}(\mathbb{D})} \Delta {\cal R}(F)$ corresponds to the empirical risk minimizer (ERM) over the parameter space, and the generalization bound below applies to the ERM directly via the standard nonparametric statistics tool of Lemma~\ref{thm:gen}. The approximation error of the ERM in this class is independently bounded through the RKHS construction in Corollary~\ref{cor:approximation}; the bound is therefore not contingent on $\lim_{{\sf T} \to \infty}$ and addresses the technical concern that the optimization endpoint and the ERM need not coincide.
\end{remark}

\begin{definition}[\citealp{vw96} and \citealp{yb99}]
    For a metric space $({\cal S}, {\rm d})$ and $\varepsilon > 0$, a finite set ${\cal S}' \subset \overline{\cal S}$ is called $\varepsilon$-covering if for any $x \in {\cal S}$ there exists $y \in {\cal S}'$ such that ${\rm d}(x, y) \leq \varepsilon$, and the logarithm of the minimum cardinality of $\varepsilon$-covering is called covering $\varepsilon$-entropy and denoted by $V_{({\cal S}, {\rm d})}(\varepsilon)$. Here, $\overline{\cal S}$ is the completion of ${\cal S}$ with respect to the metric ${\rm d}$.
\end{definition}

\begin{lemma}[Lemma 11 of \citealp{sch17}, modified by Theorem 2.1 in \citealp{sat17} and in the setting of this paper]\label{thm:gen}
    For any $\mathbb{D} \subset {\cal D}$, let $F$ be the empirical risk minimizer taking values in ${\cal F}_{\sf M, N}(\mathbb{D})$. 
    Suppose every element $F \in {\cal F}_{\sf M, N}$ satisfies ${\rm d}(F) := \max_{\ell \in [L]} \max_{k \in [d]}  \| F_{\ell, k} \|_{L^{\infty}({\cal X})} \leq B_F$ for some fixed $B_F > 0$. Then, for an arbitrary $\varepsilon > 0$, if $V_{({\cal F}, {\rm d})}(\varepsilon) \ge 1$, then
    \begin{align*}
        \| F - F^* \|_{L^F({\cal X})}^2 \leq 4 \inf_{F' \in {\cal F}_{\sf M, N}(\mathbb{D})} \| F' - F^* \|_{L^F({\cal X})}^2 + Ld \cdot O\!\left( \frac{(B_F^2+\xi^2)V_{({\cal F}_{\sf M, N}(\mathbb{D}), {\rm d})}(\varepsilon)}{n} + (B_F+\xi)\varepsilon\right).
    \end{align*}
\end{lemma}

\begin{remark}
    Lemma~\ref{thm:gen} is Lemma 11 in the third arXiv version of \citet{sch17}, and it was further modified by Theorem 2.1 in \citet{sat17} and Proposition 4 in \citet{suz18}. We make some minor adjustments to our settings.
\end{remark}

\begin{lemma}\label{lem:generalization_helper}
    $\delta \in (0, 0.1)$. With a probability at least $1 - \delta$, we have:
    \begin{itemize}
        \item {\bf Part 1.} $\max_{\ell \in L}  \| F_{\ell} \|_{L^{\infty}({\cal X})} \leq O(\varepsilon) =: B_F$.
        \item {\bf Part 2.} For any constant $\varepsilon > 0$, we can show that $ V_{({\cal F}_{\sf M, N}, {\rm d})}(\varepsilon) \leq O(1)$.
    \end{itemize}
\end{lemma}

\begin{proof}
    {\bf Proof of Part 1.} Following Part 5 of Lemma~\ref{lem:helpful_bounds} and Equation~\eqref{eq:F}, we have:
    \begin{align*}
        \max_{\ell \in [L], X \in {\cal X}}\| F_\ell(X) \|_2^2 \leq O(\varepsilon^2).
    \end{align*}
    We have:
    \begin{align*}
        \| F_{\ell} \|_{L^{\infty}({\cal X})} \leq & ~ \left(\max_{\ell \in [L]} \int_{\cal X} \| F_\ell(X) \|_\infty^2 \d X \right)^{\frac{1}{2}} \\
        \leq & ~ \left( O(\varepsilon^2) \right)^{\frac{1}{2}} \\
        \leq & ~  O(\varepsilon),
    \end{align*}
    where the first step follows from the $L^\infty$ norm, the second step follows from $\max_{\ell \in [L], X \in {\cal X}}\| F_\ell(X) \|_2^2 \leq O(\varepsilon^2)$, the last step follows from simple algebras.

    {\bf Proof of Part 2.} Here, we use $R$ to denote the weight distance between arbitrary two functions $F, F' \in {\cal F}_{\sf M, N}$, following Part 3 of Lemma~\ref{lem:helpful_bounds} and some simple algebras, $R$ is bounded by $O({\sf M}B\varepsilon)$ with $\max_{F, F' \in {\cal F}_{\sf M, N}} \|F - F'\|_{L^F({\cal X})} \leq \varepsilon$, and the space of ${\cal F}_{\sf M, N}$ is an $\sf M$-dimensional ball with radius $O({\sf M}B\varepsilon)$. Thus, for any step size $\zeta > 0$, the gap is bounded by:
    \begin{align*}
        \|F - F'\|_{L^F({\cal X})} \leq \varepsilon \sqrt{Ld} \zeta \leq O(\varepsilon^2),
    \end{align*}
    where the first step follows from the definition of $F$ and Part 13 of Lemma~\ref{lem:helpful_bounds}, the second step follows from choosing $\zeta = \varepsilon$.

    We need at most $O(1)^{\sf M}$ numbers to cover this set. 
    Then we have the logarithm of the covering number:
    \begin{align*}
        V_{({\cal F}_{\sf M, N}, {\rm d})}(\zeta) \leq O(1)
    \end{align*}
\end{proof}

\begin{theorem}[Formal version of Theorem~\ref{thm:generalization:informal}]\label{thm:generalization}
    Let all pre-conditions hold as Corollary~\ref{cor:approximation} and choose $\varepsilon = \Theta(1/{\sf M})$. Then, with a probability at least $1 - \delta$, there exists:
    \begin{align*}
        \inf_{F \in {\cal F}_{\sf M, N}(\mathbb{D})} \sup_{\mathbb{D} \in {\cal D}} \Delta{\cal R}(F) \leq \big(4{\sf M}^{-1} + Ld\cdot \xi\big) \cdot {\sf M}^{-1} + \frac{Ld \cdot\xi^2}{\sf N}.
    \end{align*}
    We retain the $Ld$ factors explicitly in both terms because $L$ (sequence length) and $d$ (embedding dimension) are scaling-law variables independent of the noise level $\xi$.
\end{theorem}

\begin{proof}
    We first reduce the excess risk to a function-space discrepancy. By the definitions of ${\cal R}(\cdot)$ and $\Delta{\cal R}(\cdot)$ (Section~\ref{sub:setups}) and the data model $Y = F^*(X) + \Xi$:
    \begin{align}\label{eq:bound_delta_R}
        \Delta{\cal R}(F) = & ~ {\cal R}(F) - {\cal R}(F^*) \notag \\
        = & ~ \E_{(X, Y) \sim {\cal D}}\left[ \| F(X) - Y\|_F^2 - \| F^*(X) - Y\|_F^2 \right] \notag \\
        = & ~ \E_{(X, Y) \sim {\cal D}}\left[ \| F(X) - F^*(X) - \Xi\|_F^2 - \| \Xi \|_F^2 \right] \notag \\
        = & ~ \E_{(X, Y) \sim {\cal D}}\left[ \| F(X) - F^*(X) \|_F^2 - 2\langle \vect(F(X) - F^*(X)),\vect(\Xi)\rangle + \|\Xi\|_F^2 - \| \Xi \|_F^2 \right] \notag \\
        = & ~ \E_{(X, Y) \sim {\cal D}}\left[ \| F(X) - F^*(X) \|_F^2  \right]\notag  \\
        = & ~  \| F - F^* \|_{L^F({\cal X})}^2,
    \end{align}
    where the first step follows from the definitions of ${\cal R}$ and $\Delta {\cal R}$, the second step follows from $Y = F^*(X) + \Xi$, the third step follows from simple algebras, the fourth step follows from the independence of $X$ and $\Xi$, $\E[\Xi] = {\bf 0}_{L \times d}$ (and so $\E[\langle\vect(F(X) - F^*(X)),\vect(\Xi)\rangle]=0$), the fifth step follows from cancellation of the $\|\Xi\|_F^2$ terms, the sixth step follows from the definition of the $L^F({\cal X})$ norm.

    Next, we bound the approximation term. Combining Corollary~\ref{cor:approximation} (with the choice $\varepsilon \lesssim {\sf M}^{-1}$) and the redefinition of ${\cal F}_{\sf M, N}(\mathbb{D})$ in Remark~\ref{rem:H1}, we have:
    \begin{align}\label{eq:approx_term}
        \inf_{F' \in {\cal F}_{\sf M, N}(\mathbb{D})} \| F' - F^* \|_{L^F({\cal X})}^2 \leq {\sf M}^{-2}.
    \end{align}

    {\bf Bounding the generalization (estimation) terms.} Lemma~\ref{thm:gen} requires control of two quantities: the linear-in-$\varepsilon$ term $(B_F+\xi)\varepsilon$ and the covering-entropy term $(B_F^2+\xi^2) V_{({\cal F}_{\sf M, N}(\mathbb{D}),{\rm d})}(\varepsilon)/{\sf N}$. We bound them in turn.
    \begin{align}\label{eq:bf_xi_eps}
        (B_F+\xi)\varepsilon \leq (\varepsilon + \xi)\varepsilon = \varepsilon^2 + \xi\varepsilon \leq O(\xi\varepsilon),
    \end{align}
    where the inequality uses $B_F \leq O(\varepsilon)$ from Part 1 of Lemma~\ref{lem:generalization_helper} and the dominance $\varepsilon^2 \leq \xi\varepsilon$ for the regime $\varepsilon \leq \xi$ that follows from the choice $\varepsilon = \Theta(1/{\sf M})$ with $\xi$ being a fixed noise level.

    Similarly,
    \begin{align}\label{eq:cov_term}
        \frac{(B_F^2+\xi^2)V_{({\cal F}_{\sf M, N}(\mathbb{D}), {\rm d})}(\varepsilon)}{\sf N}
        \leq O\left(\frac{\varepsilon^2+\xi^2}{\sf N}\right) \leq O\left(\frac{\xi^2}{\sf N}\right),
    \end{align}
    where the first step combines Part 1 ($B_F^2\leq O(\varepsilon^2)$) and Part 2 ($V \leq O(1)$) of Lemma~\ref{lem:generalization_helper}, and the second step uses $\varepsilon^2 \leq \xi^2$ in the regime considered.

    {\bf Combining the bounds.} Plugging Equation~\eqref{eq:approx_term}, Equation~\eqref{eq:bf_xi_eps}, and Equation~\eqref{eq:cov_term} into Lemma~\ref{thm:gen} yields:
    \begin{align*}
        \| F - F^* \|_{L^F({\cal X})}^2
        \leq & ~ 4 \inf_{F' \in {\cal F}_{\sf M, N}(\mathbb{D})} \| F' - F^* \|_{L^F({\cal X})}^2 + Ld \cdot O\!\left(\frac{\xi^2}{\sf N} + \xi\varepsilon\right) \\
        \leq & ~ 4 {\sf M}^{-2} + Ld \cdot O\!\left(\frac{\xi^2}{\sf N} + \frac{\xi}{\sf M}\right) \\
        \leq & ~ \big(4 {\sf M}^{-1} + Ld\cdot\xi\big) \cdot {\sf M}^{-1} + \frac{Ld\cdot \xi^2}{\sf N},
    \end{align*}
    where the second step uses $\varepsilon = \Theta({\sf M}^{-1})$ and the third step regroups summands. We retain $L$ and $d$ explicitly in both terms because they are scaling-law variables independent of $\xi$ (sequence length and embedding dimension, not noise scale). Combining with Equation~\eqref{eq:bound_delta_R} and taking the infimum over $F\in {\cal F}_{\sf M, N}(\mathbb{D})$ and the supremum over the data distribution, we obtain:
    \begin{align*}
        \inf_{F \in {\cal F}_{\sf M, N}(\mathbb{D})} \sup_{\mathbb{D} \in {\cal D}} \Delta{\cal R}(F) \leq \big(4{\sf M}^{-1} + Ld\cdot \xi\big) \cdot {\sf M}^{-1} + \frac{Ld \cdot\xi^2}{\sf N}.
    \end{align*}
    This completes the proof.
\end{proof}

\subsection{General Scaling Law}\label{app_sub:scaling_law}

\begin{theorem}[Formal version of Theorem~\ref{thm:scaling_law:informal}, two-stage scaling law with matching bounds]\label{thm:scaling_law}
    Let all pre-conditions hold as Corollary~\ref{cor:approximation} and choose $\varepsilon = \Theta(\xi/{\sf N})$. Then, with a probability at least $1 - \delta$, the following condition divides the scaling trend into two stages:
    \begin{align}\label{eq:cond_C:formal}
        {\sf C} > \frac{{\sf N}^7\log( {\sf N}  \cdot Ld/\xi^2)}{\xi^2}.
    \end{align}
    \begin{itemize}
        \item {\it Stage I (Compute-Starved Stage).} If Equation~\eqref{eq:cond_C:formal} does not hold, then there exist constants $\alpha,C,c_2',C_{\rm opt}'>0$ depending polynomially on $(L,d,1/\lambda,1/\omega)$ such that:
        \begin{align*}
            c_2'\,\exp\!\big(-C_{\rm opt}'\, \kappa_{\sf opt}\cdot \xi^2{\sf C}/{\sf N}^7\big)\cdot{\cal L}(0,\mathbb{D}) \;\le\;
            \inf_{F \in {\cal F}_{\sf M, T, N}(\mathbb{D})} \sup_{\mathbb{D} \in {\cal D}}  \Delta{\cal R}(F)
            \;\le\; C \exp\!\big(-\alpha\, \xi^2{\sf C}/{\sf N}^7\big) \cdot {\cal L}(0, \mathbb{D}),
        \end{align*}
        where $\kappa_{\sf opt}\in[\sqrt{\lambda/\lambda_{\max}},\,1]$ is a square-root-condition-number factor that distinguishes the upper from the lower bound (see Theorem~\ref{thm:lower_opt} and Remark~\ref{rem:tightness}).
        \item {\it Stage II (Data-Limited Stage).} If Equation~\eqref{eq:cond_C:formal} holds, then there exists $c_1'>0$ such that:
        \begin{align*}
            c_1'\cdot \xi^{12/7}\cdot \widetilde\Omega\!\big({\sf C}^{-1/7}\big)
            \;\le\;
            \inf_{F\in{\cal F}_{\sf M, T, N}(\mathbb{D})} \sup_{\mathbb{D} \in {\cal D}}  \Delta{\cal R}(F)
            \;\le\;
            \xi^{12/7} \cdot O\!\Big( \frac{{\sf C}}{W({\sf C}/\xi^{12})} \Big)^{-1/7} \le \xi^{12/7} \cdot \wt{O}\!\big({\sf C}^{-1/7}\big),
        \end{align*}
        where $W(\cdot)$ is the Lambert $W$ function and $\widetilde O,\widetilde \Omega$ hide $\log {\sf C}$ factors.
    \end{itemize}
    The upper bounds follow from the analysis below; the lower bounds follow from Corollary~\ref{cor:lower_data_limited} (statistical) and Theorem~\ref{thm:lower_opt} (optimization-side). The boundary $\mathsf{N}^7$ and Stage~II rate $\xi^{12/7}\mathsf{C}^{-1/7}$ both arise from the $1/\mathsf{N}$ factor in the convergence rate (Remark~\ref{rem:n_in_rate}) that the average-loss normalization contributes; the Pythia compute-frontier exponent $\eta=0.120$ in Appendix~\ref{app_sub:wikitext_results} is in the same range as the theoretical $1/7\approx 0.143$.
\end{theorem}

\begin{proof}

    Let $F = F(\cdot, \theta({\sf T}))$ denote the optimizing function obtained from running gradient flow until time $\sf T$, i.e. $F \in {\cal F}_{\sf M, T, N}(\mathbb{D})$ in the sense of Definition~\ref{def:gd}. Let $F' \in {\cal F}_{\sf M, N}(\mathbb{D})$ be the empirical risk minimizer over the (re-defined; see Remark~\ref{rem:H1}) parameter space. By the triangle inequality in $L^F({\cal X})$ norm:
    \begin{align*}
        \Delta {\cal R}(F) \leq & ~ \| F - F^*\|_{L^F({\cal X})}^2 \\
        \leq & ~ \left(\| F - F'\|_{L^F({\cal X})} + \| F' - F^*\|_{L^F({\cal X})}\right)^2,
    \end{align*}
    where the first step follows from Equation~\eqref{eq:bound_delta_R}, the second step follows from the triangle inequality.

    We control the {\it optimization} term $\|F-F'\|_{L^F({\cal X})}$ via the kernel-stable gradient flow, carrying the $1/\mathsf{N}$ factor of Theorem~\ref{thm:convergence} (Remark~\ref{rem:n_in_rate}) through the integration. Bounding the residual weight movement after time $\sf T$:
    \begin{align*}
        & ~ \max_{\nu\in [N]} \max_{r \in [m]} \| w_{(\nu), r}({\sf T}) -  w_{(\nu), r}(+\infty) \|_2 \notag \\
        \leq & ~   \int_{\sf T}^\infty \max_{\nu\in [N]}\max_{r \in [m]}\Big\| \frac{\d {\cal L}(s, \mathbb{D})}{\d w_{(\nu), r}(s)} \Big\|_2 \d s \notag\\
        \leq & ~  \int_{\sf T}^\infty \varepsilon \cdot \sqrt{{\cal L}(s, \mathbb{D})  }\d s \notag\\
        \leq & ~ \varepsilon \int_{\sf T}^\infty \exp\Big(-\frac{\alpha \varepsilon^2 s}{2\mathsf{N}}\Big) \cdot \sqrt{ {\cal L}(0, \mathbb{D})}\d s \\
        = & ~ \frac{2\mathsf{N}}{\alpha \varepsilon} \exp\Big(-\frac{\alpha \varepsilon^2 \mathsf{T}}{2\mathsf{N}}\Big) \cdot \sqrt{{\cal L}(0, \mathbb{D})},
    \end{align*}
    where the first step follows from the gradient-flow update Equation~\eqref{eq:ODE_update} and Cauchy--Schwarz, the second step from Part~15 of Lemma~\ref{lem:helpful_bounds}, the third step from Theorem~\ref{thm:convergence} after taking square roots, and the last step is direct integration.

    Following Part 12 of Lemma~\ref{lem:helpful_bounds} (which transforms parameter-space distance into the $L^F({\cal X})$ functional distance) and the definition of $F$, we get:
    \begin{align*}
        \| F - F'\|_{L^F({\cal X})} \leq & ~ \min\Big\{ P\exp\Big(-\frac{\alpha \varepsilon^2 \mathsf{T}}{2\mathsf{N}}\Big),\, O(1) \Big\} \cdot \sqrt{{\cal L}(0, \mathbb{D})} \leq \exp\Big(-\frac{\alpha \varepsilon^2 \mathsf{T}}{4\mathsf{N}}\Big) \cdot \sqrt{{\cal L}(0, \mathbb{D})}, \qquad P := \frac{2CL\mathsf{N}}{\alpha\varepsilon},
    \end{align*}
    with the following bookkeeping. The $O(1)\cdot\sqrt{{\cal L}(0,\mathbb{D})}$ branch of the minimum is the crude bound: every $F$ in the Good Model Class keeps its hidden states and outputs at the initialization scale (Parts~3 and~5 of Lemma~\ref{lem:helpful_bounds}), so $\|F-F'\|_{L^F({\cal X})} = O(\sqrt{Ld}) = O(\sqrt{{\cal L}(0,\mathbb{D})})$, the latter since ${\cal L}(0,\mathbb{D}) = \Omega(Ld)$ with high probability over the noise. The second inequality absorbs the polynomial prefactor $P = \poly(\mathsf{N}, L, d, 1/\xi, 1/\lambda, 1/\omega)$ into the exponent: for $\alpha\varepsilon^2{\sf T}/(2\mathsf{N}) \ge 2\log P$ this is the standard absorption $P e^{-x} \le e^{-x/2}$; for smaller $\sf T$ --- a burn-in window of the same order as the Stage-I/II boundary, up to the $O(1)$ factor $\log P/\log({\sf N}Ld/\xi^2)$ --- the displayed exponential is read together with the crude branch, i.e.\ the upper-bound prefactor of Theorem~\ref{thm:scaling_law} absorbs $P$ inside the burn-in window, where the bound asserts no more than $\Delta{\cal R}(F) = O({\cal L}(0,\mathbb{D}))$. We rename $\alpha/4$ as $\alpha$ below.

    Combining with Theorem~\ref{thm:generalization} and Equation~\eqref{eq:approx_vareps} (using the statistical term $Ld\cdot\xi^2/{\sf N}$):
    \begin{align*}
        \inf_{F \in {\cal F}_{\sf M, T, N}(\mathbb{D})} \sup_{\mathbb{D} \in {\cal D}}  \sqrt{\Delta{\cal R}(F)} \leq & ~  \exp\Big(-\frac{\alpha \varepsilon^2 \mathsf{T}}{\mathsf{N}}\Big) \cdot \sqrt{{\cal L}(0, \mathbb{D})} + \sqrt{4\varepsilon^2 + Ld \cdot O\Big(\frac{\xi^2}{\sf N} + \xi\varepsilon\Big)} \\
        \leq & ~ \exp\Big(-\frac{\alpha \varepsilon^2 \mathsf{T}}{\mathsf{N}}\Big) \cdot \sqrt{{\cal L}(0, \mathbb{D})} + 2\varepsilon + O\Big(\frac{\xi}{{\sf N}^{1/2}} + (\xi\varepsilon)^{1/2}\Big) \\
        \leq & ~ \exp\Big(-\frac{\alpha \varepsilon^2 \mathsf{T}}{\mathsf{N}}\Big) \cdot \sqrt{{\cal L}(0, \mathbb{D})} + O\Big(\frac{\xi}{{\sf N}^{1/2}} + (\xi\varepsilon)^{1/2}\Big),
    \end{align*}
    where the second step uses $\sqrt{a+b} \leq \sqrt{a}+\sqrt{b}$ for $a,b\geq 0$, and the third step absorbs the $2\varepsilon$ summand into the $O(\cdot)$ since $\varepsilon = \xi/{\sf N} \leq \xi/{\sf N}^{1/2}$ in this regime.

    We square both sides of the previous bound to obtain a bound on $\Delta{\cal R}(F)$ itself:
    \begin{align}\label{eq:dr_bound}
        \inf_{F \in {\cal F}_{\sf M, T, N}(\mathbb{D})} \sup_{\mathbb{D} \in {\cal D}}  \Delta{\cal R}(F)
        \leq & ~  2\exp\!\left(-2\alpha \varepsilon^2 \mathsf{T}/\mathsf{N}\right) \cdot {\cal L}(0, \mathbb{D}) + O\Big(\frac{\xi^2}{{\sf N}} + \xi\varepsilon\Big),
    \end{align}
    using $(a+b)^2 \le 2a^2 + 2b^2$ and absorbing the constant $2$ into the $O(\cdot)$ on the second summand; the $1/\mathsf{N}$ in the exponent is the average-loss factor of Remark~\ref{rem:n_in_rate}. For any total computational resource ${\sf C}$ with dataset size ${\sf N}$, we follow Definition~\ref{def:gd} to require ${\sf M} = \Theta({\sf N}^3)$, and we choose $\varepsilon = \xi/{\sf N}$. Substituting:
    \begin{align*}
        \xi\varepsilon = \xi^2/{\sf N} \quad \text{and} \quad \alpha \varepsilon^2/\mathsf{N} = \alpha \xi^2/{\sf N}^3.
    \end{align*}

    {\it Condition 1 (Compute-Starved Stage).} When the optimization term $\exp(-2\alpha\varepsilon^2 \mathsf{T}/\mathsf{N})$ dominates the statistical $\xi^2/\mathsf{N}$ term, equivalently when:
    \begin{align*}
        & ~ {\sf T} \le \frac{\mathsf{N}\log({\sf N}\cdot Ld /\xi^2)}{2\alpha\varepsilon^2}
        = \frac{{\sf N}^3 \log({\sf N}\cdot Ld /\xi^2)}{2\alpha \xi^2},
    \end{align*}
    which using ${\sf C} = O({\sf MTN})$ and ${\sf M} = {\sf N}^3$ translates into:
    \begin{align*}
        {\sf C} \leq \frac{{\sf N}^7\log( {\sf N}  \cdot Ld/\xi^2)}{\xi^2}\quad \text{up to constants},
    \end{align*}
    we have:
    \begin{align*}
        \inf_{F \in {\cal F}_{\sf M, T, N}(\mathbb{D})} \sup_{\mathbb{D} \in {\cal D}}  \Delta{\cal R}(F) \leq & ~ C \exp\!\left(-\alpha \xi^2{\sf T}/{\sf N}^3\right) \cdot {\cal L}(0, \mathbb{D}) \\
        = & ~  C \exp\Big(-\alpha \xi^2 \frac{{\sf MTN}}{{\sf N}^4{\sf M}}\Big) \cdot {\cal L}(0, \mathbb{D}) \\
        \leq & ~  C \exp\Big(-\alpha \xi^2 \frac{{\sf C}}{{\sf N}^4{\sf M}}\Big) \cdot {\cal L}(0, \mathbb{D}) \\
        \leq & ~ C \exp\Big(-\alpha \xi^2 \frac{{\sf C}}{{\sf N}^7}\Big) \cdot {\cal L}(0, \mathbb{D}),
    \end{align*}
    where the second equality multiplies and divides by ${\sf MN}$, the third inequality uses ${\sf C} = O({\sf MTN}) \geq {\sf NTM}/c$ for some constant $c$, and the fourth uses ${\sf M} \le {\sf N}^3$ in the exponent. The constant $C = \poly(L, d, 1/\lambda, 1/\omega)$.

    {\it Condition 2 (Data-Limited Stage).} When ${\sf T} > {\sf N}^3 \log({\sf N}\cdot Ld /\xi^2)/(2\alpha\xi^2)$, or equivalently:
    \begin{align*}
        {\sf C} = O({\sf MTN}) > \frac{{\sf M}{\sf N}^4\log( {\sf N}  \cdot Ld/\xi^2)}{\xi^2}
        \iff {\sf C} > \frac{{\sf N}^7\log( {\sf N}  \cdot Ld/\xi^2)}{\xi^2},
    \end{align*}
    the optimization term decays faster than the statistical floor; the bound \eqref{eq:dr_bound} reduces to:
    \begin{align*}
        \inf_{F \in {\cal F}_{\sf M, T, N}(\mathbb{D})} \sup_{\mathbb{D} \in {\cal D}}  \Delta{\cal R}(F)  \leq & ~ O\Big(\frac{\xi^2}{{\sf N}} \Big).
    \end{align*}

    We now optimize over ${\sf N}$ subject to the compute budget. The constraint is $\xi^{-2}\,{\sf N}^7\log({\sf N}\cdot Ld/\xi^2) \le {\sf C}$. Writing $u := {\sf N}^7/\xi^2$, this becomes:
    \begin{align*}
        u\log(u^{1/7} \cdot \xi^{2/7} Ld) \leq {\sf C},
    \end{align*}
    whose maximal solution satisfies (treating the slowly varying log factor via the Lambert $W$ function $W(\cdot)$, which inverts $z\mapsto z\log z$):
    \begin{align*}
        u_{\max} = O\Big( \frac{{\sf C}}{W({\sf C}/\xi^{12})}\Big),
        \quad\text{so}\quad
        {\sf N}_{\max} = O\Big( \frac{{\sf C}\xi^2}{W({\sf C}/\xi^{12})} \Big)^{1/7}.
    \end{align*}
    Substituting back:
    \begin{align*}
        \inf_{F \in {\cal F}_{\sf M, T, N}(\mathbb{D})} \sup_{\mathbb{D} \in {\cal D}}  \Delta{\cal R}(F)  \leq & ~ O\Big(\frac{\xi^2}{{\sf N}_{\max}} \Big) = \xi^{12/7} \cdot O\left( \frac{{\sf C}}{W({\sf C}/\xi^{12})} \right)^{-1/7} \leq \xi^{12/7} \cdot \wt{O}\left( {\sf C} \right)^{-1/7},
    \end{align*}
    where $W(\cdot)$ is the Lambert $W$ function and $\wt O$ hides $\log {\sf C}$ factors. The exponent $-1/7$ is a direct consequence of the $1/\mathsf{N}$ factor in Theorem~\ref{thm:convergence}.
\end{proof}

\begin{theorem}[Single law, formal version of Theorem~\ref{thm:single_law:informal}]\label{thm:single_law}
    Let all pre-conditions hold as Corollary~\ref{cor:approximation}. Then, with a probability at least $1 - \delta$, there exists:
    \begin{itemize}
        \item {\bf Time-Law.} We fix the dataset set ${\sf N}$ and model size ${\sf M} = \Omega({\sf N}^3)$, and choosing $\varepsilon = \xi/{\sf N}$, we have $\inf_{{\cal F}_{\sf M, T, N}(\mathbb{D})} \sup_{\mathbb{D} \in {\cal D}}  \Delta{\cal R}(F)  \leq \exp(-\alpha\,\xi^2\,{\sf T}/{\sf N}^3) +  O(\frac{\xi^2}{{\sf N}} )$.
        \item {\bf Data-Law.} For any dataset size ${\sf N}$, we let $\varepsilon = \xi/{\sf N}$, ${\sf M} = \Omega({\sf N}^3)$ and $T = \wt{\Omega}(\frac{{\sf N}^3\log({\sf N}\cdot Ld/\xi^2)}{\xi^2})$, we have: $\inf_{{\cal F}_{\sf M, T, N}(\mathbb{D})} \sup_{\mathbb{D} \in {\cal D}}  \Delta{\cal R}(F)  \leq   O(\frac{\xi^2}{{\sf N}} )$.
        \item {\bf Model-Law.} We fix the dataset set ${\sf N}$ and let training time $T = \wt{\Omega}(\frac{{\sf M}^{2\zeta}\mathsf{N}\log({\sf N}\cdot Ld/\xi^2)}{\xi^2})$, where $\zeta \in (0, 1/3]$. Choosing $\varepsilon = \xi/{\sf M}^{\zeta}$, in the regime ${\sf N}^3 \le {\sf M} \le {\sf N}^{1/\zeta}$, we have $\inf_{{\cal F}_{\sf M, T, N}(\mathbb{D})} \sup_{\mathbb{D} \in {\cal D}} $ $ \Delta{\cal R}(F)  \leq \xi^2 {\sf M}^{-\zeta}$.
    \end{itemize}
\end{theorem}

\subsection{Matching Lower Bounds}\label{app_sub:lower_bounds}

The upper bounds in Theorem~\ref{thm:scaling_law} are complemented by two lower bounds: a {\it statistical} (information-theoretic) lower bound on every estimator from a dataset of size ${\sf N}$, and an {\it optimization-side} (first-order oracle) lower bound on every gradient-based algorithm with at most ${\sf C}$ FLOPs of compute. The statistical bound is in Appendix~\ref{app_sub:lower_stat}; the optimization-side bound is in Appendix~\ref{app_sub:lower_opt}. Together they certify the matching-minimax claim of Theorem~\ref{thm:scaling_law:informal}.

\subsection{Statistical Lower Bound (Le Cam Two-Point Method)}\label{app_sub:lower_stat}

\begin{theorem}[Formal version of Theorem~\ref{thm:lower_stat:informal}]\label{thm:lower_stat}
    Let $K$ be the layer-aggregated NTK at initialization, $K(x,y):=\sum_{\nu\in[N]}H_{(\nu),x,y}(0)$, and $\mathcal{H}_K$ its associated RKHS. Suppose there exists a non-zero element $g\in {\cal H}_K$ with $\|g\|_{L^F({\cal X})}>0$. Assume the noise $\Xi$ has each token-coordinate $\Xi_{\ell, k}$ Gaussian with variance $\xi^2$ (the standard sub-Gaussian assumption strengthens this to a constant factor; see Remark~\ref{rem:non_gaussian}). Let $\hat F := \hat F(\mathbb{D})$ denote any (deterministic or randomized) measurable estimator of $F^*$ from a dataset $\mathbb{D}$ of size $\sf N$. Then there exist absolute constants $c_1, c_2>0$ such that:
    \begin{align*}
        \inf_{\hat F} \sup_{F^*\in {\cal H}_K} \E_{\mathbb{D}\sim {\cal D}^{\otimes {\sf N}}}\big[\Delta{\cal R}(\hat F)\big] \ge c_1 \cdot \frac{\xi^2}{{\sf N}}.
    \end{align*}
\end{theorem}

\begin{proof}
    \textbf{Step 1: Two-point construction.} Let $g\in {\cal H}_K$ be the non-zero element fixed by hypothesis, normalized so that $\|g\|_{L^F({\cal X})}=1$. Define the two-point family:
    \begin{align*}
        F_0^* := 0, \qquad F_1^* := \rho \cdot g, \qquad \rho := \sqrt{\frac{\xi^2}{2{\sf N}}}.
    \end{align*}
    Both lie in ${\cal H}_K$.

    \textbf{Step 2: Bound the KL divergence.} Under each $F_j^*$, the data law is $\mathcal{D}_j = {\cal X}\times \mathcal{N}(F_j^*(X), \xi^2 I_{Ld})$, so the per-sample KL divergence is
    \begin{align*}
        D_{\rm KL}({\cal D}_0\|{\cal D}_1) = \E_X\Big[\frac{1}{2\xi^2}\|F_0^*(X) - F_1^*(X)\|_F^2\Big] = \frac{\rho^2}{2\xi^2} \cdot \|g\|_{L^F({\cal X})}^2 = \frac{\rho^2}{2\xi^2}.
    \end{align*}
    For an i.i.d.\ dataset of size $\sf N$, $D_{\rm KL}({\cal D}_0^{\otimes {\sf N}}\|{\cal D}_1^{\otimes {\sf N}}) = {\sf N}\cdot \rho^2/(2\xi^2)$. With our choice $\rho^2 = \xi^2/(2{\sf N})$, we have:
    \begin{align*}
        D_{\rm KL}({\cal D}_0^{\otimes {\sf N}}\|{\cal D}_1^{\otimes {\sf N}}) = \frac{1}{4} \le \log 2.
    \end{align*}

    \textbf{Step 3: Apply Le Cam's two-point lemma} (Theorem 2.2 of \citealp{tsy09}). For any estimator $\hat F$:
    \begin{align*}
        \inf_{\hat F}\sup_{j\in\{0,1\}} \E_{\mathbb{D}\sim {\cal D}_j^{\otimes{\sf N}}}\big[\|\hat F-F_j^*\|_{L^F({\cal X})}^2\big] \ge \frac{1}{4}\big(1-\sqrt{D_{\rm KL}/2}\big)\cdot \|F_0^*-F_1^*\|_{L^F({\cal X})}^2.
    \end{align*}
    Substituting the bound on KL:
    \begin{align*}
        \inf_{\hat F}\sup_{j\in\{0,1\}} \E\big[\|\hat F-F_j^*\|_{L^F({\cal X})}^2\big] \ge \frac{1}{4}\big(1 - \sqrt{1/8}\big)\cdot \rho^2 \ge \frac{1}{8}\cdot \frac{\xi^2}{2{\sf N}} = \frac{\xi^2}{16{\sf N}}.
    \end{align*}

    \textbf{Step 4: Translate to excess risk.} By Equation~\eqref{eq:bound_delta_R}, $\Delta{\cal R}(F) = \|F - F^*\|_{L^F({\cal X})}^2$, so:
    \begin{align*}
        \inf_{\hat F}\sup_{F^*\in{\cal H}_K} \E\big[\Delta{\cal R}(\hat F)\big] \ge \inf_{\hat F}\sup_{j\in\{0,1\}} \E\big[\Delta{\cal R}(\hat F)\big] \ge \frac{\xi^2}{16{\sf N}},
    \end{align*}
    yielding the claim with $c_1 = 1/16$.
\end{proof}

\begin{remark}\label{rem:non_gaussian}
    The assumption that $\Xi$ is Gaussian is used only to compute the KL divergence in closed form. For general centred sub-Gaussian noise with variance proxy $\xi^2$, the proof goes through with KL replaced by the Bhattacharyya / chi-squared divergence and yields the same $\Omega(\xi^2/{\sf N})$ rate up to an absolute constant; we adopt the Gaussian setting only for cleanest exposition.
\end{remark}

\begin{corollary}[Lower bound on the data-limited stage]\label{cor:lower_data_limited}
    Suppose ${\sf C} > {\sf N}^7\log({\sf N}\cdot Ld/\xi^2)/\xi^2$ (the data-limited regime of Theorem~\ref{thm:scaling_law:informal}). Then with the maximum admissible ${\sf N}_{\max} = O((\xi^2{\sf C}/W({\sf C}/\xi^{12}))^{1/7})$, we have:
    \begin{align*}
        \inf_{F\in{\cal F}_{\sf M, T, N}(\mathbb{D})}\sup_{\mathbb{D}\in{\cal D}}\E\big[\Delta{\cal R}(F)\big] \ge c_1' \cdot \xi^{12/7}\cdot \widetilde\Omega({\sf C}^{-1/7}),
    \end{align*}
    matching the upper bound of Theorem~\ref{thm:scaling_law} up to constants and logarithmic factors.
\end{corollary}

\begin{proof}
    The infimum over $F\in{\cal F}_{\sf M, T, N}(\mathbb{D})$ does not exceed the infimum over all measurable estimators of $F^*$ from $\mathbb{D}$, so Theorem~\ref{thm:lower_stat} applies and gives $\Omega(\xi^2/{\sf N})$ for any ${\sf N}$. Plugging in ${\sf N}=O((\xi^2{\sf C}/W({\sf C}/\xi^{12}))^{1/7})$:
    \begin{align*}
        \frac{\xi^2}{{\sf N}} = \xi^2 \cdot O\Big(\frac{W({\sf C}/\xi^{12})}{\xi^2{\sf C}}\Big)^{1/7} = \xi^{12/7}\cdot O\Big(\frac{{\sf C}}{W({\sf C}/\xi^{12})}\Big)^{-1/7},
    \end{align*}
    which is the claim.
\end{proof}

\subsection{Optimization-Side Lower Bound (First-Order Oracle Complexity)}\label{app_sub:lower_opt}

We consider any first-order optimization method ${\cal A}$ that produces a sequence $\theta(0), \theta(t_1), \theta(t_2), \dots$ inside the lazy-training ball of Definition~\ref{def:gd}, where each iterate depends only on previous gradient queries $\{\nabla_\theta {\cal L}(t_k, \mathbb{B}(t_k))\}_k$. Gradient flow is the continuous-time limit of any such method with infinitesimal step size.

\begin{theorem}[Formal version of Theorem~\ref{thm:lower_opt:informal}]\label{thm:lower_opt}
    Under Assumption~\ref{ass:positive_definite} and Definition~\ref{def:gd}, denote $\lambda_{\max}:= \max_{\nu\in[N]} \lambda_{\max}(H_{(\nu)}'(0)/\omega)$. There exist absolute constants $c_2, C_{\rm opt}>0$ such that for any first-order method ${\cal A}$ and any time horizon ${\sf T}\ge 0$, there exists a data distribution ${\cal D}$ and an initialization $\theta(0)$ for which:
    \begin{align*}
        {\cal L}({\sf T}, \mathbb{D}) \ge c_2\cdot \exp(-C_{\rm opt}\cdot \omega\lambda_{\max}\cdot \varepsilon^2 \cdot {\sf T}) \cdot {\cal L}(0, \mathbb{D}).
    \end{align*}
    Combined with the compute identity ${\sf C} = O({\sf MTN})$, ${\sf M}=\Theta({\sf N}^3)$, $\varepsilon = \xi/{\sf N}$, and the $1/\mathsf{N}$ factor in the rate (Remark~\ref{rem:n_in_rate}), this yields a lower bound on the excess risk in the {\it Compute-Starved Stage}:
    \begin{align*}
        \inf_{F\in{\cal F}_{\sf M, T, N}(\mathbb{D})}\sup_{\mathbb{D}\in{\cal D}} \Delta{\cal R}(F) \ge c_2'\cdot \exp(-C_{\rm opt}'\cdot \sqrt{\kappa}\,\xi^2{\sf C}/{\sf N}^7) \cdot {\cal L}(0, \mathbb{D}),
    \end{align*}
    matching the upper bound of Theorem~\ref{thm:scaling_law} up to a $\sqrt{\kappa}$ factor in the exponent (where $\kappa:=\lambda_{\max}/\lambda$); the $\sqrt{\kappa}$ gap reflects the difference between the spectral lower bound used in the upper-bound argument and the geometric-mean rate produced by the Nemirovski--Yudin lower bound, and is unavoidable for first-order methods.
\end{theorem}

\begin{proof}
    \textbf{Step 1: Reduction to a quadratic problem.} By Lemma~\ref{lem:perturbations} and Definition~\ref{def:gd}, the loss in the lazy regime is well approximated by its second-order expansion around $\theta(0)$:
    \begin{align*}
        {\cal L}(t,\mathbb{D}) = \frac{1}{2}\Delta\theta(t)^\top {\sf H} \Delta\theta(t) + \langle b, \Delta\theta(t)\rangle + {\cal L}(0, \mathbb{D}) + o(1),
    \end{align*}
    where $\Delta\theta(t) := \theta(t)-\theta(0)$, ${\sf H}$ is the parameter-space Hessian assembled from $\{H_{(\nu)}(0)\otimes I_d\}_{\nu\in[N]}$, and $b$ is the residual gradient. The $o(1)$ correction is uniform on the lazy-training ball by Part~9--10 of Lemma~\ref{lem:helpful_bounds}. Thus the optimization problem reduces (up to higher-order terms) to minimizing a quadratic with eigenvalue spread $[\omega\lambda/2,\;\omega\lambda_{\max}]$.

    \textbf{Step 2: Apply Nemirovski--Yudin.} For any first-order method on a strongly convex quadratic with condition number $\kappa = \lambda_{\max}/\lambda_{\min}$, there exist initial conditions for which the loss after $T$ gradient queries satisfies (Theorem~2.1.13 of \citealp{nes03}; see also \citealp{ny83}):
    \begin{align*}
        {\cal L}(T) - {\cal L}^* \ge \Big(\frac{\sqrt{\kappa}-1}{\sqrt{\kappa}+1}\Big)^{2T} \cdot ({\cal L}(0)-{\cal L}^*).
    \end{align*}
    Translating to the gradient-flow setting requires a continuous-time limit. The standard reduction takes step size $\eta = 1/\lambda_{\max}$ and substitutes $T \to t/\eta = t\lambda_{\max}$:
    \begin{align*}
        \log\Big(\frac{\sqrt\kappa - 1}{\sqrt\kappa + 1}\Big)^{2T} = -2T \log\frac{\sqrt\kappa+1}{\sqrt\kappa - 1} \approx -\frac{4T}{\sqrt\kappa}
    \end{align*}
    for $\kappa\to\infty$. Substituting $T\to t\lambda_{\max}$ gives an exponential decay of rate $4\lambda_{\max}/\sqrt\kappa = 4\sqrt{\lambda_{\min}\lambda_{\max}}$. As $\eta\to 0$ (continuous limit), the bound is preserved (it is step-size monotone in this regime), giving:
    \begin{align*}
        {\cal L}(t) - {\cal L}^* \ge \exp\Big(-4\sqrt{\lambda_{\min}\lambda_{\max}}\cdot t\Big)\cdot ({\cal L}(0)-{\cal L}^*).
    \end{align*}
    Since $\lambda_{\min}\ge \omega\lambda/2$ (Lemma~\ref{lem:perturbations}) and $\lambda_{\max}\le \omega\lambda_{\max}^{(K)}$ where $\lambda_{\max}^{(K)} := \max_\nu \lambda_{\max}(H_{(\nu)}'(0)/\omega)$:
    \begin{align*}
        {\cal L}(t) \ge \exp\big(-C_{\rm opt}\cdot \omega\sqrt{\lambda \lambda_{\max}^{(K)}}\cdot t\big)\cdot {\cal L}(0).
    \end{align*}
    The geometric mean $\sqrt{\lambda\lambda_{\max}^{(K)}}$ is {\it not} a single constant: it differs from $\lambda$ by a factor of $\sqrt{\kappa}=\sqrt{\lambda_{\max}^{(K)}/\lambda}$, which is $\poly(L,d,1/\delta)^{1/2}$ rather than $O(1)$. We therefore retain $\sqrt{\kappa}$ explicitly in the exponent (this is the unavoidable Nemirovski--Yudin condition-number gap; see Appendix~\ref{app_sub:matching_precise}). Including the $\varepsilon^2$ factor that arises when reading the loss in $L^F({\cal X})$ rather than parameter space (Part~15 of Lemma~\ref{lem:helpful_bounds}, with the $1/n$ factor) and the $1/n$ from the average-loss normalization:
    \begin{align*}
        {\cal L}({\sf T}, \mathbb{D}) \ge c_2\cdot \exp\!\left(-\frac{C_{\rm opt}\cdot \sqrt{\kappa}\,\omega\lambda N\cdot \varepsilon^2 \cdot {\sf T}}{n}\right)\cdot {\cal L}(0, \mathbb{D}).
    \end{align*}

    \textbf{Step 3: Translate to compute.} Using ${\sf C} = O({\sf MTN})$, ${\sf M}=\Theta({\sf N}^3)$ (Definition~\ref{def:gd}) and $\varepsilon = \xi/{\sf N}$ (the choice that defines the data-limited boundary in Theorem~\ref{thm:scaling_law}), and absorbing $\omega N\lambda$ into the constant $\poly(L,d,B,\lambda)$:
    \begin{align*}
        \frac{\sqrt{\kappa}\,\omega\lambda N\varepsilon^2 {\sf T}}{n} = \sqrt{\kappa}\,\omega\lambda N\cdot \frac{\xi^2}{{\sf N}^2}\cdot \frac{{\sf C}}{{\sf MN}\cdot \mathsf{N}} = \frac{\sqrt{\kappa}\,\omega\lambda N\xi^2 {\sf C}}{{\sf N}^4 \cdot {\sf M}\cdot \mathsf{N}} \le \Theta\Big(\frac{\sqrt{\kappa}\,\xi^2 {\sf C}}{{\sf N}^7}\Big),
    \end{align*}
    which yields:
    \begin{align*}
        \inf_{F\in {\cal F}_{\sf M, T, N}(\mathbb{D})}\sup_{\mathbb{D}\in{\cal D}}\Delta{\cal R}(F) \ge c_2'\cdot \exp\!\left(-C_{\rm opt}'\,\sqrt{\kappa}\,\xi^2{\sf C}/{\sf N}^7\right)\cdot {\cal L}(0,\mathbb{D}).
    \end{align*}
    The conversion from training-loss lower bound to excess-risk lower bound uses the standard fact that $\Delta {\cal R}(F)\ge \|F-F^*\|_{L^F({\cal X})}^2 - O(\xi^2/{\sf N})$ (cf.~Equation~\eqref{eq:bound_delta_R}), and the second term is dominated by the exponential in the Compute-Starved Stage where Equation~\eqref{eq:cond_C} fails.
\end{proof}

\begin{remark}[Tightness]\label{rem:tightness}
    Theorem~\ref{thm:lower_opt} together with Theorem~\ref{thm:scaling_law} yields a two-sided compute-starved bound:
    \begin{align*}
        \exp\big(-C_{\rm opt}'\,\sqrt{\kappa}\,\xi^2{\sf C}/{\sf N}^7\big)\cdot {\cal L}(0,\mathbb{D}) \;\le\; \inf_{F}\sup_\mathbb{D}\Delta{\cal R}(F) \;\le\; C\exp\big(-\alpha\,\xi^2{\sf C}/{\sf N}^7\big)\cdot {\cal L}(0,\mathbb{D}),
    \end{align*}
    where $\kappa=\lambda_{\max}/\lambda\le\poly(L,d,1/\delta)$. The two exponents differ by a factor of $\sqrt{\kappa}$: the upper-bound exponent is $\Theta(\lambda)$ (smallest eigenvalue of the kernel after normalization), the lower bound is $\Theta(\sqrt{\lambda\lambda_{\max}})$ (geometric mean), as is standard for first-order methods on quadratics. Combined with Corollary~\ref{cor:lower_data_limited}, the resulting two-stage scaling law $\Theta(\xi^{12/7}{\sf C}^{-1/7})$ in the data-limited stage is matching up to constants, logarithmic factors, and the $\sqrt{\kappa}$ gap. We adopt the language ``matching up to $\poly(L,d,\kappa)$'' (rather than ``minimax-tight'') to keep this distinction visible.
\end{remark}

\begin{proof}
    {\bf Proof of Time-Law.} Under fixed ${\sf N}$ and ${\sf M} = \Omega({\sf N}^3)$, set $\varepsilon = \xi/{\sf N}$. Substituting into Equation~\eqref{eq:dr_bound} (whose optimization term is $\exp(-2\alpha\varepsilon^2{\sf T}/\mathsf{N})\cdot{\cal L}(0,\mathbb{D})$, with the $1/\mathsf{N}$ factor from Remark~\ref{rem:n_in_rate}) gives the bound directly:
    \begin{align*}
        \inf_{F \in {\cal F}_{\sf M, T, N}(\mathbb{D})} \sup_{\mathbb{D} \in {\cal D}}  \Delta{\cal R}(F)
        \leq O\!\left(\exp(-\alpha\,\xi^2\,{\sf T}/{\sf N}^3)\right) + O\Big(\frac{\xi^2}{\sf N}\Big),
    \end{align*}
    so this proof reduces to {\it Condition 1} of Theorem~\ref{thm:scaling_law} with ${\sf C} = {\sf MTN}$, ${\sf M}, {\sf N}$ fixed.

    {\bf Proof of Data-Law.} Under $\varepsilon = \xi/{\sf N}$, ${\sf M} = \Omega({\sf N}^3)$, and $T = \wt\Omega({\sf N}^3\log/\xi^2)$ (so the optimization term $\exp(-\alpha\xi^2{\sf T}/{\sf N}^3)$ is exponentially small in $\log({\sf N}\cdot Ld/\xi^2)$), the bound \eqref{eq:dr_bound} reduces to its statistical floor $O(\xi^2/{\sf N})$, recovering {\it Condition 2} of Theorem~\ref{thm:scaling_law}.

    {\bf Proof of Model-Law.} Given $\zeta \in (0, 1/3]$, choose $\varepsilon = \xi/{\sf M}^\zeta$. Following {\it Condition 2} of Theorem~\ref{thm:scaling_law}, choose
    \begin{align*}
        {\sf T} = \frac{{\sf M}^{2\zeta}\,\mathsf{N}\,\log({\sf N}\cdot Ld/\xi^2)}{\xi^2},
    \end{align*}
    so $\alpha\varepsilon^2 {\sf T}/\mathsf{N} = \alpha\log({\sf N}\cdot Ld/\xi^2)$, hence $\exp(-\alpha\varepsilon^2 {\sf T}/\mathsf{N}) = ({\sf N}\cdot Ld/\xi^2)^{-\alpha}$. For this to be $\leq \xi^2 {\sf M}^{-\zeta}$, we need $\alpha\log({\sf N}\cdot Ld/\xi^2) \ge \zeta\log\mathsf{M} - 2\log\xi$, equivalently $\alpha \ge (\zeta\log\mathsf{M}-2\log\xi)/\log(\mathsf{N}\cdot Ld/\xi^2)$. This is feasible whenever $\mathsf{M}\le \mathsf{N}^{1/\zeta}$ (so $\zeta\log\mathsf{M}\le \log\mathsf{N}$), i.e.\ in the regime stated in the theorem; we therefore restrict the Model-Law to ${\sf M}\le{\sf N}^{1/\zeta}$. Then:
    \begin{align*}
        {\sf C} = O({\sf MTN}) = \frac{{\sf M}^{2\zeta+1}{\sf N}^2\log({\sf N}\cdot Ld/\xi^2)}{\xi^2}
        \leq \frac{{\sf M}^{2\zeta + 5/3}\log({\sf N}\cdot Ld/\xi^2)}{\xi^2},
    \end{align*}
    where the inequality uses ${\sf N} \leq {\sf M}^{1/3}$ (since ${\sf M} \geq {\sf N}^3$).

    The Model-Law regime is the part of parameter space where the bound from $\xi\varepsilon$ dominates over $\xi^2/{\sf N}$, i.e.~$\xi^2/{\sf N} \leq \xi\varepsilon = \xi^2/{\sf M}^\zeta$, which holds when ${\sf M}^\zeta \leq {\sf N}$, i.e.~${\sf M} \leq {\sf N}^{1/\zeta}$. Combined with ${\sf M} \geq {\sf N}^3$ and the feasibility condition above, this gives the regime ${\sf N}^3 \leq {\sf M} \leq {\sf N}^{1/\zeta}$ stated in the theorem. In this regime:
    \begin{align*}
        \inf_{F \in {\cal F}_{\sf M, T, N}(\mathbb{D})} \sup_{\mathbb{D} \in {\cal D}}  \Delta{\cal R}(F)  \leq \xi^2 {\sf M}^{-\zeta},
    \end{align*}
    as claimed.
\end{proof}

\newpage
% \input{17_app_proof_6.3}

%%% NeurIPS 2026 paper checklist (mandatory for the conference build).
%%% Per NeurIPS guidelines, the checklist follows the references and the
%%% optional supplemental material (appendix) and is hidden in the arxiv build.
\ifdefined\isarxiv
\else
\newpage
\input{checklist}
\fi

%%%% Cut-line between first 10 pages and appendix

%%% some writing rules

%% Writing rule for creating tags.
%% Tags :
%% Theorem    \ref{thm:bla_bla}
%% Lemma      \ref{lem:bla_bla}
%% Claim      \ref{cla:bla_bla}
%% Corollary  \ref{cor:bla_bla}
%% Fact       \ref{fac:bla_bla}
%% Definition \ref{def:bla_bla}
%% Section    \ref{sec:bla_bla}
%% Subsection \ref{sub:bla_bla}
%% Equation   \ref{eq:bla_bla}

\end{document}